\documentclass{article}

\usepackage[final,nonatbib,main]{neurips_2026}
\usepackage[numbers,sort&compress]{natbib}

\usepackage[utf8]{inputenc}
\usepackage[T1]{fontenc}
\usepackage{graphicx}
\usepackage{booktabs}
\usepackage{amsmath,amssymb,amsfonts}
\usepackage{amsthm}
\usepackage{mathtools}
\usepackage{bm}
\usepackage[ruled,lined,linesnumbered]{algorithm2e}
\usepackage{enumitem}
\usepackage{multirow}
\usepackage{array}
\usepackage{float}
\usepackage{placeins}
\usepackage{wrapfig}
\usepackage{capt-of}
\usepackage{pifont}
\usepackage[colorlinks,citecolor=blue,urlcolor=blue,linkcolor=blue,breaklinks=true]{hyperref}
\usepackage[capitalize,nameinlink]{cleveref}
\hypersetup{pdftitle={STREAM: Stochastic Riemannian Flow Matching with Anisotropic Decoder for Digital Histopathology Image Generation}, pdfauthor={Won June Cho, Daeky Jeong, Hyeongyeol Lim, Hongjun Yoon}}
\usepackage{xcolor}

\newtheorem{theorem}{Theorem}
\newtheorem{proposition}{Proposition}
\newtheorem{lemma}{Lemma}
\newtheorem{corollary}{Corollary}
\newtheorem{definition}{Definition}

\DeclareMathOperator{\SLERP}{SLERP}
\DeclareMathOperator{\Exp}{Exp}
\DeclareMathOperator{\Log}{Log}
\DeclareMathOperator{\Ric}{Ric}

\newcommand{\cS}{\mathcal{S}}

\newcommand{\cX}{\mathcal{X}}
\newcommand{\cL}{\mathcal{L}}

\newcommand{\cN}{\mathcal{N}}

\newcommand{\bZ}{\mathbf{Z}}

\newcommand{\dg}{d_g}
\newcommand{\smax}{\sigma_{\max}}
\DeclareMathOperator{\supp}{supp}

\begin{document}

\title{STREAM: Stochastic Riemannian Flow Matching\\with Anisotropic Decoder\\for Digital Histopathology Image Generation}

\author{%
  Won June Cho \quad Daeky Jeong \quad Hyeongyeol Lim \quad Hongjun Yoon \\
  DEEPNOID Inc. \\
  \texttt{\{wjcho, dkjeong, hylim, hyoon\}@deepnoid.com}
}

\maketitle

\raggedbottom

\begin{abstract}
Synthetic histopathology image generation addresses patient-privacy concerns and the growing data demands of foundation models. Existing \emph{state-of-the-art} histopathology generative models use pretrained Vision Foundation Models (VFMs) as conditioning signals. We show this yields \emph{conditioning-dominated diversity}: on TCGA-BRCA, 62--75\% of their output diversity is attributable to the conditioning signal rather than the learned latent space, while \emph{de novo} synthesis still requires a VFM at inference. We instead use histopathology VFMs as the latent space itself: their patch tokens are $\ell_2$-normalized on the unit hypersphere $\cS^{d-1}$ with strong angular dominance and intrinsic curvature, motivating a Riemannian formulation. We present \emph{STREAM}, the first framework to apply Riemannian flow matching in the histopathology domain, in two stages: 1)~a bridge-type stochastic perturbation that establishes per-token rectifiability on $\cS^{d-1}$ for training a Diffusion Transformer, and 2)~a novel decoder \emph{training} design whose noise covariance is anisotropic in the \emph{left}-singular basis of the per-token tangent-projected velocity-field Jacobian, spending a large robustness budget on its low-response directions and a small one on its high-response directions. Across TCGA-BRCA and TCGA-COADREAD, STREAM achieves \emph{state-of-the-art} gFID and ranks first on nearly all histopathology-specific metrics as well. Code and a public gallery of generated images are available at \url{https://chokevin8.github.io/STREAM-Patho/}.
\end{abstract}

\section{Introduction}
\label{sec:intro}

Synthetic medical image generation addresses critical challenges in computational histopathology: patient privacy concerns that limit data sharing across institutions~\citep{wang2025selfimproving} and the growing need for training data as large Vision Foundation Models (VFMs) and Vision-Language Models are applied to histopathology~\mbox{\citep{chen2024uni,zimmermann2024virchow2,lu2024conch,xu2024gigapath,xiang2025musk}}.

For high-resolution image generation, latent diffusion models~\citep{rombach2022high,ho2020denoising,peebles2023scalable} have dominated the space. A recent line of work redesigns the underlying tokenizer~\mbox{\citep{yao2025vavae,chen2025maetok,leng2025repae,zheng2025rae}}, as the choice of latent space is increasingly recognized as critical to downstream generation quality. The underlying observation across these works is that latents generated via Variational Autoencoders (VAEs) primarily emphasize low-level details~\citep{chen2026aligntok}, with linear classifiers on VAE features substantially underperforming those on representation-encoder features~\citep{zheng2025rae,chen2025maetok}. We confirm this pattern transfers to the histopathology setting: linear probing on downstream binary classification (\cref{tab:perturbation}) shows that the VAE encoders used by current state-of-the-art histopathology generative models ZoomLDM~\citep{yellapragada2025zoomldm} and PixCell~\citep{yellapragada2025pixcell} yield substantially lower performance than histopathology VFMs~\mbox{\citep{chen2024uni,zimmermann2024virchow2,xu2024gigapath}}. Both methods condition diffusion on the [CLS]-token of a pretrained histopathology VFM, which we show empirically (\cref{sec:conditioning}) leads to conditioning-dominated diversity: most of the output diversity is attributable to the conditioning signal rather than to the learned latent space, which therefore carries little semantic content, and any \emph{de novo} synthesis requires a VFM at inference, which is impractical for clinical deployment. Moreover, digital histopathology particularly benefits from unconditional generative approaches due to the high cost of expert pathologist annotation~\citep{campanella2019clinical} and the combinatorial diversity of tissue morphology~\citep{litjens2017survey}. Therefore, using histopathology VFM patch-token features directly as the generative latent space resolves both issues, but requires accounting for their geometry: the features are naturally $\ell_2$-normalized on the unit hypersphere $\cS^{d-1}$~\citep{kumar2026manifold,chang2026dinosae}, with strong angular dominance (\cref{sec:angular}) and substantial intrinsic curvature (\cref{sec:intrinsic_geometry}), motivating Riemannian flow matching (RFM). STREAM is the first framework to apply RFM with VFM encoders in the histopathology domain. Our main contributions are:
\begin{enumerate}[leftmargin=*,nosep]
\item \textbf{Motivating RFM in the histopathology domain} (\cref{sec:motivation}): we quantify conditioning-dominated diversity in current state-of-the-art histopathology models and provide empirical evidence on the spherical geometry of histopathology VFM features that justifies a Riemannian formulation.

\item \textbf{Bridge-type stochastic RFM} (\cref{sec:bridge}): a tangent-Gaussian perturbation of the spherical linear interpolation (SLERP) geodesic on $\cS^{d-1}$ with a Brownian-bridge schedule that vanishes at the endpoints, providing full support of the marginal law for all $t \in (0,1)$ and per-token rectifiability of the bridge construction at the population level (\cref{thm:rectifiability}).

\item \textbf{Anisotropic decoder training} (\cref{sec:aniso}): a novel decoder training design whose noise covariance is shaped by the singular value decomposition (SVD) of the velocity-field Jacobian of the trained Diffusion Transformer (DiT)~\citep{peebles2023scalable}: small noise along the high-energy left singular directions $U_H$, which span the directions of largest velocity-field response, to preserve reconstruction fidelity, and large noise along the complementary low-response directions $U_L$.
\end{enumerate}

\section{Related Work}
\label{sec:related}

\paragraph{Diffusion models for histopathology generation.}
Diffusion-based generative modeling for histopathology spans large-image and multi-scale synthesis~\citep{le2024infbrush,graikos2024lrdm,yellapragada2025pixcell}. Other works incorporate text or mask conditioning for guided synthesis like pathology reports~\citep{yellapragada2024pathldm}, cell topology~\citep{xu2025topocellgen}, and joint nuclei image--label co-synthesis~\citep{min2024cosynthesis}. PathDiff~\citep{bhosale2025pathdiff} conditions on unpaired text and mask; like ZoomLDM~\citep{yellapragada2025zoomldm} and PixCell~\citep{yellapragada2025pixcell}, it is a conditional VAE-latent method. At inference it accepts text, a mask, or both, and provides no unconditional branch, so it cannot sample without a conditioning signal. While digital histopathology routinely leverages pretrained VFMs for downstream tasks (classification, segmentation, retrieval), no current state-of-the-art histopathology generation model exploits these representations as the generative latent space. Therefore, STREAM intends to focus on purely unconditional histopathology generation: a diffusion model that faithfully learns the histopathology data distribution from a VFM-encoded latent space without external conditioning.

\paragraph{Latent space modification in diffusion model training.}
VFMs have emerged as feature extractors for generative models via two paradigms. Paradigm~1 uses VFM features as training losses: REPA~\citep{yu2025repa} aligns DiT features with DINOv2~\citep{oquab2024dinov2}, REPA-E~\citep{leng2025repae} extends to end-to-end VAE tuning, and others similarly leverage VFM representations~\citep{shi2026rectok,yao2025vavae,wang2025reg}. Paradigm~2 uses VFM features as the latent space: RAE~\citep{zheng2025rae} trains DiTs directly on representation-encoder features with noise-augmented decoders, SVG~\citep{shi2025svg} similarly uses frozen DINOv3~\citep{simeoni2025dinov3} features, concatenated with a trained residual branch, and related works explore fine-tuning representation encoders under a semantic-preservation loss~\citep{tang2025unilip,chen2026aligntok}. Decoder regularization via masking or denoising losses has been shown beneficial for jointly-trained tokenizers~\citep{yang2025ldetok,qiu2025robustok,kouzelis2025eqvae}, but \citet{yang2025ldetok} explicitly note that gains depend on joint encoder--decoder training. We adopt Paradigm~2 because histopathology-specific VFMs~\citep{chen2024uni,zimmermann2024virchow2,xu2024gigapath} produce $\ell_2$-normalized features naturally suited to Riemannian geometry. STREAM differs from approaches that jointly adapt the tokenizer encoder--decoder to improve reconstruction fidelity~\citep{chen2026aligntok,leng2025repae} by keeping the histopathology VFM frozen and instead recovering diffusion-friendliness via geometry: a stochastic bridge on $\cS^{d-1}$ paired with anisotropic decoder regularization (\cref{sec:aniso}).

\paragraph{Flow matching and Riemannian extensions.}
Flow matching (FM)~\citep{lipman2023flow,albergo2023stochastic,liu2023flow} provides a simulation-free alternative for generative modeling, scaled to high-resolution synthesis~\citep{esser2024scaling,ma2024sit}. Rectified flow~\citep{liu2022rectified,liu2023flow} reframes generation as straight transport between marginals. On manifolds, the canonical formulation is the SLERP conditional flow on $\cS^{d-1}$~\citep{chen2024riemannian}, inherited by subsequent RFM work. In high-dimensional image generation, RJF~\citep{kumar2026manifold} runs on frozen-VFM latents and reweights the velocity-matching loss by a Jacobi-field factor; HAE~\citep{chang2026dinosae} replaces Vision Transformer (ViT) patchification with a trained convolutional stem and aligns to the VFM by a cosine loss; and SFM~\citep{lee2026geometry} renormalizes autoencoder latents onto the sphere. Each defines the conditional path as a deterministic geodesic between prior and data and regresses the velocity given in closed form by the time-derivative of that path, so the target lies in the tangent space by construction. STREAM's design differs further in three aspects: (i)~perturbing the SLERP path with a stochastic bridge and regressing a normalized endpoint under the chord loss (\cref{eq:vloss_chord}), (ii)~exploiting the Riemannian formulation for anisotropic decoder training, and (iii)~jointly addressing rectifiability and decoder--generator interaction in the RFM setting. Direct numerical comparison with these RFM works is not possible, since none of them has been open-sourced for reproducible training to date.

\section{Why RFM for Histopathology?}
\label{sec:motivation}

\subsection{Conditioning-Dominated Diversity and Angular Dominance in VFM Feature Spaces}
\label{sec:conditioning}

State-of-the-art histopathology generative models ZoomLDM~\citep{yellapragada2025zoomldm} and PixCell~\citep{yellapragada2025pixcell} condition diffusion models on VFM embeddings $c = E_{\text{VFM}}(I)$. We formalize a conditioning dominance ratio via the entropy decomposition $H(X) = H(X \mid C) + I(X; C)$:
\begin{equation}
  \rho_{\text{cond}} = \frac{I(X; C)}{H(X)} = 1 - \frac{H(X \mid C)}{H(X)},
  \label{eq:rho_cond}
\end{equation}
estimated via the Vendi Score~\citep{friedman2023vendi}, the Conditional Vendi Score (CVS)~\citep{jalali2024cvs}, and SPEC-diff~\citep{jalali2025specdiff}. The Vendi Score and CVS exponentiate the Shannon entropy of the normalized eigenvalues of a similarity kernel matrix on UNI embeddings. The log Vendi Score estimates the total output diversity $H(X)$, and the log CVS estimates the diversity left at fixed conditioning, $H(X \mid C)$. Their difference $I(X; C)$ is the amount due to conditioning, and $\rho_{\text{cond}}$ is its share of $H(X)$. Both models exhibit conditioning-dominated diversity (\cref{tab:collapse}): 62--75\% of output diversity is attributable to the conditioning signal rather than the learned latent space. This is compounded by VAE-extracted latents being weakly structured semantically~\citep{chen2026aligntok,zheng2025rae,chen2025maetok}: linear-probe AUROC on SPIDER-breast~\citep{nechaev2025spider} for malignant vs.\ benign tissue is $\geq 0.937$ for VFMs vs.\ $\leq 0.640$ for VAEs (\cref{tab:perturbation}; full protocol in \cref{app:linprob}). This motivates using VFM features directly as the generative latent space.

\begin{table}[t]
  \begin{minipage}[t]{0.40\linewidth}
    \centering
    \caption{Conditioning-dominated diversity on BRCA (\cref{eq:rho_cond}).}
    \label{tab:collapse}
    \scriptsize
    \setlength{\tabcolsep}{2pt}
    \begin{tabular}{@{}l ccccc@{}}
      \toprule
      Model & $H(X|C)$ & $H(X)$ & $\rho_\text{cond}$ & CVS & SPEC-diff \\
      \midrule
      ZoomLDM & 1.28 & 5.17 & 75\% & 3.59 & 0.067 \\
      PixCell & 1.81 & 4.78 & 62\% & 6.08 & 0.102 \\
      \bottomrule
    \end{tabular}
  \end{minipage}\hspace{0.03\linewidth}%
  \begin{minipage}[t]{0.56\linewidth}
    \centering
    \caption{Linear probing on SPIDER-breast (benign vs.\ malignant). AUROC is the area under the receiver operating characteristic curve.}
    \label{tab:perturbation}
    \scriptsize
    \setlength{\tabcolsep}{2pt}
    \begin{tabular}{@{}l c c ccc@{}}
      \toprule
      Model & $d$ & AUROC & $\text{Acc}_\text{dir}$ & $\text{Acc}_\text{mag}$ & $f_\text{angular}$ \\
      \midrule
      UNI                       & 1024 & 0.995          & 0.993          & 0.583 & 0.999 \\
      UNI2-h                    & 1536 & \textbf{0.997} & \textbf{0.995} & 0.603 & 0.973 \\
      DINOv3-L                  & 1024 & 0.937          & 0.930          & 0.514 & 0.999 \\
      \midrule
      PixCell VAE               & 16   & 0.640          & 0.832          & 0.549 & 0.994 \\
      ZoomLDM VAE               & 3    & 0.624          & 0.609          & 0.543 & 0.813 \\
      \bottomrule
    \end{tabular}
  \end{minipage}
\end{table}

\phantomsection\label{sec:angular}%
We next ask where in VFM features the semantic information resides via 1-nearest-neighbor (1-NN) angular vs.\ radial perturbation on SPIDER-breast (full protocol in \cref{app:perturbation_protocol}). For both histopathology VFMs and the natural-image-domain DINOv3-L, $f_\text{angular} > 0.97$ (\cref{tab:perturbation}): class-discriminative information lives in the angular structure of $\cS^{d-1}$, directly motivating RFM.

\subsection{Intrinsic Geometry of VFM Feature Manifolds}
\label{sec:intrinsic_geometry}

Following \citet{xiong2026exploiting}, we analyze the intrinsic geometry of histopathology VFM features on TCGA-BRCA (hereafter BRCA) for UNI ($d{=}1024$) and UNI2-h ($d{=}1536$). Spectral analysis yields a high covariance effective rank ($R_\mathrm{eff} = 252$ for UNI, $265$ for UNI2-h), and tangent drift exceeds $0.72$ at hop~1, direct evidence of manifold curvature (\cref{app:intrinsic_geometry}). Drift this large means straight-line interpolants leave the manifold, and the features are already $\ell_2$-normalized onto $\cS^{d-1}$, so RFM on the hypersphere is the better-matched formulation. We use UNI ($d{=}1024$) henceforth: UNI2-h buys 5\% more effective rank for 50\% more embedding dimension.

\section{Method: STREAM}
\label{sec:method}

\begin{figure}[t]
  \centering
  \includegraphics[width=0.9513\linewidth]{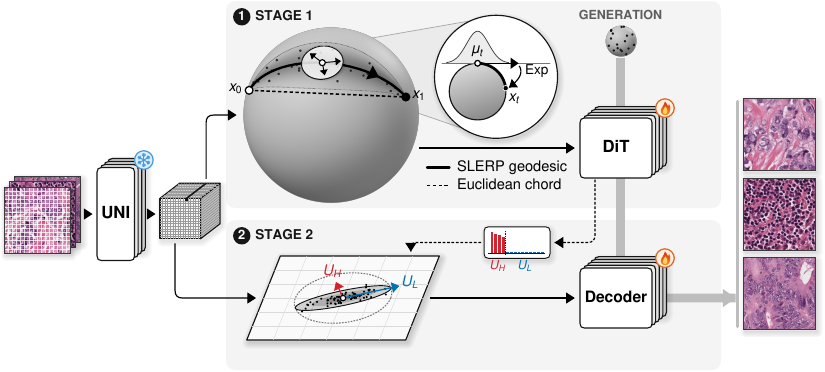}
  \caption{\textbf{Overview of STREAM.} A frozen histopathology VFM encoder maps each input patch to $N{=}256$ tokens on the unit hypersphere $\cS^{d-1}$; a DiT learns to transport a uniform source distribution on $(\cS^{d-1})^N$ to the data distribution along bridge-perturbed geodesics. A separately-trained anisotropic decoder reconstructs histopathology images from the generated features.}
  \label{fig:overview}
\end{figure}

FM~\citep{lipman2023flow} learns a time-dependent velocity field $v_\theta(x, t)$ generating a probability path $p_t$ from a source $p_0$ to data $p_1$. The intractable marginal FM loss shares its gradient with the tractable conditional FM (CFM) loss $\cL_\text{CFM} = \mathbb{E}_{t,z,X_t|z}[\|v_\theta - u_t\|^2]$, where $u_t(x \mid z)$ is the conditional velocity along a known path. RFM is the same construction on a Riemannian manifold, with each Euclidean vector-space operation replaced by its intrinsic analogue. On the unit hypersphere $\cS^{d-1} = \{x \in \mathbb{R}^d : \|x\| = 1\}$, the Riemannian counterparts of vector difference, displacement and distance are the logarithmic map $\Log_x$, the exponential map $\Exp_x$ and the geodesic distance $\dg$. The conditional path becomes a geodesic~\citep{chen2024riemannian}, SLERP with constant speed $\|\dot{\mu}_t\|_g = \Omega$:
\begin{equation}
  \SLERP(x_0, x_1, t) = \frac{\sin((1-t)\Omega)}{\sin\Omega}\,x_0 + \frac{\sin(t\Omega)}{\sin\Omega}\,x_1, \quad \Omega = \dg(x_0, x_1),
  \label{eq:slerp}
\end{equation}
with conditional velocity $u_t(x \mid x_1) = \Log_x(x_1)/(1-t)$. Background on optimal transport (OT), the curvature-dimension condition, and rectifiability theory is in \cref{app:prelim}.

\subsection{Problem Setup}
\label{sec:setup}

Our pipeline (\cref{fig:overview}) adopts the RAE framework~\citep{zheng2025rae}:
\begin{enumerate}[nosep,leftmargin=*]
\item \textbf{Encoder} $E: \cX \to (\cS^{d-1})^N$: A frozen VFM (UNI~\citep{chen2024uni}, a ViT-L/16~\citep{dosovitskiy2021image}) mapping 256$\times$256 histopathology patches to $N = (256/16)^2 = 256$ $\ell_2$-normalized patch tokens of dimension $d = 1024$.
\item \textbf{Generator} $G_\theta$: A DiT learning to generate features $\hat{\bZ} \sim p_\theta$ approximating $p_\text{data} = E_\#(p_\text{real})$.
\item \textbf{Decoder} $D_\phi: \mathbb{R}^{Nd} \to \cX$: A ViT-based decoder reconstructing images from features.
\end{enumerate}

The three components interact as follows. The encoder $E$ produces clean VFM features $\bZ_1 = E(I) \in (\cS^{d-1})^N$ for each training image $I$. The generator $G_\theta$ is trained to recover $\bZ_1$ from a bridge-perturbed interpolant on $(\cS^{d-1})^N$: a uniform source $\bZ_0 \sim p_0$ is interpolated toward $\bZ_1$ along the SLERP geodesic to produce $\mu_t$, then perturbed by tangent-Gaussian noise $X_t = \Exp_{\mu_t}(\sigma(t)\epsilon)$ with the bridge schedule $\sigma(t) = \smax\sin(\pi t)$ peaking at $t{=}0.5$ and vanishing at the endpoints (\cref{sec:bridge}). The decoder $D_\phi$ is trained separately, after $G_\theta$, on clean VFM features perturbed by anisotropic noise whose covariance is shaped by the SVD of the trained DiT's velocity-field Jacobian (\cref{sec:aniso}). This couples the decoder's directional robustness budget to the DiT velocity field's left-singular directions of largest and smallest response. At generation time, $\bZ_0 \sim p_0$ is integrated through $G_\theta$ via Riemannian midpoint steps on $(\cS^{d-1})^N$ to produce $\hat{\bZ}_1 \sim p_\theta$, and $D_\phi(\hat{\bZ}_1)$ outputs the image. Both perturbations act at training time only: at generation time $D_\phi$ reads $\hat{\bZ}_1$ unperturbed.

The generative model operates on the product manifold $(\cS^{d-1})^N$ with the product metric $g = \bigoplus_{n=1}^N g_n$, total dimension $N(d-1) = 261{,}888$. The source distribution is $p_0 = \text{Uniform}((\cS^{d-1})^N)$, the unique stationary measure of Brownian motion on $(\cS^{d-1})^N$, geometrically canonical by the Bakry--\'Emery curvature-dimension condition $\text{CD}(d-2, \infty)$~\citep{bakry1985diffusions}. Geodesics, exponential maps, and all constructions decompose as per-token operations on the product manifold; we present the theory for a single copy of $\cS^{d-1}$ throughout.

\subsection{Stochastic Bridge RFM}
\label{sec:bridge}

RFM~\citep{chen2024riemannian} on $\cS^{d-1}$ trains a velocity field $v_\theta(x_t, t)$ via the CFM loss along deterministic SLERP paths,
\begin{equation}
  \cL_\text{RFM}(\theta) = \mathbb{E}_{t \sim U[0,1],\,(x_0, x_1) \sim \pi}\!\left[\bigl\|v_\theta(x_t, t) - u_t(x_t \mid x_1)\bigr\|_g^2\right],
  \label{eq:rfm_loss}
\end{equation}
with $x_t = \SLERP(x_0, x_1, t)$ and $u_t(x_t \mid x_1) = \Log_{x_t}(x_1)/(1-t)$. Two limitations of \cref{eq:rfm_loss} motivate our bridge construction: (L1)~no general rectifiability guarantee~\citep{hertrich2025rectified}, addressed by \cref{thm:rectifiability}; and (L2)~the disconnected-intermediate-support obstruction~\citep[Prop.~10]{hertrich2025rectified} that arises when data clusters on $\cS^{d-1}$ and is not removed by a uniform source (\cref{app:fullsupport_why}), resolved by \cref{cor:disconnected}.

\begin{definition}[Stochastic bridge conditional path]
\label{def:bridge}
Given the independent coupling $\pi = p_0 \otimes p_1$ on $\cS^{d-1} \times \cS^{d-1}$, $(x_0, x_1) \sim \pi$, and noise schedule $\sigma(t) = \smax \sin(\pi t)$:
\begin{equation}
  \mu_t = \SLERP(x_0, x_1, t), \qquad
  x_t = \Exp_{\mu_t}\!\bigl(\sigma(t)\,\epsilon\bigr), \quad
  \epsilon \sim \cN(0, \Pi_{\mu_t}),
  \label{eq:bridge}
\end{equation}
where $\Pi_{\mu_t} = I - \mu_t \mu_t^\top$ is the tangent-space projector. Throughout, $t{=}0$ corresponds to the source distribution $p_0$ (uniform noise) and $t{=}1$ to the data distribution $p_1$, following the FM convention~\citep{lipman2023flow,liu2022rectified}.
\end{definition}

The stochastic noise schedule $\sigma(t) = \smax \sin(\pi t)$ is the $n=1$ Karhunen--Lo\`eve (KL) mode of the Brownian bridge (the unique sinusoid strictly positive on $(0,1)$ with vanishing endpoints), with $\smax \lesssim \pi/\sqrt{d-1}$ ensuring perturbations remain inside the injectivity radius of $\cS^{d-1}$. Vanishing endpoints $\sigma(0)=\sigma(1)=0$ are critical: the bridge must not perturb the source ($t{=}0$) or data ($t{=}1$) endpoints, so $X_0 = x_0$ and $X_1 = x_1$ exactly and the network sees clean source and data without conflating them with noise (\cref{app:noise_schedule}).

\paragraph{Training loss and pipeline.}
We train with a chordal (extrinsic Euclidean) loss, hereafter the chord loss, replacing the tangent-space geodesic distance with the ambient $\ell_2$ distance between normalized predictions, using x-prediction~\citep{li2025backtobasics} with normalized output $\hat{x}_1 = f_\theta(x_t, t)/\|f_\theta(x_t, t)\|$:
\begin{equation}
  \cL_\text{chord}(\theta) = \mathbb{E}_{t,\,(x_0,x_1) \sim \pi,\,\epsilon}\!\left[\frac{\|\hat{x}_1 - x_1\|^2}{(1-t)^2}\right].
  \label{eq:vloss_chord}
\end{equation}
Sample by sample, $\cL_\text{chord}$ and the natural Riemannian (geodesic) loss have the same minimizer $\hat{x}_1 = x_1$ and the same inference velocity $v_\theta(x_t, t) = \Log_{x_t}(\hat{x}_1)/(1-t)$, but $\cL_\text{chord}$ avoids an $\arccos$-induced singularity near $t \to 1$ (\cref{app:x_prediction}). Concretely, $f_\theta$ is a LightningDiT-XL~\citep{yao2025vavae} trained to predict per-token data endpoints on $\cS^{d-1}$ from bridge-perturbed VFM embeddings (\cref{alg:training}). At inference, generation uses 25 midpoint Euler steps on $(\cS^{d-1})^N$ (\cref{alg:generation}), followed by anisotropic decoding (\cref{sec:aniso}). The bridge smooths the endpoint regression with bandwidth $\sigma(t)$ rather than defining a path we integrate, so the marginal field it induces differs from the inference field (\cref{app:exact_velocity}).

\begin{proposition}[Full support]
\label{prop:full_support}
Under the stochastic bridge (\cref{def:bridge}) with $\smax > 0$, and for any coupling $\pi$ giving no mass to the antipodal locus $\{(x_0,x_1) : x_1 = -x_0\}$, $\operatorname{law}(X_t)$ has full support on $\cS^{d-1}$ for all $t \in (0,1)$.
\end{proposition}

\begin{theorem}[Bridge perturbation: three consequences]
\label{thm:rectifiability}
Under the hypotheses of \cref{prop:full_support}, on a single $\cS^{d-1}$ the bridge yields three separate properties:
\begin{enumerate}[nosep]
\item[(i)] \textbf{Full support} (\cref{prop:full_support}): $\operatorname{supp}(\operatorname{law}(X_t)) = \cS^{d-1}$ for all $t \in (0,1)$.
\item[(ii)] \textbf{Smooth conditional densities}: the conditional distribution of $(X_0, X_1)$ given $X_t = x$ admits a smooth density for $\operatorname{law}(X_t)$-a.e.\ $x$, the step from which part~(iii) obtains rectifiability in the sense of~\citet[Def.~1]{hertrich2025rectified}.
\item[(iii)] \textbf{Velocity-field regularity}: for the integration cutoff $\epsilon_t > 0$, the marginal velocity $v_t(x)$ is smooth on $\cS^{d-1} \times [\epsilon_t, 1-\epsilon_t]$, and $\dot{x} = v_t(x)$ admits a unique solution there for any initial condition.
\end{enumerate}
All three lift unchanged from $\cS^{d-1}$ to the product $(\cS^{d-1})^N$, with $N$ entering only the constants.
\end{theorem}

Full proof is in \cref{app:proofs_bridge}. No region of $\cS^{d-1}$ is unreachable at an interior time, which removes the disconnected-support obstruction (L2) that deterministic SLERP leaves in place. \Cref{lem:full_support_cond} guarantees the density that part~(ii) conditions on is positive almost everywhere, and the bridge kernel on $\cS^{d-1}$ supplies the smooth conditional densities themselves. Part~(ii) together with the bounded-velocity bound then gives part~(iii) by the Picard--Lindel\"of theorem, the unique solution of $\dot x = v_t(x)$ that \citet[Def.~1]{hertrich2025rectified} call rectifiability. The bridge removes the binding obstruction that \citet[Prop.~10]{hertrich2025rectified} identify by supplying the intermediate-marginal property of part~(i). The lift to $(\cS^{d-1})^N$ carries no cross-token assumption, because the SLERP geodesic, the exponential map and the tangent Gaussian act per token.

\subsection{Spectral-Informed Anisotropic Decoder Regularization}
\label{sec:aniso}

Without noise injection, a decoder trained on clean VFM features must handle generated features $\hat{z} \sim p_\theta$ that may fall in regions of $\cS^{d-1}$ underrepresented in the training distribution. Isotropic noise augmentation improves generation quality, measured by the Fr\'echet Inception Distance (FID) of generated images (gFID), but degrades reconstruction quality, measured by the FID of reconstructed images (rFID)~\citep{zheng2025rae}; with covariance $\sigma^2 I$, that cost falls on all directions equally. Empirically (\cref{tab:ablation}), allocating noise according to the SVD of the trained DiT's velocity-field Jacobian $J(z) = \nabla_z v_\theta(z, t)$ yields complementary rFID/gFID gains; we interpret this as the SVD decomposing each token's tangent space into high-energy directions $U_H$, the top-$k^*$ left singular directions by spectral energy, and low-energy directions $U_L$, the rest. Because $v_\theta$ is a tangent field, the cached $J$ is a single token's tangent-projected diagonal block, an endomorphism of $T_z\cS^{d-1}$; its left factor is an orthonormal basis of $\operatorname{range}(J)$, so $U_H$ and $U_L$ split the velocity field's response directions rather than its input-perturbation sensitivities (\cref{app:svd_factor}). We use this partition as an empirical budget for the decoder's directional robustness and inject asymmetric noise:
\begin{equation}
  \Sigma_\text{noise}(z) = \sigma_H^2\,U_H U_H^\top + \sigma_L^2\,U_L U_L^\top, \quad \sigma_H \ll \sigma_L,
  \label{eq:aniso_cov}
\end{equation}
where $U_H$ holds the top-$k^*$ singular vectors of $J(z)$ and $U_L$ their tangent complement, with $k^*$ defined by the energy threshold $\tau$: $k^* = \min\{k : \sum_{i \le k} s_i^2 / \sum_i s_i^2 \ge \tau\}$ where $s_i$ is the median cached spectrum. The basis $\{U_H(\bar{z}_c), U_L(\bar{z}_c)\}$ is precomputed once from the trained DiT at the single time anchor $t = 0.99$: training features are $k$-means clustered into $K$ centroids $\bar{z}_c$, the per-(centroid, token) SVD of $J(\bar{z}_c, t{=}0.99)$ is computed via batched forward-mode automatic differentiation (AD) with $m$ random tangent probes, and at training time each sample is assigned to its nearest centroid and receives the corresponding basis $U_H(\bar{z}_c)$. Crucially, no discrete spectral gap is required; only the continuous energy ordering matters. Full computation and deployed values are in \cref{app:implementation}.

\subsection{Decoder Training}
\label{sec:decoder_training}

The decoder $D_\phi: (\cS^{d-1})^N \to \mathbb{R}^{H \times W \times 3}$ is a ViT-XL mapping VFM embeddings on the product hypersphere back to histopathology images, trained from scratch with anisotropic noise injection from the start (\cref{alg:decoder_training}), using a 4-component loss with a Learned Perceptual Image Patch Similarity (LPIPS) term:
\begin{equation}
  \cL_\text{dec} = \lambda_{\ell_1}\|\hat{I} - I\|_1 + \lambda_\text{LPIPS}\text{LPIPS}(\hat{I}, I) + \lambda_\text{cos}(1 - \langle z_\text{rt}, z\rangle) + \omega_G \lambda_\text{adapt} \cL_\text{adv},
  \label{eq:decoder_loss}
\end{equation}
where $\hat{I} = D_\phi(\tilde{z})$ is the reconstruction from noisy features $\tilde{z} = (z + n)/\|z + n\|$ with $n = \Pi_z\eta$, $\eta \sim \cN(0, \Sigma_\text{noise}(\bar{z}_c) + \sigma_L^2\bar{z}_c\bar{z}_c^\top)$, $z_\text{rt} = E(D_\phi(\tilde{z}))$ is the round-trip re-encoding, and $\cL_\text{adv}$ is an adversarial loss against a Kaiko ViT-S/8 discriminator~\citep{kaiko2024towards}, with adaptive weight following VQGAN~\citep{esser2021taming}. The cosine round-trip loss targets the clean features $z$ (not $\tilde{z}$), forcing the decoder to denoise. Training follows a standard VQGAN-style staged schedule~\citep{esser2021taming,rombach2022high}, with anisotropic noise from the beginning so the decoder jointly allocates capacity for reconstruction and noise robustness.
Deriving the noise covariance from the SVD of the velocity-field Jacobian (\cref{eq:aniso_cov}) creates a deliberate coupling: the spectral decomposition governing the DiT's directional response determines the decoder's robustness allocation (\cref{tab:ablation}). Training the same decoder on the right singular vectors instead is worse on both metrics (\cref{app:svd_factor}), so STREAM keeps the left factor.

\section{Experiments}
\label{sec:experiments}

\subsection{Implementation Details}
\label{sec:exp_setup}

\paragraph{Datasets and encoder.} We use 12.2M patches from BRCA~\citep{tcga2012comprehensive}, following the patch set of ZoomLDM~\citep{yellapragada2025zoomldm}, and 3.4M patches from TCGA-COADREAD~\citep{tcga2012colorectal} (hereafter COADREAD), both at 20$\times$ magnification (256$\times$256 pixels). We extract all training VFM features with UNI~\citep{chen2024uni} (ViT-L/16, $d = 1024$, $N = 256$ tokens). We train and evaluate on another independent, non-TCGA dataset at 20$\times$ magnification, SPIDER-skin~\citep{nechaev2025spider}, and report that comparison in \cref{app:spider_skin}.

\paragraph{Training.} We train sequentially in two stages on 8$\times$ NVIDIA H200 GPUs. In Stage~1, we train the LightningDiT-XL~\citep{yao2025vavae} (${\sim}676$M parameters) with the extracted UNI patch embeddings for 95k steps with batch size 1024 using the stochastic bridge FM loss (\cref{alg:training}). In Stage~2, we extract the cached SVD of the velocity-field Jacobian from the trained DiT first, and use it to train a ViT-XL decoder~\citep{zheng2025rae,dosovitskiy2021image} from scratch with anisotropic noise for 90k steps with a batch size of 512. Full details are in \cref{app:implementation}.

\paragraph{Baselines.} We benchmark against the state-of-the-art histopathology-domain models ZoomLDM~\citep{yellapragada2025zoomldm} and PixCell~\citep{yellapragada2025pixcell}, and against three methods that build diffusion models on VFM-encoded latents: RAE~\citep{zheng2025rae} as our Euclidean baseline, SVG~\citep{shi2025svg}, which augments RAE's design with a learned residual encoder on top of the frozen VFM features, and REPA-E~\citep{leng2025repae}, which uses the VFM as a representation-alignment target that tunes a VAE end-to-end. We train all baselines on BRCA and COADREAD with the default hyperparameters and settings of their respective papers. Every method in \cref{tab:results} uses the same UNI encoder, so the comparison is encoder-controlled. RAE, SVG and REPA-E are therefore retrained here on UNI ($d{=}1024$) rather than on their native natural-image encoders. UNI enters each method in a different role: as the generative latent for STREAM, RAE and SVG, with no VAE; as a conditioning input on a VAE latent for ZoomLDM and PixCell; and as a representation-alignment target that also tunes the VAE end-to-end for REPA-E. RAE and SVG remain the clean encoder-controlled tests of our claim, since only they share STREAM's latent space.

\paragraph{Evaluation.} We evaluate with three ImageNet-pretrained distribution metrics (FID~\citep{heusel2017gans}, Kernel Inception Distance (KID)~\citep{binkowski2018demystifying}, CLIP Maximum Mean Discrepancy (CMMD)~\citep{jayasumana2024rethinking}) and three histopathology-aware variants FvD, KvD, vMMD, in which the underlying feature extractor is replaced by Virchow2~\citep{zimmermann2024virchow2}, another state-of-the-art histopathology VFM. LPIPS~\citep{zhang2018perceptual} is also used to measure reconstruction quality. STREAM, RAE and SVG are sampled at a matched number of function evaluations ($\mathrm{NFE}{=}50$). All models are scored at 50k generated samples against the same real reference set.

\subsection{Image Generation Results}
\label{sec:main_results}

\begin{table}[t]
  \caption{Generation results on BRCA and COADREAD $20\times$ patches ($256{\times}256$), with all six methods on the same UNI encoder. All metrics lower-is-better; \textbf{best} bold, \underline{second-best} underlined within each block. LPIPS is reconstruction-only.}
  \label{tab:results}
  \centering
  \scriptsize
  \setlength{\tabcolsep}{1.5pt}
  \begin{minipage}[t]{0.495\linewidth}
    \centering
    \begin{tabular}{@{}l ccccccc@{}}
      \multicolumn{8}{@{}c@{}}{\textbf{BRCA}} \\
      \toprule
      Method & FID & KID & CMMD & LPIPS & FvD & KvD & vMMD \\
      \midrule
      \multicolumn{8}{@{}l}{\textit{Reconstruction}} \\
      ZoomLDM         & 4.88              & 0.0040             & 0.2685             & 0.048             & 37.26              & 0.038             & 4.12              \\
      PixCell         & 2.88              & 0.0025             & \underline{0.2590} & \textbf{0.036}    & 18.11              & \underline{0.018} & \underline{2.04}  \\
      \textbf{STREAM} & \underline{2.23}  & \underline{0.0019} & \textbf{0.1509}    & \underline{0.047} & \textbf{9.91}      & \textbf{0.008}    & \textbf{0.98}     \\
      RAE             & 6.60              & 0.0058             & 1.0225             & 0.105             & 30.33              & 0.059             & 6.09              \\
      SVG             & \textbf{2.16}     & \textbf{0.0017}    & 0.8534             & 0.049             & \underline{16.68}  & 0.020             & 2.21              \\
      REPA-E          & 26.83             & 0.0211             & 1.4417             & 0.162             & 262.65             & 0.468             & 51.37             \\
      \midrule
      \multicolumn{8}{@{}l}{\textit{Generation}} \\
      ZoomLDM         & \underline{7.43}  & \underline{0.0058} & \underline{0.2611} & --                & 196.41             & 0.194             & 21.01             \\
      PixCell         & 104.18            & 0.0984             & 0.6147             & --                & 1298.00            & 1.410             & 169.38            \\
      \textbf{STREAM} & \textbf{6.16}     & \textbf{0.0041}    & \textbf{0.1282}    & --                & \textbf{78.04}     & \textbf{0.048}    & \textbf{5.16}     \\
      RAE             & 15.84             & 0.0126             & 1.0960             & --                & \underline{103.14} & \underline{0.172} & \underline{17.26} \\
      SVG             & 17.37             & 0.0149             & 0.8895             & --                & 123.04             & 0.185             & 19.90             \\
      REPA-E          & 28.57             & 0.0224             & 1.4155             & --                & 309.39             & 0.549             & 60.14             \\
      \bottomrule
    \end{tabular}
  \end{minipage}\hfill
  \begin{minipage}[t]{0.495\linewidth}
    \centering
    \begin{tabular}{@{}l ccccccc@{}}
      \multicolumn{8}{@{}c@{}}{\textbf{COADREAD}} \\
      \toprule
      Method & FID & KID & CMMD & LPIPS & FvD & KvD & vMMD \\
      \midrule
      \multicolumn{8}{@{}l}{\textit{Reconstruction}} \\
      ZoomLDM         & 4.12              & 0.0033             & 0.1068             & 0.050             & 29.50              & \underline{0.024} & \underline{2.35}  \\
      PixCell         & \underline{2.58}  & \underline{0.0021} & 0.0796             & \textbf{0.038}    & \textbf{15.87}     & \textbf{0.012}    & \textbf{1.20}     \\
      \textbf{STREAM} & 5.64              & 0.0043             & \underline{0.0407} & 0.075             & 32.08              & 0.030             & 2.72              \\
      RAE             & 6.31              & 0.0050             & 0.0872             & 0.115             & 35.31              & 0.056             & 5.43              \\
      SVG             & \textbf{1.81}     & \textbf{0.0012}    & \textbf{0.0304}    & \underline{0.045} & \underline{21.37}  & 0.028             & 2.78              \\
      REPA-E          & 13.58             & 0.0096             & 0.2654             & 0.172             & 324.93             & 0.554             & 50.47             \\
      \midrule
      \multicolumn{8}{@{}l}{\textit{Generation}} \\
      ZoomLDM         & \underline{8.09}  & \underline{0.0053} & \underline{0.0613} & --                & 139.63             & \underline{0.097} & \underline{9.52}  \\
      PixCell         & 127.24            & 0.1132             & 0.9570             & --                & 1742.85            & 2.145             & 218.70            \\
      \textbf{STREAM} & \textbf{7.68}     & \textbf{0.0042}    & \textbf{0.0293}    & --                & \underline{91.96}  & \textbf{0.055}    & \textbf{5.10}     \\
      RAE             & 19.64             & 0.0144             & 0.1605             & --                & \textbf{91.74}     & 0.129             & 11.51             \\
      SVG             & 23.64             & 0.0189             & 0.1042             & --                & 152.29             & 0.209             & 18.52             \\
      REPA-E          & 13.41             & 0.0090             & 0.2362             & --                & 345.29             & 0.568             & 51.62             \\
      \bottomrule
    \end{tabular}
  \end{minipage}
\end{table}

As shown in \cref{tab:results}, STREAM achieves the best gFID on both datasets (BRCA $6.16$, COADREAD $7.68$) against all other baselines. PixCell's design prevents unconditional generation: it always requires an input image (or a VFM embedding extracted from one) at inference. Forced to generate from its learned latent space alone, PixCell reports extremely high gFID values, consistent with the measured conditioning-dominance ratio $\rho_\text{cond}{=}62\%$. ZoomLDM partially sidesteps this requirement by training a conditional diffusion model directly on the VFM embedding, allowing unconditional sampling without an external image; yet its gFID is higher than STREAM's. Both results indicate that the VAE-induced latent space is too semantically poor to support true generative modeling, even where reconstruction remains strong.

STREAM instead treats the VFM embedding as the latent space itself, without external conditioning, outperforming its histopathology-domain counterparts. However, STREAM's performance is not solely due to UNI being utilized as the latent space: pure RFM without STREAM's bridge $\times$ anisotropic-decoder design actually performs worse than ZoomLDM in both reconstruction and generation (\cref{tab:ablation}, 45k decoder steps). STREAM also outperforms the VFM-latent baselines RAE and SVG, and because the UNI encoder is held fixed across the table, that margin belongs to the stochastic-bridge DiT and the anisotropic decoder rather than to the latent space. Furthermore, STREAM at half NFE ($\mathrm{NFE}{=}26$) still beats RAE and SVG at the full NFE ($\mathrm{NFE}{=}50$) (\cref{app:nfe}).

STREAM is not state of the art on reconstruction. SVG attains the best rFID on both datasets, and on COADREAD both PixCell and ZoomLDM also reconstruct better than STREAM does. However, that advantage does not carry over to generation, where these models have higher gFID than STREAM. STREAM therefore combines state-of-the-art generation with reasonable reconstruction, as generative models should strive to achieve high generation fidelity first, then aim for a minimal reconstruction trade-off. The histopathology-aware FvD, KvD and vMMD do not always rank methods the same way as the ImageNet-pretrained FID, KID and CMMD, so we report both families. Regardless, STREAM ranks first on nearly all of the generative metrics. The same holds on the independent non-TCGA SPIDER-skin dataset, where all six methods are retrained from scratch on UNI (\cref{app:spider_skin}).

We further evaluate STREAM-generated images for biological fidelity in two additional ways. First, we measure nuclear morphometry directly on the generated images, and STREAM reproduces the quantitative nuclear morphometry metrics of real tissue. Second, we test the prospective downstream application of STREAM, which is data augmentation for training self-supervised foundation models. A proof-of-concept encoder pretrained only on STREAM-generated samples of BRCA and COADREAD is on par with one pretrained on real tiles on clinically relevant downstream classification and gene expression prediction tasks (\cref{app:downstream}).

\begin{figure}[!ht]
  \centering
  \setlength{\tabcolsep}{1.5pt}
  \renewcommand{\arraystretch}{1.05}
  \begin{tabular}{@{}>{\centering\arraybackslash}m{0.035\linewidth}@{\hspace{3pt}}*{6}{>{\centering\arraybackslash}m{0.1533\linewidth}}@{}}
     & ZoomLDM & PixCell & STREAM (Ours) & RAE & SVG & REPA-E \\
    \rotatebox{90}{\small BRCA} &
      \includegraphics[width=\linewidth]{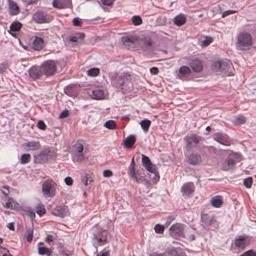} &
      \includegraphics[width=\linewidth]{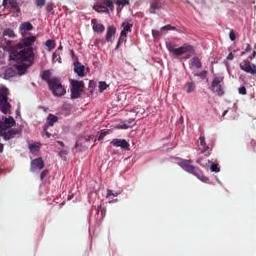} &
      \includegraphics[width=\linewidth]{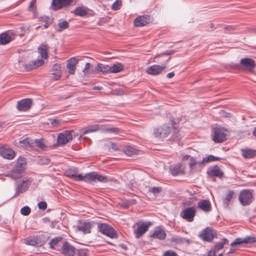} &
      \includegraphics[width=\linewidth]{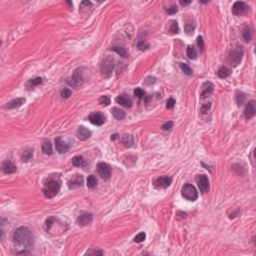} &
      \includegraphics[width=\linewidth]{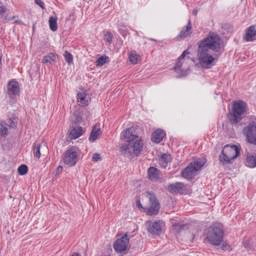} &
      \includegraphics[width=\linewidth]{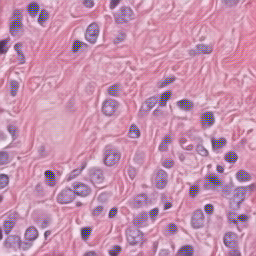} \\
    \rotatebox{90}{\small COADREAD} &
      \includegraphics[width=\linewidth]{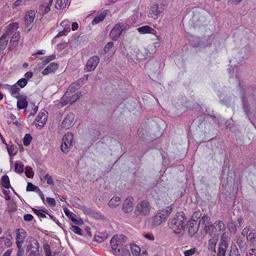} &
      \includegraphics[width=\linewidth]{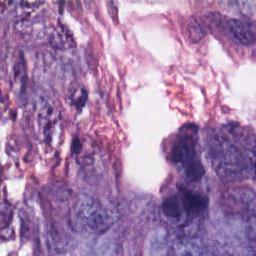} &
      \includegraphics[width=\linewidth]{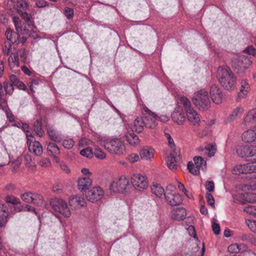} &
      \includegraphics[width=\linewidth]{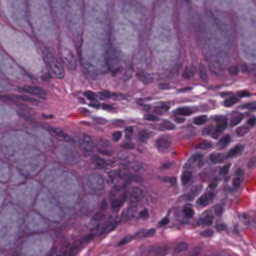} &
      \includegraphics[width=\linewidth]{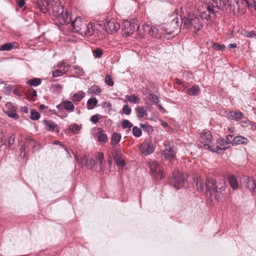} &
      \includegraphics[width=\linewidth]{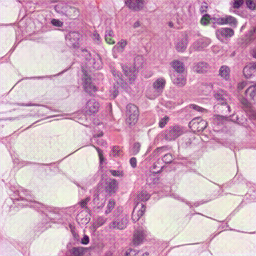} \\
  \end{tabular}
  \caption{Generation samples from each trained model on BRCA and COADREAD.}
  \label{fig:gen_gallery}
\end{figure}

For a qualitative check, \cref{fig:gen_gallery} shows generation samples from ZoomLDM, PixCell, STREAM, RAE, SVG and REPA-E, selected by K-means clustering of ground-truth (GT) images in the Virchow2 feature space; for each cluster, we display the per-model sample whose Virchow2 feature is closest to the cluster centroid. Consistent with the gFID values, \cref{fig:gen_gallery} shows that REPA-E and PixCell exhibit visible artifacts and structural features that do not represent real histopathology tissue, while the other baseline models produce globally plausible tissue but with visible color and texture artifacts when representing the stromal region of the BRCA image and the glandular region of the COADREAD image. In contrast, STREAM produces realistic-looking histopathology images with well-defined nuclei morphology and accurate textural, structural, and color detail. Additional reconstruction and generation samples for all six models are provided in \cref{app:image_gallery}.

\section{Ablation Studies}
\label{sec:ablation}

\paragraph{Ablation on STREAM design.} All ablation studies are performed on the BRCA dataset with identical training hyperparameters, except that the anisotropic decoder is trained for 45k steps. \Cref{tab:ablation} isolates the two core STREAM contributions. Anisotropic noise alone (without the bridge) yields a meaningful improvement: rFID drops from $6.51$ to $5.02$ and gFID from $9.07$ to $8.27$, consistent with allocating noise budget away from the high-response directions $U_H$. The bridge alone with isotropic noise marginally degrades gFID ($9.07 \to 9.37$) while leaving rFID unchanged. Yet the full combination produces a superadditive effect: rFID drops to $3.52$ and gFID to $6.86$, far exceeding the sum of the individual contributions. STREAM's superadditive design ties the decoder's noise to the left-singular basis of the velocity-field Jacobian, which reverses RAE's isotropic-noise rFID--gFID asymmetry~\citep{zheng2025rae} and delivers the joint improvement of \cref{tab:ablation}. Diffusion in high-dimensional latents is commonly held to trade generation for reconstruction~\citep{yao2025vavae}, yet STREAM generates directly in UNI's $1024$-dimensional VFM latent space and still improves both metrics at once.

\begin{table}[t]
  \centering
  \begin{minipage}[t]{0.40\textwidth}
    \centering
    \caption{Ablation on BRCA (UNI encoder) at 45k decoder steps.}
    \label{tab:ablation}
    \footnotesize
    \setlength{\tabcolsep}{4pt}
    \begin{tabular}{@{}cccc@{}}
      \toprule
      Bridge & Decoder & rFID $\downarrow$ & gFID $\downarrow$ \\
      \midrule
      \ding{55}        & Isotropic            & 6.51          & 9.07          \\
      \ding{55}        & Anisotropic (Ours)   & 5.02          & 8.27          \\
      \ding{51} (Ours) & Isotropic            & 6.51          & 9.37          \\
      \ding{51} (Ours) & Anisotropic (Ours)   & \textbf{3.52} & \textbf{6.86} \\
      \bottomrule
    \end{tabular}
  \end{minipage}\hfill
  \begin{minipage}[t]{0.57\textwidth}
    \centering
    \caption{Anisotropic decoder directional preference on BRCA, $n{=}2{,}000$ tiles.}
    \label{tab:panel_d}
    \footnotesize
    \setlength{\tabcolsep}{4pt}
    \begin{tabular}{@{}lcc@{}}
      \toprule
      $R_\text{dec}$ & Isotropic & Anisotropic (Ours) \\
      \midrule
      STREAM ($k{=}64$)                 & $0.42$ & $3.41$ \\
      STREAM ($k^\ast{=}236$)           & $1.12$ & $2.31$ \\
      No bridge ($k{=}64$)              & $0.62$ & $2.50$ \\
      No bridge ($k^\ast{=}232$)        & $1.28$ & $2.01$ \\
      \bottomrule
    \end{tabular}
  \end{minipage}
\end{table}

The superadditive effect has an empirical explanation in terms of directional robustness. Given (CT) of \cref{prop:approx_transport}, the bridge's transport perturbation is asymptotically $O(\smax^2)$, and by \cref{thm:rectifiability}(iii) the resulting marginal velocity field is smooth at intermediate $t$, yet the bridge changes the spectral structure that the decoder's cache is subsequently extracted from; under an isotropic decoder, which has no directional structure to exploit, that change does not help and gFID drifts from $9.07$ to $9.37$. The anisotropic decoder turns that change into a gain by spending its robustness budget on the low-response subspace $U_L$ ($\sigma_L = 0.02$) while holding $\sigma_H = 0.002$ on $U_H$. We do not claim that the generator's latent error concentrates in $U_L$ (\cref{app:residual}). In both subspaces the decoder trains under more noise than the generation error it later has to absorb. The directional-preference probe below shows which subspace the decoder ends up most sensitive to.

\paragraph{Anisotropic decoder directional preference.}
At the decoder level, we directly test the spectral mechanism via a per-token directional Lipschitz contrast on $2{,}000$ BRCA tiles, run under two partitions, $k{=}64$ and each model's own $\tau{=}0.90$ energy cutoff. For each image we sample one unit-norm tangent direction in $U_H$ and one in the orthogonal complement $U_L$ (within the cached per-(centroid, token) basis), perturb $z$ at matched per-token tangent norm $\epsilon \in \{0.005, 0.01, 0.016, 0.02\}$, decode, and measure LPIPS sensitivity. The probe range brackets the decoder's training noise envelope. The anisotropic decoder trains with $\sigma_L = 0.02$ as the maximum low-energy magnitude, so the largest probe matches the training scale and the smaller values stay strictly inside it. Per-image polyfit slopes give per-direction Lipschitz constants $L_{U_H}^\text{dec}, L_{U_L}^\text{dec}$; their ratio $R_\text{dec} = L_{U_H}^\text{dec}/L_{U_L}^\text{dec}$ characterizes the decoder's directional preference, with values $>1$ indicating $U_H$-preference and $<1$ indicating $U_L$-preference (\cref{tab:panel_d}). With or without the bridge, the isotropic decoder favors $U_L$ more than the anisotropic decoder does (\cref{tab:panel_d}), most likely because of data structure. Along-manifold tangent variations produce visually meaningful image changes that the decoder must preserve, while reconstruction largely smooths out manifold-normal variations. Anisotropic training tilts the decoder's preference toward $U_H$ in every spherical configuration we probe (\cref{tab:panel_d}). Regardless of where the isotropic decoder starts, anisotropic training raises $R_\text{dec}$: from $0.42$ to $3.41$ at $k{=}64$ ($8.1\times$) and from $1.12$ to $2.31$ at the deployed $k^\ast{=}236$ ($2.1\times$). The rise is larger with the bridge than without it at both splits, matching the superadditive gains of \cref{tab:ablation}. The same probe run on the Euclidean counterpart of RAE barely moves the ratio, so without the spherical geometry there is no $U_H/U_L$ split for the decoder to exploit (\cref{app:dirlip}). The deployed $\sigma_L = 0.02$ sits at the gFID optimum of a decoder noise ablation on BRCA (\cref{app:sigmaL}).

\paragraph{Ablation on the training loss.} To our knowledge, no prior RFM method regresses endpoints in the ambient chord metric. At the full 90k decoder schedule, STREAM's chord loss with x-prediction beats standard RFM (geodesic loss, v-prediction) on both rFID and gFID. Aligned with the current literature~\citep{li2025backtobasics}, x-prediction with the chord loss performs the best, while both mismatched pairings of loss and target give the worst gFID (\cref{app:parameterization}). The chord loss performs better because it shares the per-sample geodesic minimizer but avoids the $\arccos$ singularity (\cref{app:x_prediction}).

\paragraph{Ablation on encoder dependence.} To test whether STREAM's advantages hold on encoders other than UNI, we re-run the full pipeline with Virchow2~\citep{zimmermann2024virchow2}, another state-of-the-art histopathology VFM. With Virchow2 the anisotropic decoder gains $4.91$ rFID and $4.27$ gFID over its isotropic control (\cref{app:virchow2}). The design therefore survives a change of histopathology VFM and is not tuned to UNI's particular spectrum. However, replacing UNI with DINOv2-L, a domain-mismatched natural-image VFM, removes the isotropic-to-anisotropic gain of UNI and Virchow2 (\cref{tab:e2c}).

\begin{figure}[t]
  \centering
  \begin{minipage}[c]{0.6553\textwidth}
    \centering
    \includegraphics[width=\linewidth]{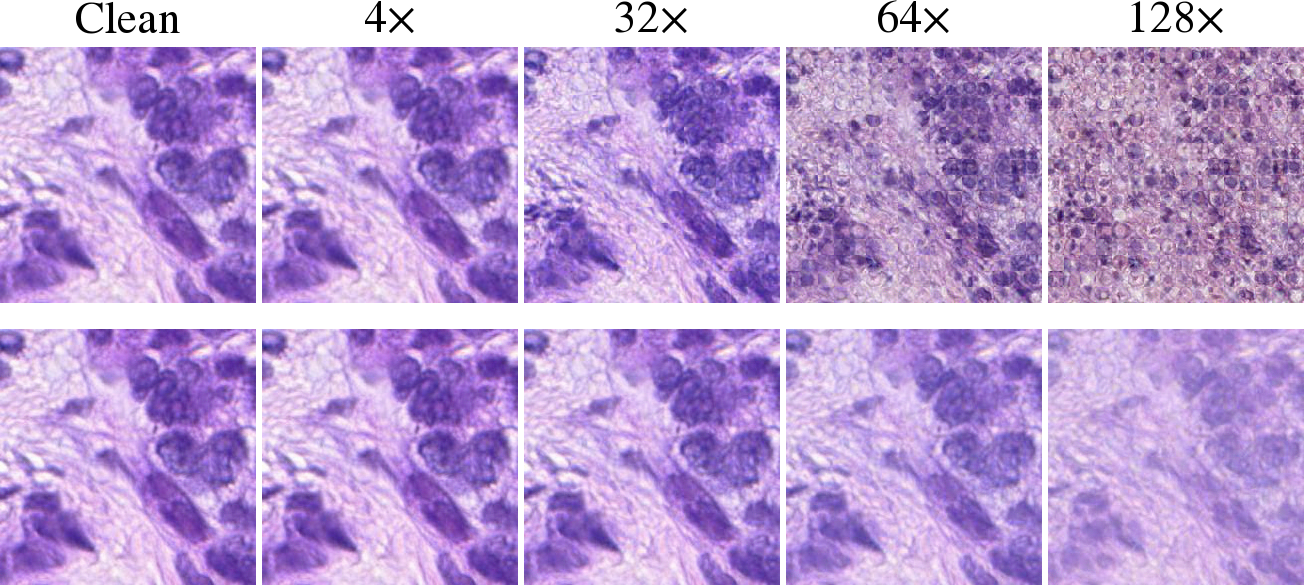}
    \caption{Pixel response (BRCA) of the anisotropic (upper row) and isotropic (lower row) decoders to a top-$64$ subspace perturbation for the clean latent and four amplitudes.}
    \label{fig:dirlip_pixel_top64_main}
  \end{minipage}\hfill
  \begin{minipage}[c]{0.3297\textwidth}
    \centering
    \captionof{table}{Encoder-dependence ablation on BRCA: domain-mismatched DINOv2-L.}
    \label{tab:e2c}
    \vspace{2pt}
    \footnotesize
    \setlength{\tabcolsep}{4pt}
    \begin{tabular}{@{}lcc@{}}
      \toprule
      Decoder            & rFID $\downarrow$ & gFID $\downarrow$ \\
      \midrule
      Isotropic          & 7.67              & 11.05             \\
      Anisotropic (Ours) & 7.71              & 11.22             \\
      \bottomrule
    \end{tabular}
  \end{minipage}
\end{figure}
This is most likely because the anisotropic decoder helps only for encoder--generator combinations that exhibit (i)~cross-centroid stability of the Jacobian's dominant singular subspace and (ii)~a non-degenerate effective rank. UNI satisfies both conditions, while DINOv2-L does not. Both conditions are measured in \cref{app:encoder_quality_dependence}: UNI's cached Jacobian is stable across centroids, and its $k^*$ distribution contains no degenerate (centroid, token) cells, which DINOv2-L does produce. Both quantities are read off the cached Jacobian SVD, so they can be screened in advance for a new encoder. A DiT trained on DINOv2-L still generates distributionally correct latents (\cref{app:e2d}), so the DINOv2-L failure lies in the decoder rather than the generator. An in-domain histopathology VFM is therefore the most reliable route to meeting both conditions (\cref{app:encoder_scope}). Under a perturbation inside the top-$64$ subspace ($U_H$ at $k{=}64$), the anisotropic decoder separates from the isotropic one and responds more strongly (\cref{fig:dirlip_pixel_top64_main}), a pixel-space view of the decoder sensitivity of \cref{tab:panel_d} behind the superadditive effect. As the perturbation amplitude grows from $4\times$ to $128\times$, the anisotropic decoder's output changes visibly from $32\times$ onward while the isotropic decoder's changes far less; the anisotropic decoder tracks latent movement along $U_H$ instead of smoothing it out.

\section{Conclusion}
\label{sec:conclusion}

We present STREAM, an RFM framework for unconditional histopathology generation in VFM latent spaces, pairing a stochastic bridge on $\cS^{d-1}$ with a spectral-informed anisotropic decoder. The two components combine superadditively, and under a single shared UNI encoder STREAM attains state-of-the-art generation quality on three different histopathology datasets. Main limitations that remain are that \cref{thm:rectifiability} holds at the population level and that our account of the superadditive mechanism is mostly empirical (\cref{app:residual}). Future work includes transferring STREAM to other imaging modalities, extending the method to conditional generation, and evaluating STREAM-generated images far more extensively, at the level of an actual clinical setting.

\clearpage
\begin{ack}
This work was supported by the Technology Innovation Program
(RS-2025-02221011, Development of Medical-Specialized Multimodal Hyperscale Generative AI Technology for Global Integration)
funded by the Ministry of Trade Industry \& Energy (MOTIE, South Korea).
\end{ack}

\bibliographystyle{plainnat}
\bibliography{main}

\clearpage
\section*{Mathematical Notation}

Symbols carried through the main text. A symbol marked \textsuperscript{\dag} is overloaded within
the main text; both senses are given in the same row. Symbols confined to a single appendix are
defined where they appear and are not repeated here.

{\footnotesize
\setlength{\tabcolsep}{4pt}
\noindent
\begin{tabular}{@{}>{\raggedright\arraybackslash}p{0.19\linewidth} >{\raggedright\arraybackslash}p{\dimexpr0.81\linewidth-2\tabcolsep\relax}@{}}
\toprule
\textbf{Symbol} & \textbf{Meaning} \\
\midrule
$\cS^{d-1}$, $(\cS^{d-1})^N$ & Unit hypersphere $\{x \in \mathbb{R}^d : \|x\|=1\}$, the VFM latent manifold, and the product of $N$ per-token spheres carrying $g = \bigoplus_n g_n$ (\cref{sec:method,sec:setup}) \\
$d$, $N$, $n$\textsuperscript{\dag} & Embedding dimension per token ($d{=}1024$ for UNI), token count per image ($N{=}256$), so the total tangent dimension is $N(d-1) = 261{,}888$. $n$ is the token index and, in \cref{eq:decoder_loss}, the ambient noise vector added to a token (\cref{sec:setup}) \\
$\Exp_x$, $\Log_x$, $\dg$, $\Omega$ & Riemannian exponential map, logarithmic map and geodesic distance on $\cS^{d-1}$, and the endpoint angle $\Omega = \dg(x_0,x_1) \le \pi$ of the geodesic conditional path (\cref{app:rprim,eq:slerp}) \\
$t$\textsuperscript{\dag} & Flow time, $t{=}0$ at the source and $t{=}1$ at the data; also the fixed anchor $t{=}0.99$ at which the spectral basis is cached (\cref{def:bridge,sec:aniso}) \\
$p_0$, $p_1$, $p_t$ & Source (uniform), data and interpolating distributions on $(\cS^{d-1})^N$; $p_\text{data} = E_\#(p_\text{real})$ and $p_\theta$ is the generated feature distribution (\cref{sec:method}) \\
$\pi$\textsuperscript{\dag} & Source--data coupling ($p_0 \otimes p_1$ in \cref{def:bridge}, a general coupling in \cref{prop:full_support}); also the circle constant, in $\sin(\pi t)$ and the sphere diameter (\cref{def:bridge}) \\
$x_0$, $x_1$, $\mu_t$, $x_t$\textsuperscript{\dag}, $X_t$ & Source and data endpoints on one sphere, the SLERP interpolant, the bridge-perturbed interpolant, and that interpolant as a random variable; in \cref{eq:rfm_loss} $x_t$ is instead the unperturbed SLERP point (\cref{eq:bridge}) \\
$\epsilon$\textsuperscript{\dag}, $\epsilon_t$, $\Pi_\mu$ & Tangent Gaussian noise of the bridge, $\epsilon \sim \cN(0,\Pi_{\mu_t})$ with projector $\Pi_\mu = I - \mu\mu^\top$; $\epsilon$ is also the tangent probe magnitude of the directional-Lipschitz test, and is distinct from the time clamp $\epsilon_t = 10^{-5}$ (\cref{eq:bridge}) \\
$\sigma(t)$, $\smax$ & Bridge noise schedule $\smax\sin(\pi t)$ and its peak amplitude, deployed $\smax = 0.01$ (\cref{def:bridge}) \\
$u_t$, $v_t$, $v_\theta$ & Conditional velocity along the known path, the marginal velocity field it induces, and the learned velocity field (\cref{sec:method}) \\
$E$, $G_\theta$, $D_\phi$, $\cX$ & Frozen VFM encoder $\cX \to (\cS^{d-1})^N$, DiT generator, ViT decoder $(\cS^{d-1})^N \to \cX$, and the image space $\cX = \mathbb{R}^{H \times W \times 3}$ (\cref{sec:setup,sec:decoder_training}) \\
$\theta$\textsuperscript{\dag}, $\phi$ & Generator and decoder parameters; $\theta$ is also the geodesic angle argument of $\Log_x$ (\cref{eq:rfm_loss,eq:log_map}) \\
$f_\theta$, $\hat x_1$ & The DiT read as an endpoint (x-prediction) network, and its normalized output $f_\theta(x_t,t)/\|f_\theta(x_t,t)\|$ (\cref{eq:vloss_chord}) \\
$\bZ_0$, $\bZ_1$, $\hat\bZ_1$ & All-token feature tensors on $(\cS^{d-1})^N$: source, clean data, and generated (\cref{sec:setup}) \\
$z$\textsuperscript{\dag}, $\tilde z$, $z_\text{rt}$ & A single clean VFM token on $\cS^{d-1}$, the noised and renormalized token $(z+n)/\|z+n\|$, and its round-trip re-encoding $E(D_\phi(\tilde z))$. Latents are $z$ and generic points of $\cS^{d-1}$ are $x$; $z$ also names the conditioning triple $(x_0,x_1,\epsilon)$ of the generic CFM loss (\cref{eq:decoder_loss}) \\
$I$\textsuperscript{\dag} & Input training image, reconstructed as $\hat I = D_\phi(\tilde z)$; also the identity matrix, and the mutual information $I(X;C)$ (\cref{sec:setup}) \\
$\cL_\text{RFM}$, $\cL_\text{chord}$, $\cL_\text{dec}$ & RFM loss, the deployed chord loss on normalized endpoints, and the four-component decoder loss with weights $\lambda_{\ell_1}$, $\lambda_\text{LPIPS}$, $\lambda_\text{cos}$, $\omega_G\lambda_\text{adapt}$ (\cref{eq:rfm_loss,eq:vloss_chord,eq:decoder_loss}) \\
$J$ & Per-token tangent-projected Jacobian of the velocity field, $\nabla_z v_\theta(z,t)$ (\cref{sec:aniso}) \\
$U_H$, $U_L$ & High-response (top-$k^*$) and low-response left singular subspaces of $J$; $H$ and $L$ label the two throughout. Unrelated to the $U[0,1]$ of \cref{eq:rfm_loss} (\cref{sec:aniso}) \\
$s_i$, $\tau$, $k$, $k^*$ & Singular values of $J$; the spectral energy threshold, deployed $0.90$; and the retained-dimension count $k^* = \min\{k : \sum_{i\le k}s_i^2/\sum_i s_i^2 \ge \tau\}$. $k$ is a generic retained-subspace dimension ($k{=}64$ in the alignment and Lipschitz probes); $k^*$ is the count that $\tau$ fixes, not the cutoff itself; it is taken on the median spectrum of the cache and is model-specific ($236$ deployed for STREAM, $232$ without the bridge, $380$ for RAE, $280$ for DINOv2-L). Singular values are written $s_i$, never $\sigma_i$ (\cref{sec:aniso}) \\
$\Sigma_\text{noise}(z)$, $\sigma_H$, $\sigma_L$ & Anisotropic decoder noise covariance $\sigma_H^2 U_HU_H^\top + \sigma_L^2 U_LU_L^\top$ and its two magnitudes, deployed $0.002$ and $0.02$ (\cref{eq:aniso_cov}) \\
$K$, $m$, $\bar z_c$ & Number of $k$-means centroids of the spectral cache ($K{=}50$), random tangent probes per (centroid, token) ($m{=}600$), and the $c$-th centroid. $K$ separately denotes the Vendi similarity kernel (\cref{sec:aniso}) \\
$R_\text{eff}$, $R_\text{dec}$, $A$ & Effective rank of the feature covariance, which is not the Jacobian cut $k^*$; the decoder's directional-Lipschitz ratio $L_{U_H}^\text{dec}/L_{U_L}^\text{dec}$; and the principal-angle subspace alignment $\lVert Q_1^\top Q_2\rVert_F^2/k \in [0,1]$ (\cref{sec:intrinsic_geometry,sec:ablation,app:e2a}) \\
$\rho_\text{cond}$, $H(X)$, $C$ & Conditioning dominance ratio $I(X;C)/H(X)$, the marginal output entropy, and the conditioning variable $C = E_\text{VFM}(I)$ (\cref{eq:rho_cond}) \\
$f_\text{angular}$, $\text{Acc}_\text{dir}$, $\text{Acc}_\text{mag}$ & Angular fraction of the discriminative between-class distance, and 1-NN accuracy on the normalized direction alone and on the norm $\|z\|$ alone (\cref{app:perturbation_protocol}) \\
\bottomrule
\end{tabular}}

\clearpage
\appendix
\crefalias{section}{appendix}\crefalias{subsection}{subappendix}\crefalias{subsubsection}{subsubappendix}

\section{Conditioning-Dominated Diversity Analysis}
\label{app:conditioning}

We provide the full conditioning-dominated-diversity analysis summarized in \cref{sec:conditioning}. Throughout, ``conditioning-dominated diversity'' names the regime in which $\rho_\text{cond}=I(X;C)/H(X)$ is close to $1$, i.e.\ almost all of a model's output diversity is attributable to variation in the conditioning signal $C$ rather than to its latent space. It is a statement about where the diversity comes from, not a claim that the models produce degenerate images.

\subsection{Entropy Estimation via Vendi Score}

For VFM-conditioned models, the output diversity decomposes as $H(X) = H(X \mid C) + I(X; C)$, where $C = E_\text{VFM}(I)$ is the conditioning embedding. We estimate these quantities using the Vendi Score of order 1~\citep{friedman2023vendi}, which computes $\text{Vendi}_1(K) = \exp(H_1(\lambda(K)))$ where $\lambda(K)$ are the normalized eigenvalues of the similarity kernel matrix $K$ and $H_1$ is Shannon entropy. For a set of samples $\{X_j\}$, $\log \text{Vendi}_1$ estimates the Shannon entropy of the kernel-induced output distribution.

\textbf{Key property.} Because both baselines are sampled with deterministic samplers that inject no noise during the reverse process, the only source of randomness is the initial latent noise $z_T \sim \cN(0, I)$. The entropy we estimate is therefore model-induced diversity (the pushforward entropy of $\cN(0, I)$ through $f_\theta(\cdot, c)$), not sampler noise. This avoids step-wise noise inflation and confounding stochasticity from the sampling procedure.

\textbf{Important caveats.} The quantities we estimate are not:
(i)~the true data entropy $H(p_\text{data}(X \mid C))$;
(ii)~the training objective entropy; or
(iii)~Shannon mutual information $I(X; C)$ in pixel space.
Rather, we estimate the conditional entropy of the generative model's output distribution, under a fixed inference procedure, measured in VFM (UNI) embedding space. This is a valid and informative diagnostic because UNI embeddings are approximately information-preserving for tissue-level semantics, which is the relevant notion of diversity in histopathology. The difference $I(X; C) = H(X) - H(X \mid C)$ is a kernel-induced Shannon mutual information, which is monotonic in diversity and comparable across models evaluated with the same kernel and embedding space.

\subsection{Measurement Protocol and Discussion}

\textbf{Step 1: Clustering.} We apply spherical $k$-means on BRCA UNI embeddings (${\sim}500$k samples) to obtain conditioning clusters. We evaluate $K \in \{5, 8, 10, 12, 15, 20, 30\}$, selecting $K^* = 10$ via the elbow method with Adjusted Rand Index stability $\geq 0.90$ across bootstrap resamples. Cluster weights $w_k = |C_k| / N_\text{total}$ reflect the empirical conditioning distribution.

\textbf{Step 2: Conditional entropy $H(X \mid C)$.} For each cluster $k$, we sample conditioning vectors $c_{k,m}$ and generate multiple images per condition by varying $z_T$. The per-condition Vendi Score estimates $H(X \mid C = c_{k,m})$. Averaging over conditions with cluster weights gives $\hat{H}(X \mid C) = \sum_k w_k \cdot \mathbb{E}_{m}[H(X \mid C = c_{k,m})]$.

\textbf{Step 3: Marginal entropy $H(X)$.} We pool all generated samples across all conditions and compute a single Vendi Score on the full set. This approximates $H(X)$ via Monte Carlo marginalization over $C$.

\textbf{Step 4: Mutual information.} $I(X; C) = H(X) - H(X \mid C)$, normalized as $\rho_\text{cond} = I(X; C) / H(X)$.

\textbf{Step 5: Inter-cluster spectral diversity.} We compute SPEC-diff~\citep{jalali2025specdiff}, the operator norm of the difference between per-cluster and overall kernel matrices, to quantify inter-cluster spectral diversity. Low CVS combined with low SPEC-diff is a signature of conditioning-dominated diversity.

\textbf{Discussion.} The decomposition $H(X) = H(X\mid C) + I(X;C)$ separates domination from desirable conditioning dependence directly. Desirable dependence means the model uses $C$ and still produces varied outputs once $C$ is held fixed, so $H(X\mid C)$ stays large and $\rho_\text{cond}$ small. Domination means nearly all the diversity is attributable to $C$ varying, so $H(X\mid C)$ falls away and $\rho_\text{cond}$ approaches $1$. \Cref{tab:collapse} measures the latter on BRCA in UNI space: $\rho_\text{cond}$ is $75\%$ for ZoomLDM and $62\%$ for PixCell, with $H(X\mid C)$ of $1.28$ and $1.81$ against totals of $5.17$ and $4.78$. Equivalently, latent utilization is $1-\rho_\text{cond} = H(X\mid C)/H(X)$, i.e.\ $25\%$ and $38\%$, and output diversity at fixed condition is $H(X\mid C)$; both quantities are already in \cref{tab:collapse} under other names.

Removing $C$ separates the two regimes operationally. PixCell cannot sample without being handed a VFM embedding, and ZoomLDM only sidesteps that by learning a second generative model over the embedding. Forced to sample from its own latent space, PixCell reaches $104.18$ gFID on BRCA against $6.16$ for STREAM (\cref{tab:results}), so its latent space alone does not support generation.

\paragraph{STREAM's marginal-entropy check.} STREAM is unconditional, so $\rho_\text{cond}$, $I(X;C)$ and $H(X\mid C)$ do not apply to it; $H(X)$ needs only samples, under the marginal-entropy estimator of Step~3 above. STREAM's $H(X) = 5.26$ from noise alone exceeds either baseline's total of $5.17$ and $4.78$: it generates more variety with no conditioning signal than they achieve with their condition free to vary.

\subsection{Linear Probing: SPIDER-breast Dataset}
\label{app:linprob}

The linear probing evaluation uses SPIDER-breast~\citep{nechaev2025spider} with an 11-class binary split into Benign/Proliferative (6 classes, 26{,}664 patches) and Malignant (5 classes, 27{,}164 patches); class-level breakdown in \cref{tab:spider_classes}. This probe measures the encoder's representation quality and is distinct from the downstream evaluation of \cref{app:downstream}, which probes encoders pretrained from scratch on STREAM-generated images. A linear classifier is trained on frozen embeddings from each model/encoder; AUROC with a 95\% bootstrap confidence interval (CI) is reported in \cref{tab:linprob_results}.

\begin{table}[!h]
  \begin{minipage}[t]{0.49\linewidth}
    \centering
    \caption{SPIDER-breast class composition. DCIS is ductal carcinoma in situ and NST is no special type.}
    \label{tab:spider_classes}
    \footnotesize
    \setlength{\tabcolsep}{4pt}
    \begin{tabular}{@{}l r@{}}
      \toprule
      Class & Patches \\
      \midrule
      \multicolumn{2}{@{}l}{\textit{Benign/Proliferative (6 classes)}} \\
      Adenosis                       & 2{,}899  \\
      Sclerosing adenosis            & 3{,}423  \\
      Fibrocystic changes            & 5{,}027  \\
      Typical ductal hyperplasia     & 5{,}546  \\
      Fibroadenoma                   & 5{,}243  \\
      Benign phyllodes               & 4{,}526  \\
      \emph{Subtotal}                & \emph{26{,}664} \\
      \midrule
      \multicolumn{2}{@{}l}{\textit{Malignant (5 classes)}} \\
      DCIS low-grade                 & 5{,}017  \\
      DCIS high-grade                & 5{,}632  \\
      Invasive NST carcinoma         & 6{,}142  \\
      Lobular invasive carcinoma     & 5{,}102  \\
      Malignant phyllodes            & 5{,}271  \\
      \emph{Subtotal}                & \emph{27{,}164} \\
      \midrule
      \textbf{Total}                 & \textbf{53{,}828} \\
      \bottomrule
    \end{tabular}
  \end{minipage}\hfill
  \begin{minipage}[t]{0.49\linewidth}
    \centering
    \caption{Linear-probe AUROC (95\% bootstrap CI) on SPIDER-breast.}
    \label{tab:linprob_results}
    \footnotesize
    \setlength{\tabcolsep}{4pt}
    \begin{tabular}{@{}l c c@{}}
      \toprule
      Encoder & $d$ & AUROC [95\% CI] \\
      \midrule
      UNI2-h                     & 1536 & \textbf{0.997} [0.996, 0.998] \\
      UNI                        & 1024 & 0.995 [0.993, 0.996]          \\
      DINOv3-L                   & 1024 & 0.937 [0.931, 0.942]          \\
      \midrule
      PixCell VAE                & 16   & 0.640 [0.628, 0.653]          \\
      ZoomLDM VAE                & 3    & 0.624 [0.611, 0.637]          \\
      \bottomrule
    \end{tabular}
  \end{minipage}
\end{table}

\subsection{1-NN Direction vs.\ Magnitude Perturbation Protocol}
\label{app:perturbation_protocol}

This subsection details the protocol behind the $\text{Acc}_\text{dir}$, $\text{Acc}_\text{mag}$, and $f_\text{angular}$ columns of \cref{tab:perturbation}, as referenced from the angular-dominance argument in \cref{sec:angular}. The setup is shared with \cref{app:linprob}: all five models are evaluated on the SPIDER-breast binary task (benign/proliferative vs.\ malignant; $N_\text{test} = 6{,}839$) using frozen embeddings, with no fine-tuning.

\textbf{1-NN leave-one-out (LOO) baseline.}  For each test embedding $z_i$ with label $y_i$, we find its nearest same-class and nearest different-class neighbor in the test set under Euclidean distance (excluding $z_i$ itself), and count a correct hit iff $d(z_i, \text{NN}_\text{same}) < d(z_i, \text{NN}_\text{diff})$.  We report three accuracies on the same test split:
\begin{itemize}
    \item $\text{Acc}_\text{full}$: 1-NN on the raw embedding $z$.
    \item $\text{Acc}_\text{dir}$ (\cref{tab:perturbation}): 1-NN on the $\ell_2$-normalized direction $\hat{z} = z/\|z\|$, isolating the angular component.
    \item $\text{Acc}_\text{mag}$ (\cref{tab:perturbation}): 1-NN on the 1-D scalar $\|z\|$, isolating the radial component.
\end{itemize}
The gap $\text{Acc}_\text{full} - \text{Acc}_\text{dir}$ is a direct estimate of how much classification value the magnitude adds beyond direction; for all three VFMs this gap is below $0.005$, i.e.\ direction alone reproduces the full 1-NN accuracy.

\textbf{Variance decomposition of between-class distance.}  For each pair $(z_i, z_j)$ with $y_i \neq y_j$, the squared Euclidean distance factors as
\begin{equation}
    \|z_i - z_j\|^2 \;=\; \underbrace{(\|z_i\| - \|z_j\|)^2}_{\Delta_\text{radial}(i,j)} \;+\; \underbrace{2\,\|z_i\|\,\|z_j\|\,(1 - \cos \theta_{ij})}_{\Delta_\text{angular}(i,j)},
\label{eq:variance_decomp}
\end{equation}
where $\theta_{ij}$ is the angle between $z_i$ and $z_j$.  Averaging the two components over the between-class pairs in the test set yields the angular fraction of discriminative distance
\begin{equation}
    f_\text{angular} \;=\; \frac{\overline{\Delta_\text{angular}}}{\overline{\Delta_\text{angular}} + \overline{\Delta_\text{radial}}},
\label{eq:fangular}
\end{equation}
which we report in \cref{tab:perturbation}. All three VFMs satisfy $f_\text{angular} > 0.97$; the only model well below $1.0$ is ZoomLDM VAE at $0.813$. This is consistent with its $32.4\%$ coefficient of variation of $\|z\|$ and its $d{=}3$ representation, in which the radial direction carries a meaningful share of the total information budget.

\textbf{Angular perturbation stress test.}  To probe whether the 1-NN signal genuinely lives in the angular structure, we rotate each $\hat{z}_i$ in a random unit tangent direction by an angle $\theta$ and recompute 1-NN LOO after restoring each original norm $\|z_i\|$, so that $\cos\theta = 1.00$ returns $\text{Acc}_\text{full}$.  We use two methods, additive Gaussian followed by re-normalization and exponential-map tangent rotation, and report the tangent method (the two agree within $0.01$).  We sweep $\cos\theta \in \{1.00, 0.80, 0.50, 0.30, 0.10, 0.00\}$ over three random seeds; \cref{tab:angular_sweep} reports the means.  UNI retains $96\%$ accuracy up to $\cos\theta = 0.50$ (a $60^\circ$ rotation), UNI2-h follows the same trend with a slightly steeper drop, DINOv3-L degrades smoothly from $0.929$ down to chance, PixCell VAE collapses to chance by $\cos\theta = 0.80$, and ZoomLDM VAE falls to $0.592$ there and plateaus at $0.583$, above chance throughout.  The VFM clusters thus enjoy angular separation on the order of tens of degrees. This is compatible with the bridge perturbation in \cref{sec:bridge}, which operates well inside that margin.

\textbf{Radial perturbation stress test.}  We replace $\|z_i\|$ by $\|z_i\| \cdot (1 + \tau \cdot \varepsilon_i)$ with $\varepsilon_i \sim \cN(0, 1)$ i.i.d.\ and re-evaluate 1-NN LOO; we sweep $\tau \in \{0.0, 0.2, 1.0, 2.0, 5.0, 10.0\}$ over three seeds (\cref{tab:radial_sweep}).  UNI and UNI2-h hold accuracy $\geq 0.98$ even at $\tau = 10$ (a $\pm 1000\%$ norm perturbation).  A control check normalizes the perturbed embeddings before computing 1-NN; under that normalization the radial-perturbation curves are perfectly flat at $\text{Acc}_\text{dir}$ for every model and every $\tau$, confirming that the radial noise is entirely undone by re-projection onto $\cS^{d-1}$.

\begin{table}[h]
  \centering
  \caption{Angular perturbation stress test on SPIDER-breast (1-NN LOO accuracy under tangent rotation, mean over 3 seeds; $\cos\theta = 1.00$ recovers the unperturbed $\text{Acc}_\text{full}$).  Chance level is $0.50$.}
  \label{tab:angular_sweep}
  \footnotesize
  \setlength{\tabcolsep}{6pt}
  \begin{tabular}{@{}l cccccc@{}}
    \toprule
    Model & $\cos\theta{=}1.00$ & $0.80$ & $0.50$ & $0.30$ & $0.10$ & $0.00$ \\
    \midrule
    UNI          & 0.993 & 0.992 & 0.960 & 0.674 & 0.521 & 0.506 \\
    UNI2-h       & 0.994 & 0.993 & 0.894 & 0.624 & 0.502 & 0.492 \\
    DINOv3-L     & 0.929 & 0.720 & 0.549 & 0.491 & 0.490 & 0.490 \\
    PixCell VAE  & 0.834 & 0.507 & 0.496 & 0.497 & 0.498 & 0.500 \\
    ZoomLDM VAE  & 0.726 & 0.592 & 0.585 & 0.587 & 0.586 & 0.583 \\
    \bottomrule
  \end{tabular}
\end{table}

\begin{table}[h]
  \centering
  \caption{Radial perturbation stress test on SPIDER-breast (1-NN LOO accuracy under $\|z\| \mapsto \|z\|(1 + \tau\varepsilon)$, $\varepsilon \sim \cN(0,1)$, mean over 3 seeds; $\tau = 0$ recovers $\text{Acc}_\text{full}$).  UNI and UNI2-h hold $\geq 0.98$ at $\tau = 10$.}
  \label{tab:radial_sweep}
  \footnotesize
  \setlength{\tabcolsep}{6pt}
  \begin{tabular}{@{}l cccccc@{}}
    \toprule
    Model & $\tau{=}0.0$ & $0.2$ & $1.0$ & $2.0$ & $5.0$ & $10.0$ \\
    \midrule
    UNI          & 0.993 & 0.990 & 0.985 & 0.985 & 0.983 & 0.985 \\
    UNI2-h       & 0.994 & 0.992 & 0.989 & 0.988 & 0.989 & 0.988 \\
    DINOv3-L     & 0.929 & 0.899 & 0.836 & 0.822 & 0.818 & 0.820 \\
    PixCell VAE  & 0.834 & 0.759 & 0.698 & 0.686 & 0.682 & 0.683 \\
    ZoomLDM VAE  & 0.726 & 0.617 & 0.606 & 0.600 & 0.597 & 0.598 \\
    \bottomrule
  \end{tabular}
\end{table}

\textbf{Interpretation.}  The joint pattern is unambiguous: ($\text{Acc}_\text{dir} \approx \text{Acc}_\text{full}$) and ($f_\text{angular} > 0.97$) and (radial-perturbation curve flat at $\geq 0.98$ for UNI and UNI2-h) and (normalized-radial control flat at $\text{Acc}_\text{dir}$) all together imply that the class-discriminative information in all three VFMs is carried almost entirely by the angular structure on $\cS^{d-1}$.  ZoomLDM VAE is the sole counter-example: its $f_\text{angular} = 0.81$ and the fact that $\ell_2$-normalization drops its 1-NN accuracy by $11.7$~pp ($0.726 \to 0.609$) show that magnitude carries non-redundant class information when the latent budget is as small as $d{=}3$.  For the histopathology VFMs we use throughout STREAM, however, every diagnostic points to the sphere as the correct ambient manifold for generative modeling. That is exactly the structural premise the RFM formulation in \cref{sec:bridge} requires.

\subsection{Intrinsic Geometry: Extended Analysis}
\label{app:intrinsic_geometry}

\begin{figure}[t]
  \centering
  \begin{minipage}[t]{0.48\linewidth}
    \centering
    \includegraphics[width=\linewidth]{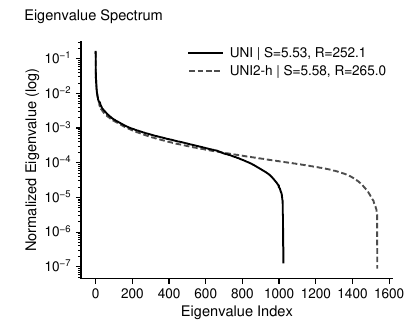}
    {\small (a) Spectral analysis}
  \end{minipage}
  \hfill
  \begin{minipage}[t]{0.48\linewidth}
    \centering
    \includegraphics[width=\linewidth]{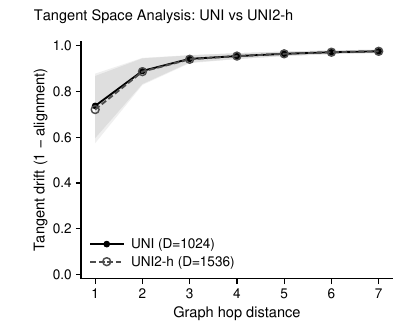}
    {\small (b) Tangent drift}
  \end{minipage}
  \caption{Intrinsic geometry of histopathology VFM feature manifolds on BRCA. (a)~Eigenvalue spectrum of the $\ell_2$-normalized feature covariance for UNI and UNI2-h, showing high covariance effective rank ($R_\text{eff} = 252$/$265$; distinct from the Jacobian-spectrum cut $k^\ast$). (b)~Tangent drift $\mathcal{D}$ vs.\ graph-hop distance, demonstrating monotonically increasing misalignment of local tangent spaces, which is direct evidence of a curved, non-flat manifold structure.}
  \label{fig:intrinsic_geometry}
\end{figure}

We extend the geometric characterization of \citet{xiong2026exploiting} to BRCA at scale (${\sim}1$M 20$\times$ patches) for UNI ($d{=}1024$) and UNI2-h ($d{=}1536$), summarized in \cref{fig:intrinsic_geometry}. Spectral analysis of the $\ell_2$-normalized feature covariance reveals a high covariance effective rank, $R_\mathrm{eff} = 252$ and $265$ respectively, with von Neumann entropy $S_1 = 5.53$ and $5.58$. This is consistent with the result of \citet{xiong2026exploiting}, $R_\mathrm{eff} = 235.5$ for UNI on Camelyon16, confirming that histopathology-specific self-supervised pretraining distributes semantic information across hundreds of dimensions. Tangent space analysis ($k{=}12$, intrinsic dimension $d_s{=}10$) shows drift rising steeply from $0.737{\pm}0.141$/$0.721{\pm}0.147$ at hop~1 to $0.943{\pm}0.016$/$0.941{\pm}0.017$ at hop~3, saturating near $0.976$ at hop~7. The sharp collapse in standard deviation (SD) beyond hop~3 ($\sigma < 0.02$) demonstrates stable, data-wide manifold curvature, substantially exceeding values reported for CONCH~\citep{lu2024conch} by \citet{xiong2026exploiting}.

\section{Mathematical Preliminaries and Complete Proofs}
\label{app:prelim}\label{app:proofs}

\Cref{app:rprim} collects established results invoked in our proofs, and \cref{app:rect_theory} the rectifiability results we compare against; the remaining subsections state and prove our own.

\subsection{Background: Riemannian Primitives on \texorpdfstring{$\cS^{d-1}$}{Sᵈ⁻¹}, OT, Curvature-Dimension, Flow Matching}
\label{app:rprim}

The proofs below invoke the Riemannian primitives of the unit hypersphere, OT on $\cS^{d-1}$, the sphere's curvature-dimension bound, and the RFM construction, each at a separate point. The geodesic distance, exponential map, and logarithmic map on $\cS^{d-1}$ supply the notation used throughout:
\begin{align}
  \dg(x, y) &= \arccos(\langle x, y\rangle), \label{eq:geodesic_dist} \\
  \Exp_x(v) &= \cos(\|v\|)\,x + \sin(\|v\|)\,\frac{v}{\|v\|}, \quad v \in T_x\cS^{d-1}, \label{eq:exp_map} \\
  \Log_x(y) &= \frac{\theta}{\sin\theta}(y - \cos\theta\,x), \quad \theta = \dg(x, y). \label{eq:log_map}
\end{align}

OT on $\cS^{d-1}$ takes this geodesic distance as its cost. The 2-Wasserstein distance $W_2^2(p_0, p_1) = \inf_{\gamma \in \Gamma(p_0, p_1)} \int d_g(x,y)^2\,d\gamma$ admits the dynamic Benamou--Brenier formulation~\citep{benamou2000computational}, which measures a transport plan by the kinetic action $\int_0^1 \mathbb{E}\|v_t\|_g^2\,dt$ of a velocity field carrying $p_0$ to $p_1$, and \cref{prop:bounded_vel} bounds that action for the bridge. On $\cS^{d-1}$, OT maps inherit smoothness from the Ma--Trudinger--Wang (MTW) condition~\citep{loeper2009regularity, ma2005regularity}, which supplies the per-factor $C^2$ control of $v_t^\textup{det}$ that the kernel-regression bias estimate of \cref{prop:approx_transport}(ii) needs, with the extension across the $N$ coupled tokens carried by hypothesis~(CT) instead. The sphere satisfies $\mathrm{CD}(d-2,\infty)$~\citep{bakry1985diffusions}, with $\Ric \ge (d-2)g$ implying $W_2$-contraction of the heat semigroup at rate $d-2$~\citep{vonrenesse2005transport}; this makes $p_0 = \mathrm{Uniform}(\cS^{d-1})$ the canonical source, and \cref{sec:setup} adopts it on the product manifold for that reason.

FM~\citep{lipman2023flow,albergo2023stochastic,liu2023flow,liu2022rectified} learns a velocity field via the CFM loss $\cL_\text{CFM} = \mathbb{E}[\|v_\theta - u_t\|^2]$; for Riemannian flows on $\cS^{d-1}$~\citep{chen2024riemannian}, the SLERP conditional path $\psi_t(x_0|x_1) = \SLERP(x_0, x_1, t)$ and conditional velocity $u_t = \Log_{x_t}(x_1)/(1-t)$ are the standard choices. \Cref{def:bridge} keeps both and displaces the path off the geodesic by tangent-Gaussian noise, so the conditional velocity bounded in \cref{prop:bounded_vel} is the expression above evaluated at the perturbed $x_t$ rather than at the SLERP point $\mu_t$.

\subsection{Rectifiability Theory}
\label{app:rect_theory}

The Euclidean rectifiability theory of \citet{hertrich2025rectified} provides a Euclidean sufficient condition, which \cref{thm:rectifiability} does not use:

\textbf{Theorem (Sufficient Rectifiability)}~\citep[Thm.~14]{hertrich2025rectified}. \emph{If $P_{X_0\mid X_1=x_1}$ is absolutely continuous with smooth positive density, then $(X_0, X_1)$ is rectifiable.} Here rectifiable carries the sense of \citet[Def.~1]{hertrich2025rectified}, under which the ordinary differential equation $\dot z_t = v_t(z_t)$ admits a unique solution. The hypothesis is a condition on the coupling's conditional law of the source given the data endpoint.

\textbf{Proposition (Disconnected-support obstruction)}~\citep[Prop.~10]{hertrich2025rectified}. \emph{There exist couplings with zero rectification loss whose intermediate marginals have disconnected support and whose rectification fixed point is not the OT coupling, invalidating earlier rectification--OT equivalence claims.}

Both results are stated and proved in $\mathbb{R}^d$; \cref{thm:rectifiability} uses neither and is proved directly on $\cS^{d-1}$ (\cref{app:proofs_bridge}). Recent rectified-flow convergence theory~\citep{bansal2024wasserstein,lee2024improving} establishes regularity-based Wasserstein bounds and training improvements analogous to ours.

The bridge perturbation does not supply the premise of \citet[Thm.~14]{hertrich2025rectified}. That premise is a condition on $P_{X_0\mid X_1 = x_1}$, and the bridge is pinned at both endpoints ($\sigma(0)=\sigma(1)=0$), so it leaves $P_{X_0\mid X_1=x_1}$ untouched; under our independent coupling $\pi = p_0\otimes p_1$ that conditional is already $\mathrm{Uniform}(\cS^{d-1})$, so the premise holds for a reason unrelated to the bridge. \Cref{thm:rectifiability} does not depend on it in any case. \Cref{lem:full_support_cond}, which also makes the Bayes denominator positive, gives the smooth conditional density of $(X_0,X_1)$ given $X_t$ directly on $\cS^{d-1}$, which is exactly part~(ii). Part~(ii) with \cref{prop:bounded_vel} then gives part~(iii) by Picard--Lindel\"of, i.e.\ the unique solution of $\dot x = v_t(x)$ that \citet[Def.~1]{hertrich2025rectified} call rectifiability. The bridge instead supplies the intermediate-marginal property (\cref{prop:full_support}), which \citet[Prop.~10]{hertrich2025rectified} show to be the binding obstruction; see \cref{app:fullsupport_why}.

\subsection{Proof of \texorpdfstring{\cref{thm:rectifiability}}{Theorem 1}: Three Consequences of the Bridge Perturbation}
\label{app:proofs_bridge}

\begin{lemma}[Exponential map Jacobian]
\label{lem:exp_jac}
For $v \in T_x\cS^{d-1}$ with $0 < \|v\| = r < \pi$ (i.e.\ $\Exp_x(v)$ outside the cut locus of $x$): $\det(d\Exp_x)_v = (\sin r / r)^{d-2}$.
\end{lemma}
\begin{proof}
In geodesic polar coordinates the Riemannian volume element is $\sin^{d-2}(r)\,dr\,d\omega$ and the Euclidean tangent-space element is $r^{d-2}\,dr\,d\omega$~\citep[Ch.~4]{docarmo1992riemannian}; the ratio is $(\sin r/r)^{d-2}$.
\end{proof}

\begin{lemma}[Full support of conditional distribution]
\label{lem:full_support_cond}
For any $\sigma > 0$ and $\mu \in \cS^{d-1}$, the distribution of $X = \Exp_\mu(\sigma\epsilon)$ with $\epsilon \sim \cN(0, \Pi_\mu)$ (the rank-$(d-1)$ tangent-space Gaussian, embedded in $\mathbb{R}^d$) has positive density on $\cS^{d-1} \setminus \{-\mu\}$.
\end{lemma}
\begin{proof}
The Gaussian $\sigma\epsilon$ has full support on the column space $T_\mu\cS^{d-1} \cong \mathbb{R}^{d-1}$ of $\Pi_\mu$. The exponential map is a diffeomorphism on the open ball of radius $\pi$ in $T_\mu\cS^{d-1}$ onto $\cS^{d-1} \setminus \{-\mu\}$. By \cref{lem:exp_jac} and the change-of-variables formula, the principal-branch pushforward density is:
$$p_X(y) = \frac{1}{(2\pi\sigma^2)^{(d-1)/2}}\exp\!\left(-\frac{\dg(\mu, y)^2}{2\sigma^2}\right)\left(\frac{\dg(\mu, y)}{\sin(\dg(\mu, y))}\right)^{d-2}, \quad y \in \cS^{d-1}\setminus\{-\mu\}.$$
Smoothness follows because $\dg(\mu,\cdot)^2$ is $C^\infty$ on $\cS^{d-1}\setminus\{-\mu\}$ and $r \mapsto (r/\sin r)^{d-2}$ is real-analytic on $[0,\pi)$. Positivity follows from positivity of the Gaussian factor. Wrap-around contributions through the cut locus are bounded by $P(\sigma\|\epsilon\| \geq \pi) \leq \exp(-\tfrac{1}{2}(\pi/\sigma - \sqrt{d-1})^2)$; in the deployed regime $\smax = 0.01 \ll \pi/\sqrt{d-1}\approx 0.098$ for $d=1024$, this is below $e^{-3\times 10^4}$ and does not affect positivity or smoothness.
\end{proof}

\begin{proof}[Full proof of \cref{prop:full_support}]
Fix $t \in (0,1)$. Then $\sigma(t) = \smax\sin(\pi t) > 0$. By \cref{lem:full_support_cond}, $\operatorname{law}(X_t \mid x_0, x_1)$ has positive smooth density on $\cS^{d-1} \setminus \{-\mu_t(x_0,x_1)\}$ for each $(x_0, x_1)$ off the antipodal locus $\bar\Delta$.

\emph{Topological support.} Let $\pi$ be any coupling with $\pi(\bar\Delta) = 0$. For every nonempty open $O \subseteq \cS^{d-1}$ and every $(x_0,x_1)\notin\bar\Delta$, the conditional density is strictly positive on $O$ minus at most one point, so $P(X_t \in O \mid x_0,x_1) > 0$; integrating against $\pi$ gives $P(X_t \in O) > 0$. Hence $\supp(\operatorname{law}(X_t)) = \cS^{d-1}$. No absolute continuity of $\pi$ is used, which matters because the rectification iterates of \cref{cor:disconnected} are deterministic couplings.

\emph{Density.} If in addition $\pi$ is absolutely continuous, the map $(x_0, x_1) \mapsto -\mu_t(x_0, x_1)$ is a smooth submersion on $(\cS^{d-1} \times \cS^{d-1}) \setminus \bar{\Delta}$, so for $\mathrm{vol}_g$-a.e.\ $y \in \cS^{d-1}$ the preimage $\{(x_0,x_1) : y = -\mu_t(x_0,x_1)\}$ is a codimension-$(d-1)$ submanifold and hence $\pi$-null. By Fubini, $\operatorname{law}(X_t) = \int \operatorname{law}(X_t \mid x_0, x_1)\,d\pi$ then has positive density $\mathrm{vol}_g$-a.e.
\end{proof}

\begin{proposition}[Bounded velocity and finite action]
\label{prop:bounded_vel}
Under $\sigma(t) = \smax\sin(\pi t)$: $\|u_t(x_t \mid x_0, x_1)\|_g \leq \smax\pi\|\epsilon\| + \Omega$ for all $t \in (0,1)$, where $\Omega = \dg(x_0, x_1)$.
\end{proposition}
\begin{proof}
For $t \in [0,1)$, $\sin(\pi t) = \sin(\pi(1-t)) \leq \pi(1-t)$ (using $\sin x \leq x$ for $x \geq 0$), so $\sigma(t)/(1-t) \leq \smax\pi$. The geodesic distance satisfies $\dg(x_t, \mu_t) \leq \sigma(t)\|\epsilon\|$ unconditionally: when $\sigma(t)\|\epsilon\| < \pi$ (inside the injectivity ball), this holds with equality by the radial isometry $\Exp_{\mu_t}$; otherwise $\dg(x_t,\mu_t) \leq \pi \leq \sigma(t)\|\epsilon\|$ trivially since $\cS^{d-1}$ has diameter $\pi$. The triangle inequality therefore yields, unconditionally:
$$\|u_t\|_g = \frac{\dg(x_t, x_1)}{1-t} \leq \frac{\dg(x_t,\mu_t) + \dg(\mu_t,x_1)}{1-t} \leq \frac{\sigma(t)\|\epsilon\| + (1-t)\Omega}{1-t} \leq \smax\pi\|\epsilon\| + \Omega.$$
Using $\mathbb{E}[\|\epsilon\|^2] = d-1$, $\mathbb{E}\|\epsilon\|\le\sqrt{d-1}$ (Jensen), $\Omega \leq \pi$, and independence of $\Omega$ and $\|\epsilon\|$, the action integral satisfies $\int_0^1 \mathbb{E}[\|v_t\|_g^2]\,dt \leq (\smax\pi\sqrt{d-1} + \pi)^2 < \infty$.
\end{proof}

\begin{proposition}[Approximate transport: an asymptotic rate in $\smax$]
\label{prop:approx_transport}
Let $\Omega_\textup{max}$ be an upper bound on the endpoint angle $\Omega$ of \cref{eq:slerp}. Part~(i) holds unconditionally. Parts~(ii) and~(iii) assume in addition hypothesis \textbf{(CT)}, that the deterministic marginal velocity $v_t^\textup{det}$ is $C^2$ on $(\cS^{d-1})^N \times [\epsilon_t, 1-\epsilon_t]$ jointly in all $N$ token coordinates, and hold in the limit $\smax \to 0$ at fixed $(N, d, \Omega_\textup{max}, \epsilon_t)$. The bridge-augmented flow satisfies:
\begin{enumerate}[nosep]
\item[(i)] boundary consistency, $\sigma(0) = \sigma(1) = 0$;
\item[(ii)] velocity perturbation $\|\tilde{v}_t - v_t^\textup{det}\|_\infty = O(\smax^2\sin^2(\pi t))$, uniformly on $[\epsilon_t, 1-\epsilon_t]$, where $\tilde{v}_t$ is the bridge-smoothed marginal velocity;
\item[(iii)] $W_2\bigl((\Phi_1^{\tilde{v}})_\# p_0,\, p_1\bigr) = O(\smax^2)$, writing $\Phi_1^{v}$ for the time-1 flow map of a field $v$.
\end{enumerate}
Parts (ii) and (iii) are rates. The implied constants are uniform for $\smax$ in a bounded neighborhood of $0$, depend on $(N, d, \Omega_\textup{max}, \epsilon_t)$, and are not quantified here; in particular \cref{prop:approx_transport} makes no finite-parameter guarantee at the deployed $\smax = 10^{-2}$ and $\epsilon_t = 10^{-5}$, and asserts no specific numerical transport error there.
\end{proposition}
\begin{proof}
(i)~Immediate from $\sigma(0) = \sigma(1) = 0$.

(ii)~The bridge-augmented velocity $\tilde{v}_t$ is a kernel-smoothed conditional expectation with bandwidth $\sigma(t)$. By on-manifold kernel-regression bias theory for closed Riemannian manifolds~\citep{pelletier2006regression}, the bias is $O(\sigma(t)^2)$ on $[\epsilon_t, 1-\epsilon_t]$ when $v_t^\text{det}$ is $C^2$ (the Riemannian volume Jacobian of $\Exp_{\mu_t}$ is $1 - \tfrac{d-2}{6}r^2 + O(r^4)$ for $r = \|\sigma(t)\epsilon\|$; the $O(r^2)$ term carries a dimensional $(d-2)/6$ factor that we absorb into the $O(\sigma(t)^2)$ bound, leaving leading-order bias unchanged). On a single $\cS^{d-1}$, $C^2$ regularity of $v_t^\text{det}$ follows from $C^\infty$ regularity of the OT map (MTW theory~\citep{loeper2009regularity,ma2005regularity}) propagated through $v_t^\text{det} = \Log_x(x_1)/(1-t)$, bounded on $[\epsilon_t, 1-\epsilon_t]$. On the product manifold $(\cS^{d-1})^N$ this per-token argument does not suffice, because self-attention couples the tokens and MTW theory is a per-factor statement; the joint $C^2$ regularity is hypothesis~(CT), and it is assumed, not proved. (CT) enters here and nowhere else: it bounds the transport rate of \cref{prop:approx_transport} and is used by neither \cref{prop:full_support} nor any part of \cref{thm:rectifiability}. Note also that (CT) constrains $v_t^\text{det}$, the deterministic velocity, and not the bridge-smoothed marginal field $v_t$ of \cref{thm:rectifiability}(iii) (written $\tilde v_t$ here), whose regularity the smoothing itself produces.

(iii)~On $[\epsilon_t, 1-\epsilon_t]$, the marginal velocity $v_t$ is uniformly Lipschitz with constant $L = L(N, d, \Omega_\text{max}, \smax, \epsilon_t)$. By Gr\"onwall's inequality, $\sup_x \dg(\Phi_1^{\tilde{v}}(x), \Phi_1^{v^\text{det}}(x)) \leq C\smax^2$ with $C$ depending on the same parameters; coupling $X \sim p_0$ through both flows gives a valid coupling of $(\Phi_1^{\tilde{v}})_\# p_0$ and $(\Phi_1^{v^\text{det}})_\# p_0 = p_1$, so $W_2^2((\Phi_1^{\tilde{v}})_\# p_0, p_1) \leq \mathbb{E}_{x \sim p_0}[\dg(\Phi_1^{\tilde{v}}(x), \Phi_1^{v^\text{det}}(x))^2] \leq C^2 \smax^4$, giving $W_2 \leq C\smax^2$. We note that $C$ depends on $\epsilon_t$ via the velocity blow-up at $t \to 1$, and that the proof of~(ii) absorbs the dimensional $(d-2)/6$ volume-Jacobian factor into the $O(\sigma(t)^2)$ bound. The $O(\smax^2)$ rate is what we claim; the absolute constant at the deployed cutoff $\epsilon_t = 10^{-5}$ is large, so \cref{prop:approx_transport} is not a numerical guarantee at the deployed $\smax = 0.01$. It does not, for instance, assert a $10^{-4}$ transport error. The qualitative smallness of the transport perturbation is corroborated empirically instead (\cref{tab:ablation}, isotropic-decoder rows; \cref{app:noise_schedule}).
\end{proof}

\begin{corollary}[Bridge resolution of the disconnected-support obstruction]
\label{cor:disconnected}
Let $\pi$ be any coupling on $\cS^{d-1}\times\cS^{d-1}$ giving no mass to the antipodal locus, $\pi(\bar\Delta) = 0$. This is a strictly weaker hypothesis than absolute continuity, and one that the deterministic couplings arising as rectification iterates satisfy. Then the bridge perturbation with $\smax > 0$ ensures $\supp(\operatorname{law}(X_t)) = \cS^{d-1}$ for all $t \in (0,1)$, whatever the structure of $p_1$. Consequently, no rectification iterate can become a non-optimal fixed point of the kind \citet[Prop.~10]{hertrich2025rectified} identify, exhibiting the disconnected-support obstruction.
\end{corollary}
\begin{proof}
By \cref{lem:full_support_cond} the conditional density is strictly positive on $\cS^{d-1}\setminus\{-\mu_t(x_0,x_1)\}$ for every pair $(x_0,x_1)\notin\bar\Delta$, so every nonempty open set has positive probability under any $\pi$ with $\pi(\bar\Delta)=0$; no absolute continuity of $\pi$ is used. Hence $\operatorname{law}(X_t)$ has full topological support for every $t\in(0,1)$ and every such $\pi$, including the deterministic rectification iterates. Hence no fixed point of the bridge--rectification map exhibits the disconnected-support pathology underlying the counterexample of \citet{hertrich2025rectified}.
\end{proof}

\begin{proof}[Full proof of \cref{thm:rectifiability}]
\leavevmode
\begin{enumerate}[nosep]
\item[(i)] \cref{prop:full_support}.

\item[(ii)] By Bayes' rule (with respect to the Riemannian volume measure on $\cS^{d-1}$ as the dominating reference): $p(x_0, x_1 \mid X_t = x) \propto p(X_t = x \mid x_0, x_1)\,d\pi$, where the conditional density $p(X_t = x \mid x_0, x_1)$ is given by \cref{lem:full_support_cond} on the cut-locus complement (a full-measure subset of $\cS^{d-1}$ for $d \geq 3$). The numerator is smooth on this domain (Gaussian pushforward through the exponential map). The denominator is positive (\cref{lem:full_support_cond}) and smooth on the cut-locus complement, using the per-direction Gaussian decay of \cref{lem:full_support_cond} as the integrable dominating function for differentiation under the integral. On $(\cS^{d-1})^N$ the conditional density of $X_t$ given $(x_0,x_1)$ is the product of the $N$ per-token densities of \cref{lem:full_support_cond}, so it is positive and smooth exactly when each factor is, and the Bayes argument above applies verbatim with $\Pi_{\mu_t}$ replaced by $\bigoplus_n \Pi_{\mu_t^{(n)}}$. What is independent is the bridge noise, which is exact by construction.

\item[(iii)] On $[\epsilon_t, 1-\epsilon_t]$: $v_t(x) = \mathbb{E}[u_t \mid X_t = x]$ is smooth by differentiation under the integral (using the Gaussian dominating function from (ii) and the bounded $u_t$ from \cref{prop:bounded_vel}). Differentiating once more, $\|\nabla_x v_t\|_\infty \leq L$ with $L = L(d, \Omega_\text{max}, \smax, \epsilon_t)$ finite on the closed interval (the kernel-derivative bound scales as $\sigma(t)^{-1}$, which is bounded below by $\smax\sin(\pi\epsilon_t) > 0$). The Picard--Lindel\"of theorem for time-dependent Lipschitz vector fields on a compact Riemannian manifold~\citep{lee2018riemannian} then gives existence and uniqueness of the initial value problem $\dot x = v_t(x), x(t_0) = x_0$ on $[\epsilon_t, 1-\epsilon_t]$ for any $t_0$ in the same interval. Part~(iii) also lifts to $(\cS^{d-1})^N$: the argument uses only the dominating function from~(ii), the velocity bound of \cref{prop:bounded_vel} and $\sigma(t)\ge\smax\sin(\pi\epsilon_t)>0$, all of which hold on the product, with $N$ entering the Lipschitz constant $L$ and not the hypotheses. In particular the cross-token hypothesis~(CT) of \cref{prop:approx_transport} is not used anywhere in \cref{thm:rectifiability}. \qedhere
\end{enumerate}
\end{proof}

\subsection{Exact Conditional Velocity of the Stochastic Bridge}
\label{app:exact_velocity}

In FM the conditional velocity is the time derivative of the conditional path, and that is the sense in which \cref{sec:bridge} and \cref{app:x_prediction} use the term. For the bridge it is $\dot X_t$ under the conditioning $z = (x_0, x_1, \epsilon)$; the marginal field that transports the bridge's own marginals is $\mathbb{E}[\dot X_t \mid X_t = x]$, which admits no closed form. We give $\dot X_t$ in closed form.

\Cref{def:bridge} fixes the marginal law of $X_t$ at each $t$ but not a joint process across $t$, so $\dot X_t$ requires choosing a coupling of the noise across time. We take parallel transport of $\epsilon$ along the SLERP curve $\mu$. Parallel transport is a linear isometry, so every marginal in \cref{def:bridge} is unchanged and $t \mapsto X_t$ becomes differentiable.

\begin{proposition}[Exact conditional velocity of the bridge]
\label{prop:exact_velocity}
Write $r_t = \sigma(t)\|\epsilon\|$ and $w_t = P_{\mu_0 \to \mu_t}\epsilon/\|\epsilon\|$, so that $X_t = \cos(r_t)\,\mu_t + \sin(r_t)\,w_t$, and set $a = \langle \dot\mu_t, w_t\rangle$. Then
\begin{equation}
  \dot X_t \;=\; \bigl[\sigma'(t)\|\epsilon\| + a\bigr]\,\nu_t
  \;+\; \cos\!\bigl(\sigma(t)\|\epsilon\|\bigr)\,\dot\mu_t^{\perp},
  \qquad
  \begin{aligned}
    \nu_t &= \cos(r_t)\,w_t - \sin(r_t)\,\mu_t, \\
    \dot\mu_t^{\perp} &= \dot\mu_t - a\,w_t .
  \end{aligned}
  \label{eq:exact_bridge_velocity}
\end{equation}
The right-hand side is tangent to $\cS^{d-1}$ at $X_t$ and depends only on $(t, \mu_t, \dot\mu_t, X_t, \epsilon)$. Here $\nu_t = -\Log_{X_t}(\mu_t)/r_t$ is the outward radial unit vector at $X_t$ of the geodesic emanating from $\mu_t$. Setting $\sigma \equiv 0$ returns $\dot X_t = \dot\mu_t$, the deterministic SLERP velocity.
\end{proposition}

\begin{proof}
Parallel transport along $\mu$ keeps the ambient derivative of $w_t$ normal to the sphere, so $\dot w_t = \lambda\mu_t$ for some scalar $\lambda$; differentiating $\langle w_t, \mu_t\rangle = 0$ gives $\lambda = \langle \dot w_t, \mu_t\rangle = -\langle w_t, \dot\mu_t\rangle = -a$, i.e.\ $\dot w_t = -a\,\mu_t$. Differentiating $X_t = \cos(r_t)\mu_t + \sin(r_t)w_t$ and using $\dot r_t = \sigma'(t)\|\epsilon\|$,
\[
  \dot X_t = \dot r_t\bigl(\cos(r_t)w_t - \sin(r_t)\mu_t\bigr) + \cos(r_t)\dot\mu_t - a\sin(r_t)\mu_t .
\]
Decompose $\dot\mu_t = \dot\mu_t^{\perp} + a\,w_t$. The term $a\cos(r_t)w_t$ combines with $-a\sin(r_t)\mu_t$ into $a\,\nu_t$, giving \cref{eq:exact_bridge_velocity}. Tangency: $\langle X_t, \nu_t\rangle = -\cos r_t\sin r_t + \sin r_t\cos r_t = 0$ and $\langle X_t, \dot\mu_t^{\perp}\rangle = \sin(r_t)(\langle w_t,\dot\mu_t\rangle - a) = 0$.
\end{proof}

\paragraph{What the perturbation changes.} Against the deterministic $\dot\mu_t$, the bridge alters the conditional velocity in exactly two ways. It adds a radial component along $\nu_t$ at rate $\sigma'(t)\|\epsilon\|$, and it attenuates the transporting drift $\dot\mu_t^{\perp}$ by $\cos r_t$, which is the basepoint derivative of $\Exp$ and hence a curvature effect. At the deployed values ($d = 1024$, $\smax = 10^{-2}$, $\|\epsilon\|\approx\sqrt{d-1}$) the attenuation $1-\cos r_t$ is $5.1\%$ at $t = 0.5$. The $a$ term rides on $\nu_t$ alongside the radial rate and is small in root-mean-square, $\Omega/\sqrt{d-1} \approx 0.05$ with $\Omega = \dg(x_0,x_1)$.

\paragraph{Consistency with the inference velocity.} The inference velocity is not intended to reproduce $\dot X_t$, and does not. The bridge acts only at training time: it is a $\sigma(t)$-bandwidth smoothing of the endpoint regression, not a path to be integrated. \Cref{alg:training} fits a normalized endpoint $\hat x_1$ under the chord loss \cref{eq:vloss_chord}, so the learned field is $\Log_x(\hat x_1)/(1-t)$ with $\hat x_1$ the normalized $\mathbb{E}[x_1 \mid X_t = x]$. Only the transport of $p_0$ to $p_1$ has to agree, and it agrees approximately rather than exactly. \Cref{prop:approx_transport} quantifies the gap for $\tilde v_t = \mathbb{E}[u_t \mid X_t]$ rather than for the chord field, with part~(ii) bounding $\tilde v_t - v_t^\textup{det}$ and part~(iii) converting it into $W_2 = O(\smax^2)$, subject to the scoping stated there.

\paragraph{Why we regress the endpoint target rather than $\dot X_t$.} CFM on the bridge as a conditional path would regress $\dot X_t$; we regress the SLERP-type endpoint target instead, and \cref{prop:approx_transport} bounds what that deliberate misspecification costs. The trade is safe because the radial channel nets to zero over the path, $\int_0^1 \sigma'(t)\,dt = \sigma(1)-\sigma(0) = 0$, so it transports nothing even though it is far from pointwise small. At the deployed values it peaks at $\smax\pi\sqrt{d-1} \approx 1.00$ at the endpoints, against a geodesic speed $\Omega \approx \pi/2$. Regressing it would add real variance for zero net transport, and the smoothing bias left behind is \cref{prop:approx_transport}(ii)'s $O(\smax^2)$.

\subsection{Support of the Intermediate Marginals Under Deterministic SLERP}
\label{app:fullsupport_why}

The source marginal $p_0 = \operatorname{law}(X_0)$ and the intermediate marginals $\operatorname{law}(X_t)$ for $t \in (0,1)$ are constrained by different parts of the construction. $p_0$ constrains the source alone, and under deterministic SLERP the intermediate marginals are fixed by $p_1$ and $t$. Rectifiability in the sense of \citet[Def.~1]{hertrich2025rectified} asks for a unique solution of $\dot z_t = v_t(z_t)$, a property of the velocity field, and $v_t$ follows from the conditional law of the endpoints given $X_t$ rather than from $p_0$. A uniform source therefore does not supply it.

\paragraph{The unreachable cap.} Since $\dg(\SLERP(x_0,x_1,t), x_1) = (1-t)\Omega$ with $\Omega = \dg(x_0,x_1) \le \pi$, for fixed $x_1$ the support of $\operatorname{law}(X_t \mid x_1)$ under deterministic SLERP is the closed ball $\bar B(x_1, (1-t)\pi)$, leaving the open cap $B(-x_1, t\pi)$ unreachable. With $p_1 = \delta_{x_1}$, $p_0 = \mathrm{Uniform}(\cS^{d-1})$ and $t = \tfrac12$, the support is the closed hemisphere centered at $x_1$ and the open hemisphere around $-x_1$ carries zero mass, half the sphere under a uniform source. In general
\[
  \supp\bigl(\operatorname{law}(X_t)\bigr) \;=\; \bigl\{ y : \dg(y, \supp p_1) \le (1-t)\pi \bigr\},
\]
so if $p_1$ lies in a cap of angular radius $\alpha$ about $c$, then a cap of radius $t\pi - \alpha$ about $-c$ is missed for every $t > \alpha/\pi$. Any $\smax > 0$ fills it, since \cref{lem:full_support_cond}'s tangent Gaussian is positive on $\cS^{d-1}\setminus\{-\mu_t\}$ and \cref{prop:full_support} then gives full support at every $t \in (0,1)$.

\paragraph{Consequences at $\smax = 0$.} Off the support, the density of $\operatorname{law}(X_t)$ vanishes, so Bayes' rule does not define the conditional law of $(X_0,X_1)$ given $X_t = x$ there, and the velocity it produces covers only a proper closed subset of $\cS^{d-1}$. On the support, deterministic SLERP still misses the density property of \cref{thm:rectifiability}(ii), because conditioning on $X_t = x$ leaves a direction ($d-2$ parameters) and a length ($1$), so the fiber is $(d-1)$-dimensional inside a $2(d-1)$-dimensional product and is singular with respect to the product Riemannian volume measure on $\cS^{d-1}\times\cS^{d-1}$, admitting no density against it. The vanishing density off the support and the singular fiber on it are properties of the Bayes construction at $\smax = 0$. The deterministic velocity $v_t^\textup{det}$ may still be smooth and admit a unique solution of $\dot z_t = v_t^\textup{det}(z_t)$ for a given $p_1$.

\section{Additional Technical Details}
\label{app:technical}

\subsection{Noise Schedule: Derivation and Ablation}
\label{app:noise_schedule}

\paragraph{KL derivation.} The standard Brownian bridge $B$ on $[0,1]$ has covariance $\mathrm{Cov}(B_s, B_t) = \min(s,t) - st$. The associated eigenproblem
\[
  \int_0^1 \bigl(\min(s,t) - st\bigr)\,\phi_n(s)\,ds \;=\; \lambda_n\,\phi_n(t)
\]
is solved by differentiating twice, giving $\phi_n'' = -\lambda_n^{-1}\phi_n$ with $\phi_n(0)=\phi_n(1)=0$, hence
\[
  \phi_n(t) = \sqrt{2}\,\sin(n\pi t), \qquad \lambda_n = \frac{1}{(n\pi)^2}, \qquad n \ge 1,
\]
and the KL expansion $B_t = \sum_{n\ge1} Z_n \sqrt{2}\,\sin(n\pi t)/(n\pi)$ with $Z_n \stackrel{\text{i.i.d.}}{\sim} \cN(0,1)$, whose variance $\sum_{n\ge1} 2\sin^2(n\pi t)/(n\pi)^2$ sums to $t(1-t)$ and so recovers the bridge covariance. Mode $n$ contributes a marginal SD proportional to $|\sin(n\pi t)|/(n\pi)$.

We retain the $n{=}1$ mode alone. It is the unique KL mode strictly positive on the open interval $(0,1)$: every $n \ge 2$ mode has $n-1$ interior zeros, at which the perturbation would vanish and the bridge would degenerate to the deterministic path at interior times, defeating the purpose of the construction. Rescaling the retained mode so that its peak is $\smax$ gives $\sigma(t) = \smax\sin(\pi t)$, with total noise budget $\int_0^1\sigma(t)^2\,dt = \smax^2/2$.

\paragraph{Concentration ceiling on $\smax$.} The perturbation must stay inside the injectivity radius of $\cS^{d-1}$. With $\epsilon \sim \cN(0,\Pi_{\mu_t})$ one has $\|\epsilon\| \approx \sqrt{d-1}$, so the typical geodesic displacement is $\sigma(t)\|\epsilon\| \le \smax\sqrt{d-1}$, and requiring this to stay below $\pi$ gives $\smax \lesssim \pi/\sqrt{d-1} \approx 0.098$ at $d = 1024$. This ceiling comes from the analysis rather than from search, so the admissible window is fixed by $d$; the largest value we test, $2\times10^{-2}$, sits well below the ceiling of $0.098$, and where the optimum sits inside the window remains empirical.

\paragraph{Why a bridge-type perturbation rather than variance-exploding (VE) or variance-preserving (VP) noise.}
\begin{enumerate}[leftmargin=*,topsep=2pt,itemsep=2pt,label=(\roman*)]
    \item \emph{Exact boundary conditions}: $\sigma(0) = \sigma(1) = 0$ ensures $X_0 = x_0$ and $X_1 = x_1$ exactly; VE-type processes do not naturally satisfy $\sigma(1) = 0$.
    \item \emph{Intrinsic to $\cS^{d-1}$}: VP-type linear interpolation $(1-\beta)x_0 + \beta x_1 + \sigma\epsilon$ leaves the sphere; our construction uses the exponential map, keeping all intermediate points on $\cS^{d-1}$.
    \item \emph{Controlled noise budget}: $\int_0^1 \sigma(t)^2\,dt = \smax^2/2$ directly controls the transport rate via \cref{prop:approx_transport}.
\end{enumerate}

\paragraph{Empirical comparison of schedules.} We put the deployed kl1 schedule against three alternatives on BRCA. All four share the same SLERP backbone, the tangent Gaussian and the exponential map; only the scalar amplitude $\sigma(t)$ changes. kl1 and kl2 truncate the KL expansion above at one and two modes; under the common normalization both peak at $\smax$, so the pair tests first-mode sufficiency. si is the stochastic-interpolant noise profile $\gamma(t)=\sqrt{2t(1-t)}$ \citep{albergo2023stochastic}, rescaled so its peak is also $\smax$, making the comparison one of shape rather than amplitude. ve is an unpinned monotone control: $\smax\sqrt{t}$ is the marginal SD of the constant-diffusion member of the VE family \citep{song2021scorebased}, and is a representative unpinned profile rather than the schedule VE implementations usually deploy.

\begin{table}[!h]
  \centering
  \caption{Stage-1 noise-schedule ablation on BRCA (UNI encoder), full 90k decoder schedule, single checkpoint. $\int_0^1\sigma^2$ is in units of $\smax^2$.}
  \label{tab:schedule_ablation}
  \footnotesize
  \setlength{\tabcolsep}{5pt}
  \begin{tabular}{@{}llccc c@{}}
    \toprule
    Schedule & $\sigma(t)$ & $\sigma(1)$ & $\int_0^1\sigma^2$ & rFID $\downarrow$ & gFID $\downarrow$ \\
    \midrule
    kl1 (KL mode $n{=}1$; deployed) & $\smax\sin(\pi t)$ & $0$ & $\tfrac12$ & $\mathbf{2.23}$ & $6.16$ \\
    kl2 (KL modes $n{=}1,2$) & $\smax\sqrt{\sin^2\pi t + \tfrac14\sin^2 2\pi t}$ & $0$ & $\tfrac58$ & $3.33$ & $\mathbf{5.64}$ \\
    si (stochastic interpolant) & $\smax\cdot 2\sqrt{t(1-t)}$ & $0$ & $\tfrac23$ & $3.28$ & $7.15$ \\
    ve (VE, unpinned) & $\smax\sqrt{t}$ & $\smax$ & $\tfrac12$ & $4.65$ & $10.58$ \\
    \bottomrule
  \end{tabular}
\end{table}

\paragraph{Noise-budget accounting.} At a matched peak, $\int_0^1\sigma^2$ is the total noise each schedule spends, so the table separates shape from volume. Under a uniform reading of $t$ the unpinned control spends exactly kl1's budget, so what it loses cannot be charged to a larger budget; training draws $t$ from a logit-normal law centered below $1/2$, under which ve in fact spends about $38\%$ less than kl1, and it still loses $3.29$ gFID against removing the bridge altogether (\cref{tab:sigmamax_ablation}, top row). Spending less noise and placing it at the data end is worse than spending none, which is what pinning at $\sigma(1)=0$ buys. ve is itself pinned at the source, $\sigma(0)=0$, so the comparison bears on the data end alone. si shares kl1's pinning and peak but spends $33\%$ more on the uniform reading and loses on both metrics. kl2 does attain a better gFID, on a $25\%$ larger budget and a $49\%$ worse rFID ($3.33$ against $2.23$); we retain kl1 because it is the schedule the theory of \cref{sec:bridge} is stated for and the configuration every other result in this paper is measured on.

\paragraph{Bridge noise scale $\smax$.} \Cref{tab:sigmamax_ablation} sweeps the bridge amplitude on BRCA. At $\smax = 0$ the construction reverts to deterministic SLERP and records the worst rFID of the four settings, $4.56$. Raising the amplitude to $5\times10^{-3}$ improves both metrics, and the deployed $10^{-2}$ gives the best rFID, $2.23$, and the best gFID, $6.16$. Doubling to $2\times10^{-2}$ reverses both, costing $1.46$ rFID and $1.72$ gFID against the deployed value and recording the worst gFID of the four, $7.88$. The deployed value therefore sits in the interior of the sweep on both metrics.

\begin{table}[!h]
  \centering
  \caption{Stage-1 bridge noise scale $\smax$ on BRCA, full 90k decoder schedule, single checkpoint.}
  \label{tab:sigmamax_ablation}
  \footnotesize
  \setlength{\tabcolsep}{6pt}
  \begin{tabular}{@{}lcc@{}}
    \toprule
    $\smax$ & rFID $\downarrow$ & gFID $\downarrow$ \\
    \midrule
    $0$ (no bridge)      & $4.56$ & $7.29$ \\
    $5\times10^{-3}$     & $3.22$ & $6.93$ \\
    $\mathbf{10^{-2}}$ \emph{(deployed)} & $\mathbf{2.23}$ & $\mathbf{6.16}$ \\
    $2\times10^{-2}$     & $3.69$ & $7.88$ \\
    \bottomrule
  \end{tabular}
\end{table}

\subsection{x-Prediction Formulation and Chord-vs-Geodesic Loss}
\label{app:x_prediction}

The conditional velocity contains the factor $\sigma'(t)/\sigma(t) = \pi\cot(\pi t)$, which diverges at $t=0$ and $t=1$. The velocity itself stays finite, because $\pi\cot(\pi t)$ multiplies the radial displacement $\dg(\mu_t, X_t) = \sigma(t)\|\epsilon\|$, which vanishes at the same rate (\cref{prop:exact_velocity}). The parameterization is ill-conditioned at the endpoints, not the field. The x-prediction formulation~\citep{li2025backtobasics} avoids this: the network $f_\theta(x_t, t)$ predicts the data endpoint with normalized output $\hat{x}_1 = f_\theta(x_t, t)/\|f_\theta(x_t, t)\|$, and at inference the velocity is recovered as $v_\theta(x_t, t) = \Log_{x_t}(\hat{x}_1)/(1-t)$. The $\sin(\pi t)/(1-t) \to \pi$ cancellation (\cref{prop:bounded_vel}) keeps the velocity target $u_t = \Log_{x_t}(x_1)/(1-t)$ bounded as $t \to 1$. Time is sampled per sample from the logit-normal distribution of \citet{esser2024scaling}, $t = \operatorname{sigmoid}(u)$ with $u \sim \cN(-0.5, 1)$, clamped to $[\epsilon_t, 1-\epsilon_t]$ with $\epsilon_t = 10^{-5}$; no resolution-dependent time shift is applied, neither in training nor on the uniform sampling grid.

\paragraph{Chord vs.\ geodesic loss (full derivation).}
The natural Riemannian (geodesic) training loss measures velocity error in the tangent space via the logarithmic map:
\begin{equation}
  \cL_\text{geo}(\theta) = \mathbb{E}_{t,\,(x_0,x_1) \sim \pi,\,\epsilon}\!\left[\frac{\|\Log_{x_t}(\hat{x}_1) - \Log_{x_t}(x_1)\|_g^2}{(1-t)^2}\right].
  \label{eq:vloss_geo}
\end{equation}
However, the Log map involves $\arccos$ (cf.\ \cref{eq:log_map}), whose Jacobian diverges as $\theta = \arccos(\langle x_t, \hat{x}_1\rangle) \to 0$, i.e.\ when the prediction collapses toward the current point $x_t$, a regime that becomes unavoidable near $t = 1$ where $x_t \approx x_1$. This creates a poorly conditioned gradient landscape. The chord loss \cref{eq:vloss_chord} replaces the tangent-space distance with the ambient $\ell_2$ distance between normalized predictions and shares the same per-sample minimizer ($\hat{x}_1 = x_1$) and the same inference velocity $v_\theta(x_t, t) = \Log_{x_t}(\hat{x}_1)/(1-t)$; the training gradient landscape changes, avoiding the $\arccos$ singularity.

\subsection{Stage 1 and Stage 2 Training Algorithms}
\label{app:training_alg}

\Cref{alg:training} summarizes Stage 1 training of the LightningDiT-XL generator $f_\theta$ on bridge-perturbed VFM embeddings (\cref{def:bridge}) with the chord loss \cref{eq:vloss_chord}. For each minibatch, source latents $\bZ_0$ are sampled from the uniform distribution on $(\cS^{d-1})^N$, paired with precomputed data latents $\bZ_1$ from the encoder; each sample $j$ draws its own time $t_j$; per-token bridge interpolants $z_{t,j}^{(n)} = \Exp_{\mu_{t,j}^{(n)}}(\smax\sin(\pi t_j)\,\epsilon_j^{(n)})$ are formed with $\epsilon_j^{(n)} \sim \cN(0, \Pi_{\mu_{t,j}^{(n)}})$; and the network's normalized prediction is regressed against the data endpoint with the $1/(1-t)^2$ weighting absorbed into the chord loss. The deployed loss floors $1-t_j$ at $10^{-3}$ and caps each per-token term at $10^{4}$.

\begin{algorithm}[t]
\caption{STREAM Stage 1 DiT Training}
\label{alg:training}
\KwIn{Precomputed data features $\{\bZ_{1,j}\}_{j=1}^B \sim p_\text{data}$ on $(\cS^{d-1})^N$, noise level $\smax$}
\For{each minibatch}{
  Sample $\bZ_{0,j} \sim \text{Uniform}((\cS^{d-1})^N)$ for $j = 1, \ldots, B$ \tcp*{independent coupling}
  Sample $t_j = \operatorname{sigmoid}(u_j)$, $u_j \sim \cN(-0.5, 1)$, clamp to $[\epsilon_t, 1-\epsilon_t]$, for $j = 1, \ldots, B$\;
  \For{each sample $j$ and token $n = 1, \ldots, N$}{
    $\mu_{t,j}^{(n)} \gets \SLERP(z_{0,j}^{(n)}, z_{1,j}^{(n)}, t_j)$\;
    Sample $\epsilon_j^{(n)} \sim \cN(0, \Pi_{\mu_{t,j}^{(n)}})$\;
    $z_{t,j}^{(n)} \gets \Exp_{\mu_{t,j}^{(n)}}(\smax \sin(\pi t_j) \cdot \epsilon_j^{(n)})$\;
  }
  Predict $\hat{\bZ}_{1,j} = f_\theta(\bZ_{t,j}, t_j)$; normalize each token: $\hat{z}_{1,j}^{(n)} \gets \hat{z}_{1,j}^{(n)} / \|\hat{z}_{1,j}^{(n)}\|$\;
  $\cL \gets \frac{1}{BN} \sum_{j=1}^{B}\sum_{n=1}^N \min\bigl(\|\hat{z}_{1,j}^{(n)} - z_{1,j}^{(n)}\|^2 / \max(1-t_j, 10^{-3})^2,\ 10^{4}\bigr)$ \tcp*{chord loss}
  Update $\theta$ by AdamW on $\cL$ with gradient-norm clipping at $5.0$; update the EMA of $\theta$\;
}
\end{algorithm}

\Cref{alg:decoder_training} summarizes Stage 2 decoder training. The trained DiT enters only through the precomputed cache of centroids $\bar{\bZ}_c$ and per-(centroid, token) bases $U_H$ (\cref{app:implementation}). Each minibatch loads precomputed features $\bZ_j = E(I_j)$, assigns each sample to one centroid, perturbs each token with anisotropic noise from that centroid's basis for the token, projects back onto $\cS^{d-1}$, and decodes; the 4-component loss \cref{eq:decoder_loss} combines pixel-space $\ell_1$ (a per-pixel mean), LPIPS, a token-averaged cosine round-trip term targeting the clean latent (forcing the decoder to denoise), and an adaptive adversarial term that is active only in phase C of \cref{tab:hparams}.

The cosine round-trip loss in Stage 2 targets the clean features $z$, not the noisy input $\tilde{z}$. Per token, the noise is $n = \Pi_z(\sigma_H U_H U_H^\top \eta + \sigma_L (I_d - U_H U_H^\top)\eta)$, $\eta \sim \cN(0, I_d)$, with $U_H$ from the assigned centroid; at $z = \bar{z}_c$ its covariance is exactly \cref{eq:aniso_cov}. As $n \in T_z\cS^{d-1}$, $\tilde{z} = (z + n)/\|z + n\|$ agrees with $\Exp_z(n)$ up to $O(\|n\|^3)$. Adversarial adaptive weight: $\lambda_\text{adapt} = \|\nabla_\text{last}\cL_\text{rec}\| / (\|\nabla_\text{last}\cL_\text{adv}\| + 10^{-6})$, clamped to $[0.05, 10^4]$, where $\cL_\text{rec}$ is the non-adversarial part of \cref{eq:decoder_loss} ($\ell_1$, LPIPS and cosine round-trip) and $\nabla_\text{last}$ the gradient at the decoder's last layer.

\begin{algorithm}[t]
\caption{STREAM Stage 2 Decoder Training}
\label{alg:decoder_training}
\KwIn{Frozen encoder $E$, centroids $\{\bar{\bZ}_c\}_{c=1}^K \subset (\cS^{d-1})^N$ and per-(centroid, token) bases $\{U_H(\bar{z}_c^{(n)})\}$ from the DiT's cache, noise scales $\sigma_H, \sigma_L$, loss weights $(\lambda_{\ell_1}, \lambda_\text{LPIPS}, \lambda_\text{cos}, \omega_G)$, phase boundaries $s_B < s_C$}
\For{each step $s$ and minibatch $\{(I_j, \bZ_j)\}_{j=1}^B$, $\bZ_j = E(I_j)$ precomputed}{
  Apply random horizontal and vertical flips (each w.p.\ $\tfrac12$) and a rotation by $k\cdot 90^\circ$, $k \sim \mathcal{U}\{0,1,2,3\}$, identically to $I_j$ and to the $16{\times}16$ token grid of $\bZ_j$\;
  $c_j \gets \arg\max_c \sum_{n=1}^N \langle z_j^{(n)}, \bar{z}_c^{(n)} \rangle$ \tcp*{nearest centroid, per sample}
  \For{each token $n = 1, \ldots, N$}{
    $U \gets U_H(\bar{z}_{c_j}^{(n)})$; sample $\eta \sim \cN(0, I_d)$\;
    $n_j^{(n)} \gets \Pi_{z_j^{(n)}}\bigl(\sigma_H\, U U^\top \eta + \sigma_L (I_d - U U^\top)\, \eta\bigr)$ \tcp*{tangent noise}
    $\tilde{z}_j^{(n)} \gets (z_j^{(n)} + n_j^{(n)}) / \|z_j^{(n)} + n_j^{(n)}\|$ \tcp*{retract to $\cS^{d-1}$}
  }
  Decode: $\hat{I}_j \gets D_\phi(\tilde{\bZ}_j)$; round-trip: $z_{\text{rt},j} \gets E(\hat{I}_j)$ \tcp*{frozen encoder}
  $\cL_\text{rec} \gets \lambda_{\ell_1} \|\hat{I}_j - I_j\|_1 + \lambda_\text{LPIPS}\,\text{LPIPS}(\hat{I}_j, I_j) + \lambda_\text{cos}\,(1 - \langle z_{\text{rt},j}, z_j \rangle)$\;
  $\cL_\text{dec} \gets \cL_\text{rec} + \omega_G\,\lambda_\text{adapt}\,\cL_\text{adv}$ if $s \ge s_C$, else $\cL_\text{dec} \gets \cL_\text{rec}$\;
  Update $\phi$ by AdamW on $\cL_\text{dec}$; update the EMA of $\phi$\;
  \lIf{$s \ge s_B$}{update the discriminator (hinge loss)}
}
\end{algorithm}

\subsection{Integration Approximation and Generation Algorithm}
\label{app:integration}

\begin{algorithm}[t]
\caption{STREAM Generation}
\label{alg:generation}
\KwIn{Trained DiT $f_\theta$ (EMA weights), number of steps $n_\text{steps} = 25$, decoder $D_\phi$}
Sample $\bZ_0 \sim \text{Uniform}((\cS^{d-1})^N)$; $\bZ \gets \bZ_0$\;
$dt \gets 1 / n_\text{steps}$ \tcp*{uniform time grid}
\For{$i = 0, \ldots, n_\text{steps} - 1$}{
  $t \gets i / n_\text{steps}$\;
  $\hat{\bZ}_1 \gets f_\theta(\bZ, t)$; normalize each token\;
  \For{each token $n$}{
    $v^{(n)} \gets \Log_{z^{(n)}}(\hat{z}_1^{(n)}) / (1-t)$ \tcp*{velocity at $t$}
    $z_{1/2}^{(n)} \gets \Exp_{z^{(n)}}\bigl((dt/2) \cdot v^{(n)}\bigr)$ \tcp*{half step}
  }
  $\hat{\bZ}_1^{(1/2)} \gets f_\theta(\bZ_{1/2}, t + dt/2)$; normalize each token \tcp*{second DiT call}
  \For{each token $n$}{
    $v_\text{mid}^{(n)} \gets \Log_{z_{1/2}^{(n)}}(\hat{z}_1^{(1/2),(n)}) / (1 - t - dt/2)$\;
    $z^{(n)} \gets \Exp_{z^{(n)}}\bigl(\Pi_{z^{(n)}}(dt \cdot v_\text{mid}^{(n)})\bigr)$ \tcp*{full step from $z^{(n)}$}
  }
}
Normalize each token of $\bZ$\;
\Return $D_\phi(\bZ)$\;
\end{algorithm}

The deployed midpoint sampler (\cref{alg:generation}) steps with the exact per-token exponential map, so it has no retraction error; the midpoint velocity is carried from $T_{z_{1/2}}\cS^{d-1}$ to $T_z\cS^{d-1}$ by the projection $\Pi_z$ instead of parallel transport, which changes the step only at $O(dt^3)$. Every operation acts per token, since a global operation on $\cS^{Nd-1}$ would destroy the product structure. The code floors $1-t$ at $10^{-3}$, which is inactive on this grid. The scheme is second-order Runge--Kutta, with $O(dt^3)$ per-step error (global $O(dt^2)$) against $O(dt^2)$ ($O(dt)$ global) for a Riemannian Euler step, so fewer steps suffice for matched accuracy.

\section{STREAM Training Hyperparameters}
\label{app:implementation}

\paragraph{Architectures.}
\Cref{tab:arch} gives the four components of the pipeline and their configurations; the training settings that drive them are in \cref{tab:hparams}.

\begin{table}[!ht]
  \caption{Component architectures of the deployed STREAM pipeline.}
  \label{tab:arch}
  \centering
  \small
  \setlength{\tabcolsep}{4pt}
  \begin{tabular}{@{}p{0.13\linewidth} >{\raggedright\arraybackslash}p{0.21\linewidth} p{0.58\linewidth}@{}}
    \toprule
    Component & Backbone & Configuration \\
    \midrule
    DiT generator & LightningDiT-XL~\citep{yao2025vavae} & 28 layers, hidden $1152$, 16 heads, SwiGLU feed-forward network, RMSNorm, rotary positional embeddings, AdaLN-Zero modulation (time-only), $\sim\!676$M parameters; per-token $\ell_2$-normalized output for x-prediction on $\cS^{d-1}$ \\
    Decoder & ViT-XL decoder of RAE~\citep{zheng2025rae} & Tanh output mapped to $[-1,1]$ \\
    Encoder & UNI ViT-L/16~\citep{chen2024uni} & Frozen; 256 $\ell_2$-normalized tokens of dimension 1024 from $256\!\times\!256$ patches \\
    Discriminator & Kaiko ViT-S/8~\citep{kaiko2024towards} & A histopathology-specific ViT pretrained on large-scale tissue patches, with frozen backbone and trainable spectrally-normalized projection heads, so the adversarial signal is computed in a histopathology-aware feature space rather than a natural-image one \\
    \bottomrule
  \end{tabular}
\end{table}

\paragraph{Fair-comparison protocol.}
Baselines vary in batch size and reported step count. To equalize compute, we train every baseline and STREAM for the ImageNet-equivalent of 80 epochs, setting each method's number of optimization steps from its batch size and the dataset size. All hyperparameters other than schedule length follow each method's published configuration.

\begin{table}[t]
  \caption{STREAM training hyperparameters.}
  \label{tab:hparams}
  \centering
  \small
  \setlength{\tabcolsep}{4pt}
  \begin{tabular}{@{}p{0.20\linewidth} p{0.36\linewidth} p{0.40\linewidth}@{}}
    \toprule
    Setting & Stage 1 (DiT) & Stage 2 (Decoder) \\
    \midrule
    Architecture & LightningDiT-XL & ViT-XL decoder + Kaiko disc.\ \\
    Parameters & $\sim\!676$M & ViT-XL + frozen UNI + frozen Kaiko \\
    Loss & Chord loss \cref{eq:vloss_chord} & $\ell_1$ + LPIPS + cos round-trip + GAN \\
    Loss weights & --- & $\lambda_{\ell_1}{=}1.0,\ \lambda_\text{LPIPS}{=}0.5,\ \lambda_\text{cos}{=}0.5,\ \omega_G{=}0.75$ \\
    Time sampling & Logit-Normal$(-0.5,1)$ on $[\epsilon_t,1{-}\epsilon_t]$, $\epsilon_t{=}10^{-5}$ & --- \\
    Bridge $\smax$ & $0.01$ & --- \\
    Noise injection & --- & Anisotropic SVD (\cref{eq:aniso_cov}): $\sigma_H{=}0.002$, $\sigma_L{=}0.02$, $\tau{=}0.90$ \\
    Spectral basis & --- & $K{=}50$ centroids, $m{=}600$ tangent probes, time anchor $t{=}0.99$, deployed $k^\ast{=}236$, $\dim U_L = d-1-k^\ast = 787$ \\
    Optimizer & AdamW, $\beta{=}(0.9,0.95)$ & AdamW, $\beta{=}(0.9,0.95)$, weight decay $0.05$ \\
    Learning rate & $4\!\times\!10^{-4}$, constant & $2\!\times\!10^{-4}\!\to\!2\!\times\!10^{-5}$ cosine \\
    Warmup & 5k steps & 3k steps \\
    Schedule & --- & 3-phase: A $(0\text{--}17.5\%)$ recon; B $(17.5\text{--}20\%)$ disc; C $(20\text{--}100\%)$ full \\
    Steps & 95k & 90k \\
    Batch size & 1024 & 512 \\
    Grad clip & $5.0$ & --- \\
    Exponential moving average decay & $0.9999$ & $0.9978$ \\
    Generation & 25 midpoint Euler steps & --- \\
    \bottomrule
  \end{tabular}
\end{table}

\paragraph{SVD of the velocity-field Jacobian (computation).}
The cached basis is extracted at a single time anchor, $t = 0.99$: one basis per (centroid, token), neither time-averaged nor recomputed during sampling. The $1/(1-t)$ prefactor of \cref{sec:bridge} is a positive scalar, so it rescales every singular value equally and moves no singular vector, no energy ratio and no $k^\ast$; anchoring near $t{=}1$ is therefore not by itself a blow-up artifact. The single-anchor cache is an approximation, and \cref{app:anchor} reports how much the basis moves with the anchor.

Extraction uses \texttt{torch.func.jvp} for batched forward-mode AD on the DiT's EMA weights, with $v_\theta = \Log_{\bZ}(\hat{\bZ}_1)/(1-t)$ at $t = 0.99$. For centroid $c$ and token $n$, $m = 600$ Gaussian probes (seed $cN + n$) are projected onto $T_{\bar{z}_c^{(n)}}\cS^{d-1}$ and unit-normalized, not orthogonalized, and perturb token $n$ only; the token-$n$ block of the $m$ Jacobian-vector products (JVPs), projected onto the same tangent space, forms a $d \times m$ response matrix whose exact thin SVD gives the cached vectors and spectrum $s_i$. Because forward-mode JVPs apply $J$ rather than $J^\top$, the returned vectors are left singular vectors spanning $\operatorname{range}(J)$. \Cref{app:svd_factor} retrains the decoder on the right singular vectors instead and finds the left basis better by $1.15$ rFID and $1.25$ gFID.

One $k^\ast = 236$ serves every cell: the $\tau = 0.90$ cutoff of the element-wise median of the $50 \times 256$ cached spectra. $U_L$ is the orthogonal complement of $U_H$ in $T_{\bar{z}_c^{(n)}}\cS^{d-1}$, of dimension $d-1-k^\ast = 787$. For the centroids, $50{,}000$ random training features (seed $42$) are mean-pooled over tokens, $\ell_2$-normalized and clustered into $K = 50$ groups by mini-batch $k$-means; $\bar{\bZ}_c \in (\cS^{d-1})^N$ is its group's per-token mean, renormalized per token. In decoder training each sample goes to the centroid maximizing $\sum_n \langle z^{(n)}, \bar{z}_c^{(n)} \rangle$ on its full token grid (\cref{alg:decoder_training}).

\subsection{Computational Cost}
\label{app:cost}

All figures for the training-time computational cost for STREAM (\cref{tab:gpu_hours}) are measured from our training logs and file sizes for the deployed STREAM configuration on one node of $8\times$ NVIDIA H200 GPUs. The STREAM-specific cost is the JVP and SVD stage alone: $69$~GPU-h, $12.7\%$ of the pipeline. Everything else is what any standard latent generative model with a learned diffusion model and a pixel decoder at a similar size costs.

\begin{table}[!h]
  \centering
  \caption{Training-time cost of the deployed STREAM pipeline on BRCA.}
  \label{tab:gpu_hours}
  \footnotesize
  \setlength{\tabcolsep}{6pt}
  \begin{tabular}{@{}lcc@{}}
    \toprule
    Stage & Wall-clock & GPU-hours \\
    \midrule
    Encoder feature extraction & $5.2$~h & $42$ \\
    DiT training, 95k steps & $16.9$~h & $135$ \\
    $k$-means ($K{=}50$), JVP probes, SVD, basis caching & $8.6$~h & $69$ \\
    Decoder training, 90k steps & $37.4$~h & $299$ \\
    \midrule
    \textbf{Total} & $\mathbf{{\sim}68.1}$~\textbf{h} & $\mathbf{{\sim}545}$ \\
    \bottomrule
  \end{tabular}
\end{table}

\subsection{Inference: NFE and quality}
\label{app:nfe}

The deployed midpoint sampler evaluates $f_\theta$ twice per step (\cref{alg:generation}), so the $25$ steps behind \cref{tab:results} (\cref{tab:hparams}) correspond to $\mathrm{NFE} = 50$, and RAE and SVG are run at that same NFE. All points below use $50{,}000$ samples and three independent sampling runs (\cref{app:seeds}) on BRCA with the UNI encoder. They are a separate re-measurement of the models in \cref{tab:results}, which reports a single run at the same sample count, so the RAE and SVG values differ from it.

\begin{table}[!h]
  \centering
  \caption{gFID against NFE on BRCA with the UNI encoder, mean $\pm$ SD over 3 sampling seeds at $50{,}000$ samples per point. The midpoint sampler spends two DiT calls per step, so NFE $6$ and $26$ stand in for NFE $5$ and $25$, and all models run at the same NFE.}
  \label{tab:nfe}
  \footnotesize
  \setlength{\tabcolsep}{6pt}
  \begin{tabular}{@{}lccc@{}}
    \toprule
    NFE & STREAM & RAE & SVG \\
    \midrule
    $6$  & $49.70 \pm 0.09$ & $55.03 \pm 0.05$ & $55.80 \pm 0.07$ \\
    $10$ & $27.05 \pm 0.06$ & $29.87 \pm 0.06$ & $30.53 \pm 0.18$ \\
    $26$ & $9.31 \pm 0.02$  & $18.27 \pm 0.04$ & $18.82 \pm 0.15$ \\
    $50$ & $6.12 \pm 0.04$  & $15.88 \pm 0.04$ & $17.52 \pm 0.14$ \\
    \bottomrule
  \end{tabular}
\end{table}

STREAM leads at every NFE ($6$, $10$, $26$ and $50$), and at $\mathrm{NFE} = 26$ it already beats both baselines at $\mathrm{NFE} = 50$, $9.31$ against $15.88$ and $17.52$ (\cref{tab:nfe}). At the matched $\mathrm{NFE} = 50$ it is also the cheapest of the three to sample, at $12.14$ TFLOPs and $94.8$~ms per image against $16.09$ and $156.7$~ms for RAE and $24.16$ and $175.0$~ms for SVG (\cref{tab:flops}). STREAM therefore reaches the better sample quality at the lower inference cost, and reaches the baselines' $\mathrm{NFE} = 50$ quality on half their function evaluations.

\begin{table}[!h]
  \centering
  \caption{Inference cost per generated image at $\mathrm{NFE}=50$. FLOPs are counted over the complete sampling call (every network evaluation plus the decoder). Latency is wall-clock at a generation batch of $32$, divided by the batch, on a single NVIDIA H200.}
  \label{tab:flops}
  \footnotesize
  \setlength{\tabcolsep}{6pt}
  \begin{tabular}{@{}lcc@{}}
    \toprule
    Model & TFLOPs / image & ms / image \\
    \midrule
    \textbf{STREAM} & $\mathbf{12.14}$ & $\mathbf{94.8}$ \\
    RAE & $16.09$ & $156.7$ \\
    SVG & $24.16$ & $175.0$ \\
    \bottomrule
  \end{tabular}
\end{table}

\section{Additional Experiments and Ablations}
\label{app:controlled}

Every cell in this appendix uses UNI as the encoder on BRCA unless stated otherwise, so the encoder is held fixed and only geometry, bridge, decoder regularization and backbone move. Each table is internally budget-matched; comparisons are within-table throughout.

\subsection{Additional experiment: all six methods retrained on an independent non-TCGA dataset}
\label{app:spider_skin}

\begin{table}[!h]
  \centering
  \caption{All six methods retrained from scratch on SPIDER-skin, with the encoder held at UNI throughout.}
  \label{tab:encoder_controlled}
  \footnotesize
  \setlength{\tabcolsep}{6pt}
  \begin{tabular}{@{}lcc@{}}
    \toprule
    SPIDER-skin, encoder $=$ UNI & rFID $\downarrow$ & gFID $\downarrow$ \\
    \midrule
    \textbf{STREAM} & $4.65$ & $\mathbf{9.19}$ \\
    ZoomLDM & $7.85$ & $11.14$ \\
    REPA-E & $17.76$ & $18.00$ \\
    RAE & $7.12$ & $16.52$ \\
    SVG & $\mathbf{1.30}$ & $22.24$ \\
    PixCell & $3.32$ & $453.28$ \\
    \bottomrule
  \end{tabular}
\end{table}

SPIDER-skin~\citep{nechaev2025spider} is an independent non-TCGA skin dataset, and all six methods are retrained on it from scratch (\cref{tab:encoder_controlled}). STREAM reaches the best gFID at $9.19$, with ZoomLDM next at $11.14$, while SVG and PixCell hold the better rFID and perform poorly on gFID. As in \cref{tab:results}, STREAM gives the best generation quality at little cost to reconstruction.

\subsection{Ablation on prediction geometry}
\label{app:geometry_control}

We test whether the prediction is constrained to the manifold by training STREAM with identical design and hyperparameters, varying only the underlying geometry the model is trained on. In the Euclidean row the predicted endpoint is not $\ell_2$-normalized before the loss, so the model is fit as a Euclidean regressor instead of on the sphere. Both cells run at the full 90k decoder schedule rather than the matched 45k budget of \cref{tab:ablation}, so this pair is read within itself.

\begin{table}[!h]
  \centering
  \caption{Geometry control on BRCA. Both rows are STREAM; only the $\ell_2$ renormalization of the prediction differs.}
  \label{tab:geometry_control}
  \footnotesize
  \setlength{\tabcolsep}{6pt}
  \begin{tabular}{@{}lcc@{}}
    \toprule
     & rFID $\downarrow$ & gFID $\downarrow$ \\
    \midrule
    Euclidean: prediction not projected onto the sphere & $4.52$ & $16.28$ \\
    \textbf{Riemannian: projected (deployed)} & $\mathbf{2.23}$ & $\mathbf{6.16}$ \\
    \bottomrule
  \end{tabular}
\end{table}

Euclidean geometry vastly underperforms the deployed Riemannian geometry, at $16.28$ gFID against $6.16$ (\cref{tab:geometry_control}). However, geometry is necessary but not sufficient for state-of-the-art generation. The first row of \cref{tab:ablation}, at the 45k decoder budget, is spherical FM with neither of our STREAM components, and it performs worse than the ZoomLDM baseline of \cref{tab:results}. This control holds the model completely fixed and therefore excludes architectural confounds; the Euclidean grid of \cref{app:euclidean_2x2} does not, and we keep the two claims separate for that reason.

\subsection{Ablation on STREAM's components in a Euclidean pipeline}
\label{app:euclidean_2x2}

\begin{table}[!h]
  \centering
  \caption{Both components rebuilt RAE-style in a Euclidean latent space (rFID / gFID), BRCA, UNI encoder.}
  \label{tab:euclidean_2x2}
  \footnotesize
  \setlength{\tabcolsep}{6pt}
  \begin{tabular}{@{}lcc@{}}
    \toprule
    RAE & Isotropic & Anisotropic \\
    \midrule
    No bridge & $6.00$ / $24.39$ & $5.80$ / $25.38$ \\
    Bridge    & $6.00$ / $23.89$ & $5.82$ / $24.06$ \\
    \bottomrule
  \end{tabular}
\end{table}

Next, we ask whether STREAM's components of bridge perturbation and anisotropic decoder can be applied to our closest Euclidean counterpart, RAE. The same ablation as \cref{tab:ablation} performed on RAE shows that adding both components brings no benefit (\cref{tab:euclidean_2x2}), empirical evidence that STREAM's design is specific to the sphere. These four cells are not comparable to the paper-faithful RAE baseline of \cref{tab:results}, because they use a plain LightningDiT-XL without the Decoupled Diffusion Transformer (DDT) head~\citep{wang2025ddt}, so that latent dimensions match STREAM and the SVD extraction is clean.

\subsection{Ablation on prediction target and loss}
\label{app:parameterization}

We ask which pairing of loss space and prediction target the sphere selects, crossing an x-space (chord) loss against a v-space (geodesic) loss and endpoint prediction against velocity prediction. On $\cS^{d-1}$ the metric is the Euclidean inner product restricted to the tangent space, so a velocity loss admits no extrinsic-versus-intrinsic choice and the chord-versus-geodesic choice arises only once endpoints are regressed. RFM methods usually define their target in tangent space, whether the network emits a velocity~\citep{chen2024riemannian} or an endpoint, and none regresses endpoints in the ambient chord metric, so running this comparison at our manifold dimension ($N{=}256$ spheres at $d{=}1024$, i.e.\ $261{,}888$ tangent dimensions) is, to our knowledge, new.

\begin{table}[!h]
  \centering
  \caption{Stage-1 prediction/loss parameterization on BRCA, full 90k decoder schedule, single checkpoint.}
  \label{tab:parameterization}
  \footnotesize
  \setlength{\tabcolsep}{6pt}
  \begin{tabular}{@{}lcc@{}}
    \toprule
    Parameterization & rFID $\downarrow$ & gFID $\downarrow$ \\
    \midrule
    \textbf{x loss (chord) $+$ x-prediction (ours)} & $\mathbf{2.23}$ & $\mathbf{6.16}$ \\
    Standard RFM: v loss (geodesic) $+$ v-prediction & $4.78$ & $7.84$ \\
    x loss (chord) $+$ v-prediction & $4.73$ & $11.26$ \\
    v loss (geodesic) $+$ x-prediction & $3.12$ & $12.57$ \\
    \bottomrule
  \end{tabular}
\end{table}

The deployed pairing beats standard RFM on both metrics, $2.23$ against $4.78$ rFID and $6.16$ against $7.84$ gFID, and the two mismatched cells give the worst gFID of the four combinations (\cref{tab:parameterization}).

\subsection{Ablation on discriminator}
\label{app:discriminator}

We test whether the histopathology-specific Kaiko ViT-S/8 that supplies the decoder's adversarial term carries any part of STREAM's margin. We substitute DINO ViT-S/8 rather than a PatchGAN because DINO ViT-S/8 is the discriminator RAE uses, and RAE is the pipeline STREAM's decoder is built on, which makes it the field-standard control for this architecture and a stronger baseline than PatchGAN from the same natural-image domain.

\begin{table}[!h]
  \centering
  \caption{Discriminator ablation on BRCA at 45k decoder steps, everything else identical.}
  \label{tab:discriminator}
  \footnotesize
  \setlength{\tabcolsep}{6pt}
  \begin{tabular}{@{}lcc@{}}
    \toprule
    Discriminator & rFID $\downarrow$ & gFID $\downarrow$ \\
    \midrule
    Kaiko ViT-S/8 (deployed) & $\mathbf{3.52}$ & $\mathbf{6.86}$ \\
    DINO ViT-S/8 (RAE's) & $4.33$ & $7.21$ \\
    \bottomrule
  \end{tabular}
\end{table}

The domain-matched discriminator improves STREAM by $0.81$ rFID and $0.35$ gFID (\cref{tab:discriminator}), far less than STREAM's margin over the Euclidean baselines in \cref{tab:results}, where RAE reaches $15.84$ gFID against STREAM's $6.16$, so the adversarial prior does not explain that margin. Giving RAE the Kaiko discriminator in place of its own DINO ViT-S/8, on identical UNI embeddings and BRCA, raises its gFID from $15.84$ to $16.81$, so the domain-matched discriminator gives RAE no advantage.

\subsection{Reproducibility across sampling seeds}
\label{app:seeds}

We ask how far the reported numbers move when the sampling seed is the only thing that changes.

We report SDs over sampling seeds. Each seed redraws $\bZ_0 \sim \mathrm{Uniform}((\cS^{d-1})^N)$ and re-runs generation and evaluation, while the trained weights are held fixed. We report gFID only, because rFID reconstructs a fixed set of real images and is seed-independent.

Five independent sampling seeds at 10k samples each estimate the spread, and because FID depends on sample count, the reported means are the 50k-sample values with the seed SD beside them. \Cref{tab:seeds_ablation} reports the four ablation cells of \cref{tab:ablation} under this protocol. We additionally re-measure end-to-end at the full 50k-sample count over three seeds, which is the stricter check and is reported in \cref{tab:seeds_50k}.

\begin{table}[!h]
  \centering
  \caption{gFID at 50k samples (as in \cref{tab:ablation}) $\pm$ SD over 5 sampling seeds of 10k samples each, for the four cells of \cref{tab:ablation} at the 45k decoder budget on BRCA.}
  \label{tab:seeds_ablation}
  \footnotesize
  \setlength{\tabcolsep}{6pt}
  \begin{tabular}{@{}lc@{}}
    \toprule
    Ablation cell & gFID (50k) $\pm$ seed SD \\
    \midrule
    No bridge, isotropic & $9.07 \pm 0.24$ \\
    No bridge, anisotropic & $8.27 \pm 0.21$ \\
    Bridge, isotropic & $9.37 \pm 0.15$ \\
    Bridge, anisotropic (ours) & $\mathbf{6.86 \pm 0.18}$ \\
    \bottomrule
  \end{tabular}
\end{table}

\begin{table}[!h]
  \centering
  \caption{gFID re-measured end-to-end at the full $50{,}000$-sample count, mean $\pm$ SD over 3 sampling seeds, BRCA. All three models use the shared UNI encoder of \cref{tab:results}. The RAE and SVG values differ from \cref{tab:results} because this is a three-seed re-measurement rather than the single run reported there.}
  \label{tab:seeds_50k}
  \footnotesize
  \setlength{\tabcolsep}{6pt}
  \begin{tabular}{@{}lc@{}}
    \toprule
    Model & gFID at 50k, mean $\pm$ SD \\
    \midrule
    \textbf{STREAM} & $\mathbf{6.12 \pm 0.04}$ \\
    RAE & $15.88 \pm 0.04$ \\
    SVG & $17.52 \pm 0.14$ \\
    \bottomrule
  \end{tabular}
\end{table}

RAE and SVG reproduce their reported means to within $0.15$ gFID (\cref{tab:seeds_50k}). STREAM's three-seed $50{,}000$-sample mean of $6.12 \pm 0.04$ lies within $0.05$ of the \cref{tab:results} value of $6.16$. Its five-seed $10{,}000$-sample SD is $0.20$. STREAM's seed statistics are measured on the deployed step-090000 decoder. The four ablation cells of \cref{tab:seeds_ablation} are the only ones repeated across sampling seeds, and every other ablation cell is a single run.

\section{Downstream Evaluation of Generated Images}
\label{app:downstream}

Distribution metrics cannot on their own establish diagnostic utility. Since STREAM is unconditional and there are not many public labelers (i.e.\ classification models) that are robust and clinically validated enough for a downstream evaluation, the evaluations are therefore built either on the generated pixels directly or on a label-free pretraining setting. \Cref{app:morphometry} and \cref{app:ssl} both corroborate that STREAM-generated images are indeed biologically faithful.

\subsection{Evaluation: Nuclear morphometry}
\label{app:morphometry}

We run HoVer-NeXt~\citep{baumann2024hovernext}, a published pretrained nuclei segmentation-and-classification model, over 50k real and 50k STREAM-generated COADREAD patches (${\sim}2.5$M nuclei per arm). Shape descriptors follow scikit-image \texttt{regionprops} conventions~\citep{vanderwalt2014scikit}, with perimeter by the Crofton estimator; circularity follows the definition of \citet{janssen2022nuclear}. Nuclei per tile is HoVer-NeXt's instance count.

The real set is halved, so real~A against real~B gives a tile-disjoint agreement floor.

\begin{table}[!h]
  \centering
  \caption{Nuclear morphometry on COADREAD, mean values over ${\sim}2.5$M nuclei per arm.}
  \label{tab:morphometry}
  \footnotesize
  \setlength{\tabcolsep}{6pt}
  \begin{tabular}{@{}lccc@{}}
    \toprule
    Mean, COADREAD & real A & real B & STREAM \\
    \midrule
    nuclei per tile & $51.87$ & $51.49$ & $55.36$ \\
    area (px$^2$) & $195.4$ & $197.1$ & $195.2$ \\
    perimeter (px) & $49.61$ & $49.78$ & $49.91$ \\
    circularity & $0.840$ & $0.840$ & $0.846$ \\
    eccentricity & $0.741$ & $0.742$ & $0.737$ \\
    solidity & $0.935$ & $0.935$ & $0.937$ \\
    major axis (px) & $19.11$ & $19.18$ & $19.13$ \\
    minor axis (px) & $11.14$ & $11.17$ & $11.36$ \\
    \bottomrule
  \end{tabular}
\end{table}

As \cref{tab:morphometry} shows, STREAM's mean sits within $2.0\%$ of real on all seven per-nucleus shape descriptors and within $7.5\%$ on nuclei per tile. Nuclear size and shape are the features a pathologist reads, so agreement at this level indicates that STREAM's samples are biologically faithful, while a blinded reader study with board-certified pathologists would be the ideal evaluation of actual clinical relevance.

\subsection{Evaluation: Self-supervised pretraining and downstream evaluations}
\label{app:ssl}

We believe that STREAM is best suited to this image augmentation setting that needs no labels and can help train foundation models by increasing the pretraining corpus with STREAM-generated synthetic data. We mock this by training a proof-of-concept DINO ViT-S/16~\citep{caron2021emerging} model from scratch for $75$ epochs under identical hyperparameters on 1M STREAM-generated tiles per organ (BRCA and COADREAD). As a fair comparison, we also train the same model separately on 1M real tiles. Then, we evaluate the two on downstream tasks with a random-init ViT-S/16 as the control. The two DINO models act as encoders and they are frozen and probed identically so only embedding quality is compared. We report linear-probe AUROC on SPIDER~\citep{nechaev2025spider} with no slide contamination between train and test splits, plus MHIST~\citep{wei2021mhist} for COADREAD. Both probe sets are out-of-distribution (no TCGA overlap), while both encoders, real or STREAM-generated, see only TCGA tiles during DINO training. The nine linear probes are reported in \cref{tab:downstream}. It is safe to conclude that an encoder pretrained only on STREAM-generated images is essentially on par with one pretrained on real images.

\begin{table}[t]
  \caption{Downstream linear-probe evaluation of frozen DINO ViT-S/16 encoders pretrained only on 1M STREAM-generated tiles vs.\ 1M real TCGA tiles per organ. AUROC on out-of-distribution SPIDER/MHIST binary tasks. HP is hyperplastic polyp, SSA is sessile serrated adenoma, HG is high grade and LG is low grade. Rand.\ is a random-init control averaged over 3 network draws (5 for MHIST).}
  \label{tab:downstream}
  \centering
  \scriptsize
  \setlength{\tabcolsep}{3pt}
  \begin{tabular}{@{}l ccc@{}}
    \toprule
    Linear-probe AUROC $\uparrow$ & Rand. & \textbf{STREAM} & Real \\
    \midrule
    SPIDER COADREAD Adenoma vs.\ Adenoca.           & 0.890 & 0.952 & 0.958 \\
    SPIDER COADREAD Adenoca.\ HG vs.\ LG            & 0.883 & 0.964 & 0.959 \\
    SPIDER COADREAD Adenoma HG vs.\ LG              & 0.668 & 0.951 & 0.943 \\
    SPIDER COADREAD HP vs.\ sessile serrated lesion & 0.926 & 0.977 & 0.979 \\
    MHIST COADREAD HP vs.\ SSA                      & 0.612 & 0.879 & 0.867 \\
    SPIDER BRCA DCIS HG vs.\ LG                     & 0.976 & 0.998 & 0.998 \\
    SPIDER BRCA NST vs.\ Lobular                    & 0.766 & 0.953 & 0.966 \\
    SPIDER BRCA In-situ vs.\ Invasive               & 0.935 & 0.994 & 0.994 \\
    SPIDER BRCA Benign vs.\ Malignant               & 0.737 & 0.979 & 0.983 \\
    \bottomrule
  \end{tabular}
\end{table}

Classification is close to saturated in histopathology, as the high AUROC values in \cref{tab:downstream} show, so we run a more difficult downstream task on the same pretrained encoders. HEST~\citep{jaume2024hest} pairs hematoxylin and eosin (H\&E) tiles with matched spatial transcriptomics; the task is to predict, from a tile's frozen embedding alone, the expression of each of the top-50 highly variable genes. We fit a ridge regression and a gradient-boosted regressor and report the average Pearson correlation coefficient (PCC), per gene and averaged over the 50 genes and the folds, on the official uncontaminated split, with results in \cref{tab:hest_pcc}. Across both datasets (BRCA and COADREAD) and downstream model (Ridge, XGBoost), STREAM again performs on par with real on average. We believe that this is stronger evidence of the biological fidelity of STREAM-generated synthetic images, because morphology only predicts correct expression when the embedding has captured underlying biologically meaningful features.

\begin{table}[t]
  \caption{HEST spatial-transcriptomics regression on the same frozen encoders. Values are the average PCC over the top-50 highly variable genes. Rand.\ is a random-init control averaged over 3 network draws.}
  \label{tab:hest_pcc}
  \centering
  \scriptsize
  \begin{tabular*}{\textwidth}{@{\extracolsep{\fill}}l cc cc cc@{}}
    \toprule
    & \multicolumn{2}{c}{Rand.} & \multicolumn{2}{c}{\textbf{STREAM}} & \multicolumn{2}{c}{Real} \\
    \cmidrule(lr){2-3} \cmidrule(lr){4-5} \cmidrule(lr){6-7}
    HEST PCC $\uparrow$ & COADREAD & BRCA & COADREAD & BRCA & COADREAD & BRCA \\
    \midrule
    Ridge   & 0.144 & 0.163 & 0.284 & 0.313 & 0.280 & 0.333 \\
    XGBoost & 0.128 & 0.145 & 0.247 & 0.307 & 0.269 & 0.283 \\
    \bottomrule
  \end{tabular*}
\end{table}

While these two downstream evaluations are not clinically validated, STREAM offers the first steps of an actual deployment scenario by addressing two constraints that \citet{vanderlaak2021pathclinic} document in computational pathology. Privacy and approval requirements narrow the reuse of existing data for the development of artificial-intelligence models and raise the cost of assembling enough cases, and collecting annotations remains an obstruction to scaling algorithms. A generative model like STREAM transfers morphology without transferring any patient's tissue, which addresses the first. Because annotation is the scarce resource, an unconditional generative model matches the second constraint rather than falling short of it, where a conditional one would require exactly the labels that are hardest to obtain. The self-supervised result above is that workflow end to end, with a receiving institution pretraining an encoder on transferred synthetic tiles and reaching real-data quality while holding neither a label nor a patient slide.

\section{Ablations and Experiments on Decoder Noise and the Spectral Basis}
\label{app:sigmaL_ablation}

\subsection{Ablation on Low-Energy Noise Magnitude \texorpdfstring{$\sigma_L$}{σ\_L}}
\label{app:sigmaL}

We test how much noise the low-energy directions need by sweeping $\sigma_L$ on BRCA at 45k decoder training steps, holding everything else at the deployed configuration. The deployed $0.02$ gives the best gFID at $6.86$, while $0.01$ reaches a better rFID at $2.41$ and $0.03$ is worse on both (\cref{tab:sigmaL_ablation}). Raising $\sigma_L$ trains the decoder to denoise larger perturbations, so its reconstruction of clean inputs degrades and rFID rises throughout the sweep. gFID does not follow, and the deployed value is at the optimal minimum point.

\begin{table}[!h]
  \caption{Decoder $\sigma_L$ ablation on BRCA at 45k decoder steps. Bridge ($\smax{=}0.01$) and $\sigma_H{=}0.002$ held fixed.}
  \label{tab:sigmaL_ablation}
  \centering
  \footnotesize
  \setlength{\tabcolsep}{6pt}
  \begin{tabular}{@{}lcc@{}}
    \toprule
    $\sigma_L$ & rFID $\downarrow$ & gFID $\downarrow$ \\
    \midrule
    $0.01$           & \textbf{2.41}    & 7.39          \\
    $\mathbf{0.02}$ \emph{(deployed)} & 3.52 & \textbf{6.86} \\
    $0.03$           & 5.12             & 7.97          \\
    \bottomrule
  \end{tabular}
\end{table}

\subsection{Ablation on Energy Threshold \texorpdfstring{$\tau$}{τ}}
\label{app:tau}

We test where to cut the spectrum by ablating the energy threshold $\tau$, which sets $k^\ast$ and therefore the split between $U_H$ and $U_L$, at the matched 45k decoder budget with $\sigma_H = 0.002$, $\sigma_L = 0.02$ and the deployed bridge held fixed. Lowering $\tau$ from $0.90$ to $0.80$ costs $2.29$ rFID and $1.94$ gFID (\cref{tab:tau_ablation}). The likely mechanism is that a lower $\tau$ shrinks $k^\ast$ and moves the least energetic of the retained directions into the low-energy subspace, where they meet the far larger $\sigma_L$.

\begin{table}[!h]
  \centering
  \caption{Energy-threshold ablation on BRCA at 45k decoder steps.}
  \label{tab:tau_ablation}
  \footnotesize
  \setlength{\tabcolsep}{6pt}
  \begin{tabular}{@{}lcc@{}}
    \toprule
    $\tau$ & rFID $\downarrow$ & gFID $\downarrow$ \\
    \midrule
    $\mathbf{0.90}$ \emph{(deployed)} & $\mathbf{3.52}$ & $\mathbf{6.86}$ \\
    $0.80$ & $5.81$ & $8.80$ \\
    \bottomrule
  \end{tabular}
\end{table}

\subsection{Ablation on Left vs.\ Right Singular Vectors}
\label{app:svd_factor}

We test whether the decoder's noise should be shaped by the left or the right singular vectors of the Jacobian. Because the noise enters the decoder's input space, the right singular vectors look like the natural choice. $J$ is an endomorphism of one token's tangent space, so both bases span that space and both select $k^\ast = 236$ at $\tau = 0.90$, differing only in which $236$ directions they pick. Forward-mode JVPs apply $J$ rather than $J^\top$, so the deployed cache holds the left singular vectors, and $\operatorname{range}(J)$ is where the velocity field's local variation lives, which makes $U_H$ the directions of largest field response.

Retraining a full anisotropic decoder on the right singular vectors while holding everything else fixed costs $1.15$ rFID and $1.25$ gFID (\cref{tab:svd_factor}). On the right basis too, shaping the noise by the spectrum beats isotropic noise, by $1.84$ rFID and $1.26$ gFID, and both spectral partitions inject less total variance than the isotropic baseline rather than more. Spectral shaping therefore gives most of the rFID gain, and the shaping and the switch to the left basis give similar gFID gains.

\begin{table}[!h]
  \centering
  \caption{Left vs.\ right singular vector ablation on BRCA at 45k decoder steps.}
  \label{tab:svd_factor}
  \footnotesize
  \setlength{\tabcolsep}{6pt}
  \begin{tabular}{@{}lcc@{}}
    \toprule
    Decoder noise basis & rFID $\downarrow$ & gFID $\downarrow$ \\
    \midrule
    \textbf{Left singular vectors, forward-mode (deployed)} & $\mathbf{3.52}$ & $\mathbf{6.86}$ \\
    Right singular vectors, reverse-mode & $4.67$ & $8.11$ \\
    Isotropic, $0.02$ on every direction & $6.51$ & $9.37$ \\
    \bottomrule
  \end{tabular}
\end{table}

\subsection{Projected Energy of the Prediction Residual}
\label{app:residual}

The anisotropic allocation comes from the DiT's learned velocity field. The SVD of the field's Jacobian separates the directions of largest response, $U_H$, from the directions of smallest response, $U_L$ (\cref{sec:aniso}). The decoder trains under only $\sigma_H = 0.002$ along $U_H$, while $U_L$ absorbs the large robustness budget $\sigma_L = 0.02$. We measure where the generator's error lands, which is a separate question. The allocation does not require that error to concentrate in $U_L$.

For each of $2{,}000$ BRCA tiles we encode $z_1$ with UNI, form the deployed bridge interpolant $x_t$ from $x_0$ drawn uniformly on $(\cS^{d-1})^N$, and query the deployed DiT once. The residual $e = \Pi_{z_1}(\hat z_1 - z_1)$ is the tangent part of this teacher-forced prediction error. We parallel-transport $e$ into the tangent space of the tile's assigned centroid, where the deployed $t{=}0.99$ cache is orthonormal. For token $n$, $E_{H,n}$ is the energy of $e_n$ along the top-$k$ cached left singular vectors of that (centroid, token), and $E_{L,n} = \lVert e_n\rVert^2 - E_{H,n}$ is the energy in their orthogonal complement. $E_H$ and $E_L$ sum these over the $256$ tokens and average over tiles. The per-dimension density ratio
\[
  \rho = \frac{E_L/\dim U_L}{E_H/\dim U_H}
\]
equals $1$ when the residual spreads evenly across directions. The tiles come from the training distribution, since no held-out split exists for the deployed BRCA DiT.

\begin{table}[!h]
  \centering
  \caption{Projected energy of the DiT's teacher-forced endpoint residual on BRCA, $n{=}2{,}000$ tiles, UNI encoder, projected onto the deployed $t{=}0.99$ cache. $E_H$ and $E_L$ are per-tile sums over the $256$ tokens averaged over tiles; the $U_L$ share is $E_L/(E_H{+}E_L)$ and the even-spread share is $\dim U_L/1023$; $\rho$ is the per-dimension density ratio, equal to $1$ under an even spread. The upper block varies the partition at the deployed anchor and the lower block varies the anchor at the deployed partition. The per-token row uses each (centroid, token)'s own $k^\ast$, and its dimensions are means over the $2{,}000$ tiles.}
  \label{tab:residual_energy}
  \footnotesize
  \setlength{\tabcolsep}{4pt}
  \begin{tabular}{@{}lccccccc@{}}
    \toprule
    & $\dim U_H$ & $\dim U_L$ & $E_H$ & $E_L$ & share in $U_L$ & even spread & $\rho$ \\
    \midrule
    \multicolumn{8}{@{}l}{\emph{Partition varied at the deployed anchor $t{=}0.99$}} \\
    $k{=}32$ & $32$ & $991$ & $0.0563$ & $0.6686$ & $92.2\%$ & $96.9\%$ & $0.383$ \\
    $k{=}64$ & $64$ & $959$ & $0.0931$ & $0.6318$ & $87.2\%$ & $93.7\%$ & $0.453$ \\
    $k{=}128$ & $128$ & $895$ & $0.1474$ & $0.5775$ & $79.7\%$ & $87.5\%$ & $0.560$ \\
    $k{=}236$ (deployed) & $236$ & $787$ & $0.2254$ & $0.4995$ & $68.9\%$ & $76.9\%$ & $0.664$ \\
    $k{=}380$ & $380$ & $643$ & $0.3218$ & $0.4031$ & $55.6\%$ & $62.9\%$ & $0.740$ \\
    per-token $k^\ast$ & $237.9$ & $785.1$ & $0.2260$ & $0.4989$ & $68.8\%$ & $76.7\%$ & $0.669$ \\
    \midrule
    \multicolumn{8}{@{}l}{\emph{Anchor varied at the deployed partition $k{=}236$}} \\
    $t{=}0.5$ & $236$ & $787$ & $4.1103$ & $12.4434$ & $75.2\%$ & $76.9\%$ & $0.908$ \\
    $t{=}0.7$ & $236$ & $787$ & $2.4151$ & $7.7522$ & $76.2\%$ & $76.9\%$ & $0.963$ \\
    $t{=}0.9$ & $236$ & $787$ & $0.9256$ & $3.2628$ & $77.9\%$ & $76.9\%$ & $1.057$ \\
    \bottomrule
  \end{tabular}
\end{table}

The residual does not concentrate in $U_L$ (\cref{tab:residual_energy}). At the deployed anchor $\rho$ stays below $1$ at every partition, the deployed one included, so the residual is denser per dimension in $U_H$. At the deployed partition $\rho$ sits near $1$ away from that anchor. A density below $1$ does not undercut the allocation. To first order the decoder responds to the total energy in each subspace rather than to its density, and the larger total lands in $U_L$, where the large noise goes.

\subsection{Directional-Lipschitz Probe (Extended)}
\label{app:dirlip}

\Cref{tab:dirlip_rae} runs the directional-Lipschitz probe of \cref{tab:panel_d} on the Euclidean RAE pipeline, which removes the spherical geometry while keeping the latent dimension, the decoder family and the noise magnitudes. On the sphere, the no-bridge $k{=}64$ row of \cref{tab:panel_d} reverses the same way as the bridge row, from $0.62$ isotropic to $2.50$ anisotropic, so the reversal belongs to the spherical geometry and the anisotropic allocation, and the bridge sharpens it rather than creating it. The RAE ratios show no directional split, and anisotropic training does not create one. The same absence explains the null result of \cref{app:euclidean_2x2}, because without the spherical geometry there is no $U_H/U_L$ split for the anisotropic decoder to act on.

\begin{table}[!h]
  \centering
  \caption{Directional-Lipschitz ratio $R_\text{dec} = L_{U_H}^\text{dec}/L_{U_L}^\text{dec}$ for the Euclidean RAE pipeline on BRCA, $n{=}2{,}000$ tiles, as the control for \cref{tab:panel_d}. Bridge rows use the stochastic bridge and no-bridge rows do not; $k^\ast{=}380$ is RAE's own $\tau{=}0.90$ energy cutoff, the counterpart of STREAM's $236$. These rows use a plain LightningDiT-XL without the DDT head so that latent dimensions match STREAM and the SVD extraction is clean.}
  \label{tab:dirlip_rae}
  \footnotesize
  \setlength{\tabcolsep}{6pt}
  \begin{tabular}{@{}lcc@{}}
    \toprule
    Pipeline & Isotropic & Anisotropic \\
    \midrule
    RAE, no bridge ($k{=}64$)        & $0.982$  & $0.975$  \\
    RAE, no bridge ($k^\ast{=}380$)  & $0.996$  & $0.992$  \\
    RAE, bridge ($k{=}64$)           & $0.9863$ & $0.9797$ \\
    RAE, bridge ($k^\ast{=}380$)     & $0.9979$ & $0.9928$ \\
    \bottomrule
  \end{tabular}
\end{table}

\subsection{Directional Response in Pixel Space}
\label{app:dirlip_pixel}

\Cref{fig:dirlip_pixel_top64,fig:dirlip_pixel_complement} show in pixels the ablation of \cref{tab:ablation} and the directional-LPIPS values of \cref{tab:panel_d}; \cref{fig:dirlip_pixel_top64} is the full version of \cref{fig:dirlip_pixel_top64_main}, with the $1\times$ and $8\times$ columns and the $|\Delta|$ rows added. We take $16$ latents generated by the bridge DiT of \cref{tab:ablation}, and the figures display one of them. For each latent we perturb all $256$ tokens inside a subspace of that latent's own per-(centroid, token) frame and renormalize each token onto $\cS^{d-1}$. The perturbation is drawn once, and the isotropic and anisotropic bridge decoders of \cref{tab:ablation}, both at 45k steps, decode the same perturbed point. \Cref{fig:dirlip_pixel_top64} perturbs the leading $64$ directions of the frame, the $k{=}64$ partition of \cref{tab:panel_d} rather than the deployed $k^\ast{=}236$ one. \Cref{fig:dirlip_pixel_complement} perturbs $U_L$, the $787$-dimensional orthogonal complement of the deployed $U_H$ in the $1023$-dimensional tangent space. We draw it as \cref{tab:panel_d} does, projecting an ambient Gaussian onto the tangent space at each token and removing its $U_H$ component.

The amplitude unit is $\epsilon_H = \sigma_H\sqrt{k^\ast} = 0.0307$, the per-token displacement that the deployed covariance of \cref{eq:aniso_cov} places inside $U_H$. Column $m\times$ displaces every token by $m\,\epsilon_H$ in its tangent space. The deployed displacement inside $U_L$, $\sigma_L\sqrt{787} = 0.56$, falls between the $8\times$ and $32\times$ columns. The statistic is the mean absolute pixel difference from each decoder's own clean decode on the $0$--$255$ scale. It is not LPIPS, and its magnitude is not comparable with $R_\text{dec}$. For display, each $|\Delta|$ panel is divided by its own displacement $m\,\epsilon_H$ and multiplied by one factor shared by both figures, so black is no change and brightness compares across the two figures.

On the top-$64$ block the anisotropic decoder clearly separates from the isotropic one and responds more strongly to the perturbation (\cref{fig:dirlip_pixel_top64}). On $U_L$, the control, both decoders train under the same $\sigma_L$ and behave the same, as expected (\cref{fig:dirlip_pixel_complement}). We read the separation as an empirical directional-robustness effect, not as an established principle. Nothing is visible at the deployed amplitude ($1\times$); in the H\&E rows the separation becomes visible from $32\times$ onward, and the amplified $|\Delta|$ rows show it at every amplitude. The figures cover one cell and $16$ generated latents, so they are descriptive only.

\begin{figure}[!htb]
  \centering
  \includegraphics[width=\textwidth]{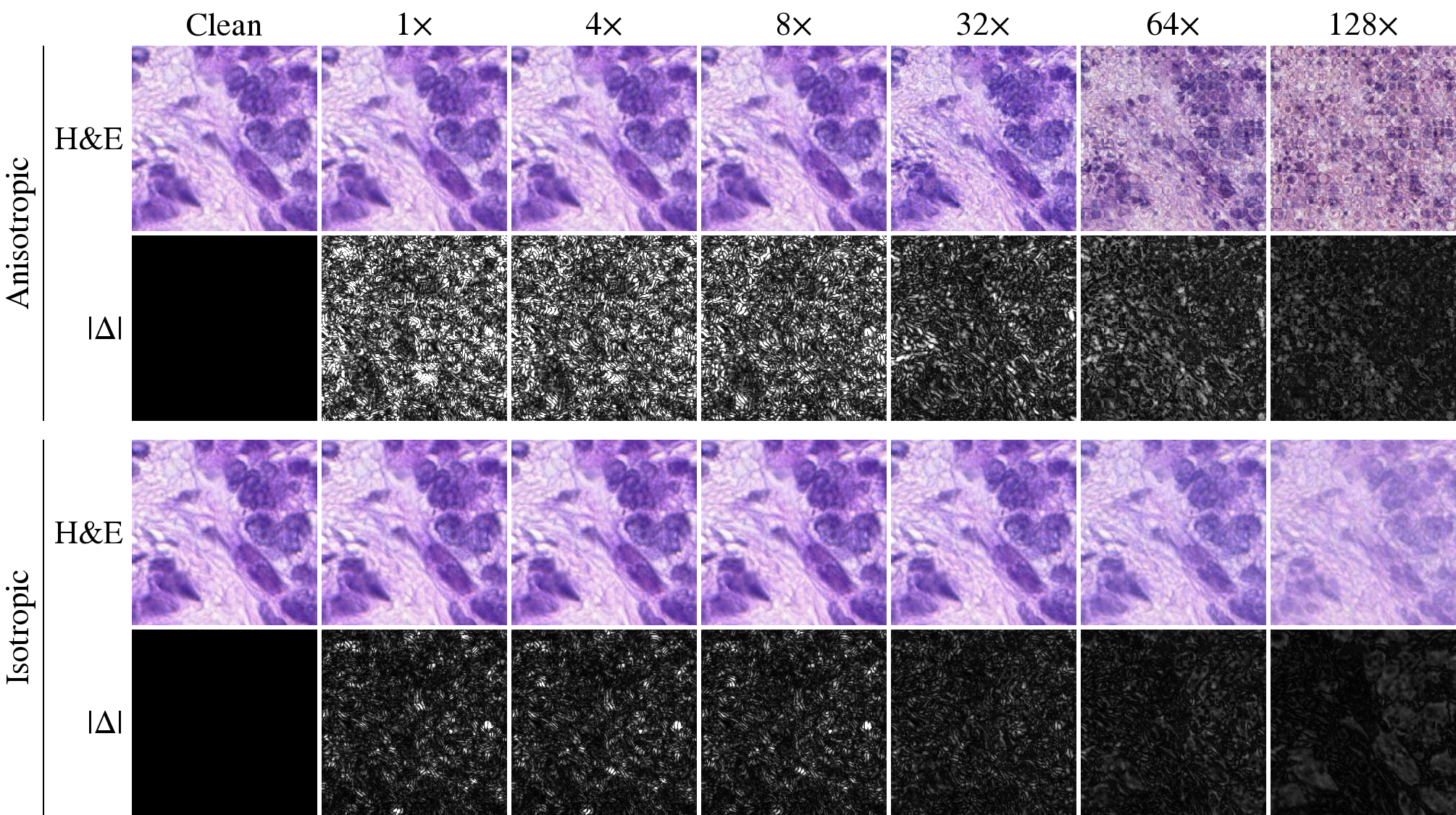}
  \caption{Pixel response of the anisotropic (upper block) and isotropic (lower block) decoders to a perturbation inside the top-$64$ subspace, BRCA. In each block the H\&E row holds the decodes and the $|\Delta|$ row their amplified absolute difference from the clean decode, for the clean latent and six amplitudes $m\times$.}
  \label{fig:dirlip_pixel_top64}
\end{figure}

\begin{figure}[!htb]
  \centering
  \includegraphics[width=\textwidth]{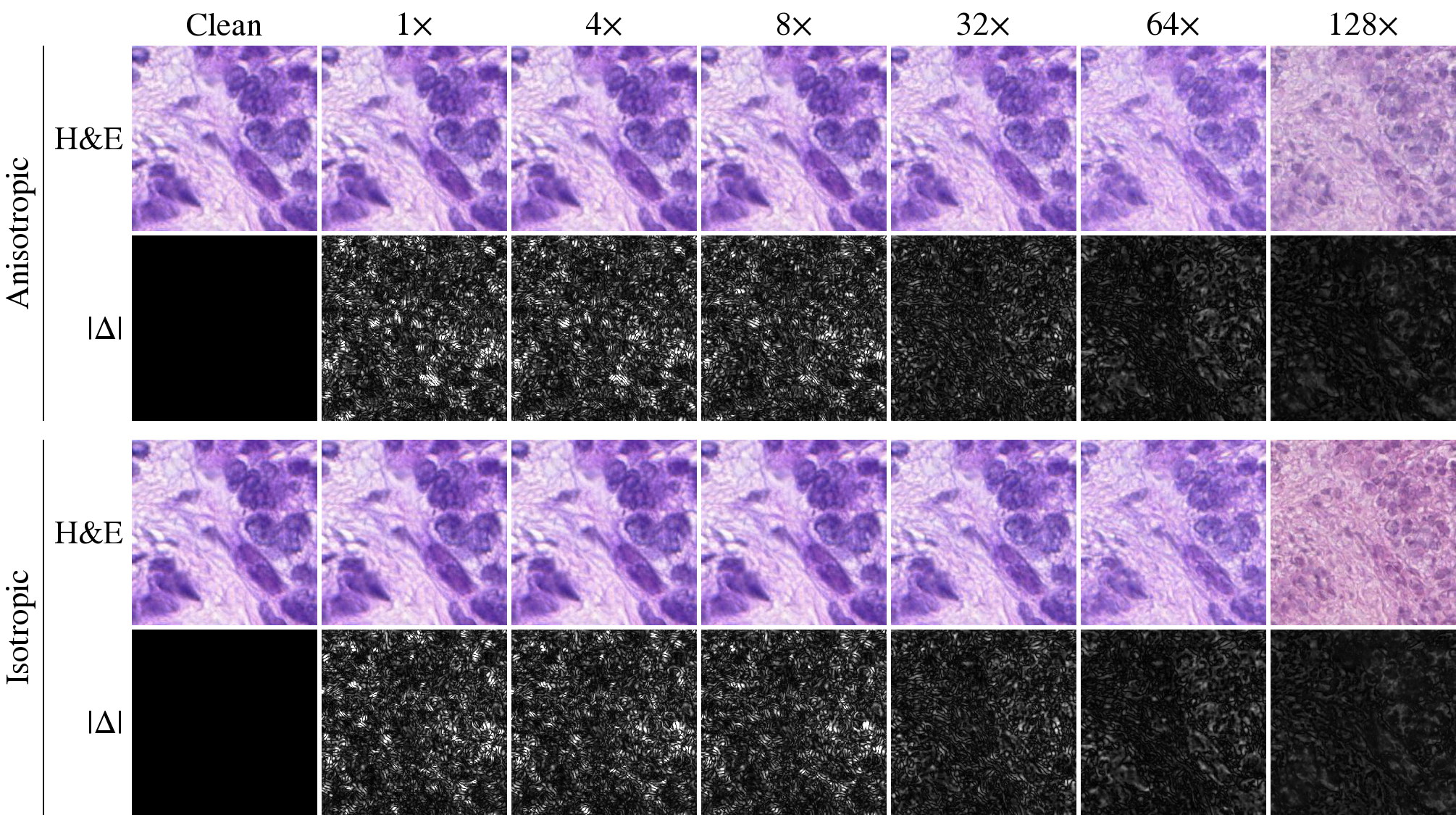}
  \caption{The same measurement as \cref{fig:dirlip_pixel_top64} with the perturbation inside $U_L$, the orthogonal complement of the deployed $U_H$.}
  \label{fig:dirlip_pixel_complement}
\end{figure}
\FloatBarrier

\section{Encoder Quality Dependence}
\label{app:encoder_quality_dependence}

This appendix measures the two spectral conditions behind the encoder dependence of \cref{tab:e2c} and tests whether the DiT or the decoder fails under DINOv2-L. It also measures how the spectral basis depends on its time anchor and reruns STREAM on a second in-domain encoder. The cross-encoder analyses place the failure of the domain-mismatched encoder, DINOv2-L on BRCA, in the decoder rather than the DiT, and show that DINOv2-L violates the effective-rank precondition that UNI meets.

\subsection{Subspace Stability Across Centroids}
\label{app:e2a}

We measure how stable the dominant subspace of UNI's velocity-field Jacobian is across centroids, since a single static cache relies on that stability. From the cached per-(centroid, token) SVD at the deployed anchor $t{=}0.99$ we compute the principal-angle alignment~\citep{bjorck1973angles} of the top-$k{=}64$ subspaces between every pair of the $K{=}50$ centroids, $\binom{50}{2} = 1{,}225$ pairs in all. The alignment is $A = \lVert Q_1^\top Q_2\rVert_F^2 / k \in [0,1]$ for orthonormal bases $Q_1, Q_2$ of the two subspaces, and $A = 1$ for identical subspaces.

\paragraph{Aggregation order.} Two aggregations of this quantity appear in this paper and they are not interchangeable. \emph{(a) Pooled.} For each centroid, the $256$ per-token top-$k$ bases are concatenated and only then QR-orthonormalized, giving one basis per centroid; $A$ is computed on those and the median over the $1{,}225$ pairs is reported. This is what produced the numbers below. \emph{(b) Matched per-(centroid, token).} $A$ is computed within each token separately and averaged over tokens; this is the statistic used in \cref{app:anchor}, where it must be matched to the per-(centroid, token) anchor comparison.

Under aggregation~(a) the median alignment is $0.360$, with interquartile range (IQR) $[0.330, 0.388]$ and range $[0.158, 0.462]$. Under aggregation~(b) the mean alignment is $0.529$, the reference that \cref{app:anchor} uses because its anchor comparison is also per (centroid, token). Both lie far above the $k/(d-1) = 0.063$ expected of independent frames, so UNI's dominant subspace is stable across centroids at the level the static cache requires.

\subsection{Effective Rank Distribution Across Centroids}
\label{app:e2b}

We compare how the effective rank of the Jacobian is distributed across cells for UNI and DINOv2-L. For each of the $K{=}50$ centroids and $N{=}256$ tokens we compute $k^\ast_\tau = \min\{k : \sum_{i \le k} s_i^2 / \sum_i s_i^2 \ge \tau\}$ at $\tau{=}0.90$. We report the $\tau{=}0.90$ cut of the median spectrum, the statistic that sets the deployed $k^\ast$, together with the IQR and range of the per-(centroid, token) values.

UNI gives $k^\ast = 236$ with per-cell IQR $[214.5, 253]$ and range $[178, 336]$. DINOv2-L gives $k^\ast = 280$ with per-cell IQR $[268, 293.75]$ and range $[1, 343]$. DINOv2-L's range includes degenerate (centroid, token) cells with $k^\ast = 1$, where a single direction dominates the Jacobian, and UNI has none. These cells violate the effective-rank precondition of \cref{app:encoder_scope}.

\subsection{DiT vs.\ Decoder Isolation via Latent-Space Fr\'echet Distance (FD)}
\label{app:e2d}

We test whether the DINOv2-L DiT itself fails to generate plausible latents, which would place the failure of \cref{tab:e2c} in the generator rather than the decoder. For each encoder we compute the FD between Gaussian fits to $1{,}000$ DiT-generated latents and to the corresponding real test latents, on mean-pooled tokens in that encoder's own feature space. It is the FID statistic applied to encoder features. Generation uses the deployed integrator at $25$ steps, so no decoder enters the measurement.

The UNI DiT reaches an FD of $0.030$ and the DINOv2-L DiT $0.032$, so the DINOv2-L DiT generates distributionally correct latents. The end-to-end failure of \cref{tab:e2c} therefore lies in the decoder rather than the DiT. The cause is the out-of-domain encoder: the DiT learns DINOv2-L's latent distribution as faithfully as UNI's, but these latents encode natural-image features, so the decoder has a harder time mapping them to high-fidelity histopathology images.

\subsection{Time-Anchor Sensitivity of the Spectral Basis}
\label{app:anchor}

We measure how far the spectral basis moves when its time anchor changes. The deployed cache is extracted at the single anchor $t = 0.99$ (\cref{app:implementation}) and is neither time-averaged nor recomputed during sampling. The dominant subspace is not invariant in $t$, and we do not claim that it is.

We re-extract the cache at the adjacent anchor $t = 0.9$ on the same $50$ centroids, with only the anchor differing, and align the top-$64$ subspaces of the two anchors per (centroid, token), as in aggregation~(b) of \cref{app:e2a}. The reference is two different centroids at the same anchor, the variation the $K{=}50$ cache is built to absorb. The floor is two random $64$-frames.

The anchor matters more than the cross-centroid variation the cache is built to absorb (\cref{tab:anchor}). The effective rank moves with it.

\begin{table}[!h]
  \centering
  \caption{Top-$k{=}64$ subspace alignment $A$ of the UNI cache on BRCA. Intervals are $95\%$ percentile bootstrap over the $50$ centroids.}
  \label{tab:anchor}
  \footnotesize
  \setlength{\tabcolsep}{6pt}
  \begin{tabular}{@{}lc@{}}
    \toprule
    Comparison & $A$ \\
    \midrule
    Adjacent anchors, $t{=}0.9$ vs.\ $t{=}0.99$, same centroid & $0.297$ $[0.283, 0.312]$ \\
    Different centroids, same anchor $t{=}0.99$ (cache reference) & $0.529$ \\
    Two random $64$-frames (independence floor, $k/(d-1)$) & $0.063$ \\
    \bottomrule
  \end{tabular}
\end{table}

\paragraph{Limitations.} The single-anchor cache is therefore an approximation and a limitation. We have not measured stability along the sampler's own trajectory, which is a different quantity. The extraction linearizes at clean centroids, and measuring it would require re-linearizing at the sampler's iterates, where the base point moves as well as the time label. However, the single-anchor cache still buys a measured gain. At the matched 45k decoder budget the deployed basis improves gFID over isotropic noise (\cref{tab:ablation}), and swapping it for a different but equally valid basis costs $1.25$ gFID (\cref{app:svd_factor}).

\subsection{A Second In-Domain Encoder: Virchow2}
\label{app:virchow2}

\Cref{tab:e2c} replaces UNI with a domain-mismatched encoder. To test whether STREAM's gain is specific to UNI's spectrum rather than to in-domain histopathology pretraining generally, we rerun the full pipeline with Virchow2 features and hold the dataset, the tiling and the 45k decoder budget fixed.

\begin{table}[!h]
  \centering
  \caption{Second in-domain histopathology encoder, BRCA at the matched 45k decoder budget. Only the within-encoder isotropic-to-anisotropic contrast is a valid comparison; see the confounds below.}
  \label{tab:virchow2}
  \footnotesize
  \setlength{\tabcolsep}{6pt}
  \begin{tabular}{@{}lcc@{}}
    \toprule
    Configuration & rFID $\downarrow$ & gFID $\downarrow$ \\
    \midrule
    \textbf{Virchow2, bridge + anisotropic (STREAM)} & $\mathbf{6.37}$ & $\mathbf{10.48}$ \\
    Virchow2, bridge + isotropic & $11.28$ & $14.75$ \\
    \midrule
    \textbf{UNI, bridge + anisotropic (STREAM)} & $\mathbf{3.52}$ & $\mathbf{6.86}$ \\
    UNI, bridge + isotropic & $6.51$ & $9.37$ \\
    \bottomrule
  \end{tabular}
\end{table}

STREAM's bridge and anisotropic decoder transfer to Virchow2 (\cref{tab:virchow2}). With the bridge in both arms, the anisotropic decoder improves on its isotropic control by $4.91$ rFID and $4.27$ gFID. In relative terms the gains are close to UNI's, $44\%$ and $29\%$ against $46\%$ and $27\%$, so they are not specific to UNI's particular spectrum. They remain encoder-dependent, as DINOv2-L shows in \cref{tab:e2c}, and \cref{app:encoder_scope} states the condition.

\paragraph{Two confounds, both forced by the encoder.} First, Virchow2 emits $1280$-dimensional patch tokens, above the hidden size of $1152$ of the deployed LightningDiT-XL (\cref{tab:arch}). RAE~\citep{zheng2025rae} reports that a DiT whose hidden size is below the token dimension hits a bottleneck, and \citet[Thm.~1]{zheng2025rae} establishes a non-zero loss lower bound for diffusion models of the form $Bf(Ax_t,t)$ when the hidden dimension satisfies $d<n$. This run therefore uses LightningDiT-1.0B ($1.03$B parameters against $676$M) with a correspondingly larger SVD cache and decoder. Only the within-encoder contrast is valid, and the absolute Virchow2 and UNI rows are not comparable. UNI's $d{=}1024$ sits below $1152$, so the deployed configuration is unaffected. Second, Virchow2's patch size is $14$ rather than $16$, so each tile is resized from $256$ to $224$ pixels to keep $256$ tokens per tile. The training and evaluation data are otherwise identical, but Virchow2 sees each tile at lower resolution. We do not re-tune STREAM's hyperparameters for Virchow2 and expect their optimal values to differ.

\subsection{Scope of the Encoder Condition}
\label{app:encoder_scope}

The anisotropic decoder needs a spectral condition with two distinct parts, and only the second concerns the singular values.

\begin{enumerate}[leftmargin=*,topsep=2pt,itemsep=3pt,label=(\roman*)]
  \item \emph{Cross-centroid stability of the Jacobian's dominant singular subspace}, measured by principal angles between the top-$k$ subspaces of every pair of the $K$ centroids (\cref{app:e2a}). A single static cache is meaningful only if the basis of a sample's assigned centroid describes that sample.
  \item \emph{A non-degenerate effective rank} (\cref{app:e2b}). The $\tau$-cut $k^\ast$ must sit in the regime $1 \ll k^\ast \ll d-1$ across cells, so that the anisotropic covariance differs meaningfully from the isotropic one. DINOv2-L's $k^\ast$ distribution contains degenerate (centroid, token) cells with $k^\ast = 1$, which UNI never produces.
\end{enumerate}

We therefore expect the anisotropic decoder to help only for encoder and generator combinations that meet both conditions. UNI meets both (\cref{app:e2a,app:e2b}), and Virchow2 shows the same isotropic-to-anisotropic gain (\cref{tab:virchow2}). We do not claim that the anisotropic decoder is effective independently of the encoder. Both quantities are read directly off the cached Jacobian SVD at no additional training cost, so they can screen a new encoder or modality before a decoder run is committed. We regard domain matching as the most reliable way to meet the condition, not as a second, independent requirement. If a target domain has no encoder of this quality, we expect the anisotropic decoder's benefit to degrade. The failure then lies in the decoder rather than the generator, since a DiT trained on a domain-mismatched encoder still produces distributionally correct latents (\cref{app:e2d}).

\clearpage
\section{Image Gallery}
\label{app:image_gallery}

\subsection{Reconstruction comparison on BRCA}
\label{app:gallery_h1}

\begin{figure}[H]
  \centering
  \setlength{\tabcolsep}{1.5pt}
  \renewcommand{\arraystretch}{1.05}
  \begin{tabular}{@{}>{\centering\arraybackslash}m{0.035\linewidth}@{\hspace{2pt}}*{3}{>{\centering\arraybackslash}m{0.18\linewidth}}@{}}
     & GT 1 & GT 2 & GT 3 \\
    \rotatebox{90}{GT} &
      \includegraphics[width=\linewidth]{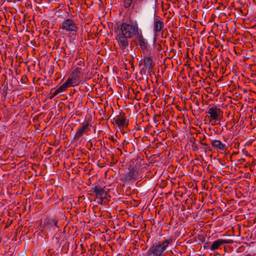} &
      \includegraphics[width=\linewidth]{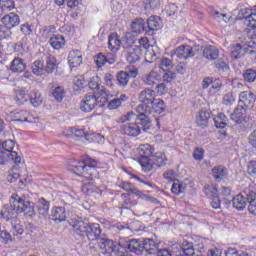} &
      \includegraphics[width=\linewidth]{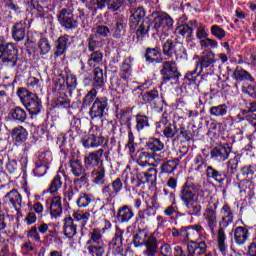} \\
    \rotatebox{90}{ZoomLDM} &
      \includegraphics[width=\linewidth]{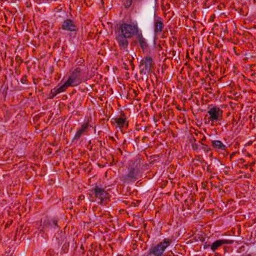} &
      \includegraphics[width=\linewidth]{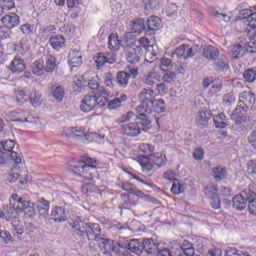} &
      \includegraphics[width=\linewidth]{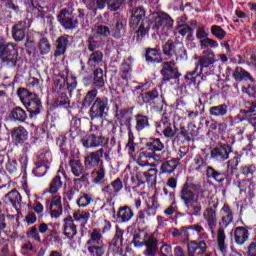} \\
    \rotatebox{90}{PixCell} &
      \includegraphics[width=\linewidth]{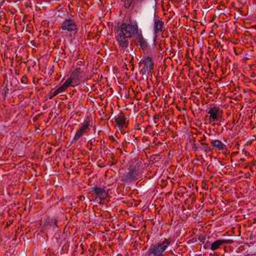} &
      \includegraphics[width=\linewidth]{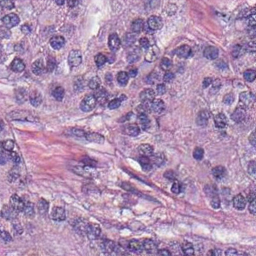} &
      \includegraphics[width=\linewidth]{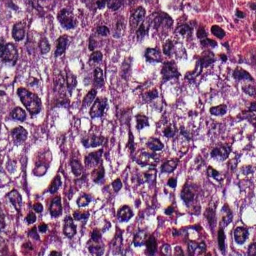} \\
    \rotatebox{90}{STREAM (Ours)} &
      \includegraphics[width=\linewidth]{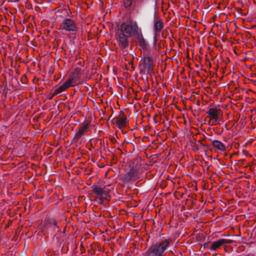} &
      \includegraphics[width=\linewidth]{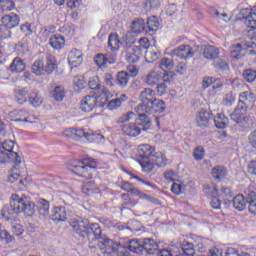} &
      \includegraphics[width=\linewidth]{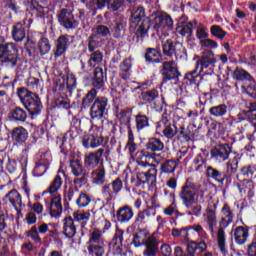} \\
    \rotatebox{90}{RAE} &
      \includegraphics[width=\linewidth]{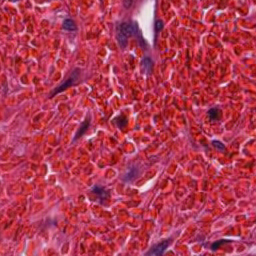} &
      \includegraphics[width=\linewidth]{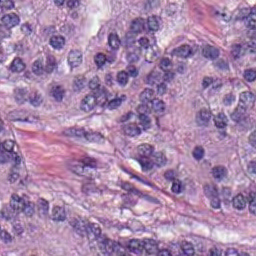} &
      \includegraphics[width=\linewidth]{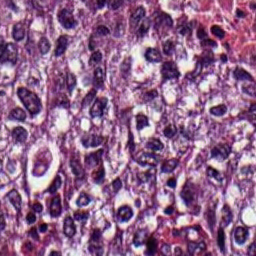} \\
    \rotatebox{90}{SVG} &
      \includegraphics[width=\linewidth]{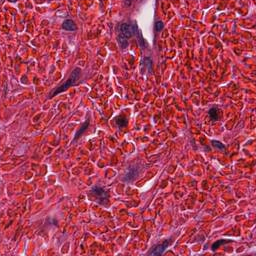} &
      \includegraphics[width=\linewidth]{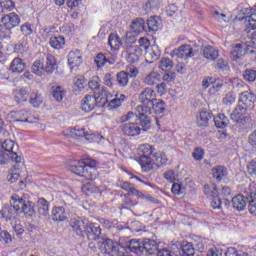} &
      \includegraphics[width=\linewidth]{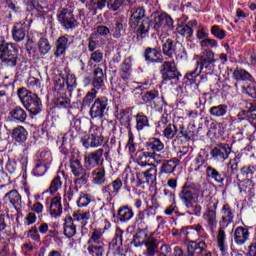} \\
    \rotatebox{90}{REPA-E} &
      \includegraphics[width=\linewidth]{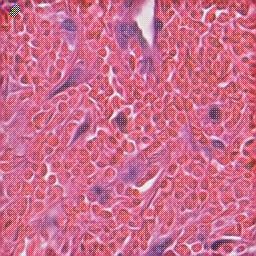} &
      \includegraphics[width=\linewidth]{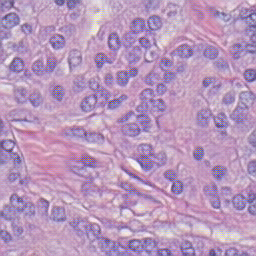} &
      \includegraphics[width=\linewidth]{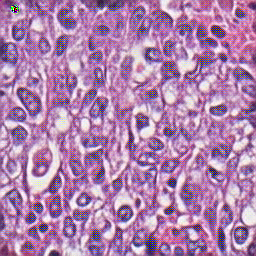} \\
  \end{tabular}
  \caption{Reconstructions on BRCA from ZoomLDM, PixCell, STREAM (Ours), RAE, SVG and REPA-E, taken from the rFID evaluation, with GT tiles in the top row. The tiles are selected by per-image LPIPS ranking of the GT--reconstruction pairs to highlight perceptual differences across models. The six models shown all use the shared UNI encoder of \cref{tab:results}.}
  \label{fig:gallery_h1}
\end{figure}

\clearpage
\subsection{Reconstruction comparison on COADREAD}
\label{app:gallery_h2}

\begin{figure}[H]
  \centering
  \setlength{\tabcolsep}{1.5pt}
  \renewcommand{\arraystretch}{1.05}
  \begin{tabular}{@{}>{\centering\arraybackslash}m{0.035\linewidth}@{\hspace{2pt}}*{3}{>{\centering\arraybackslash}m{0.18\linewidth}}@{}}
     & GT 1 & GT 2 & GT 3 \\
    \rotatebox{90}{GT} &
      \includegraphics[width=\linewidth]{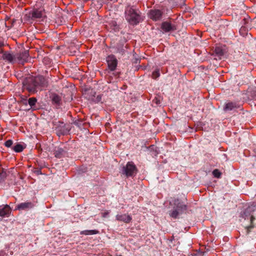} &
      \includegraphics[width=\linewidth]{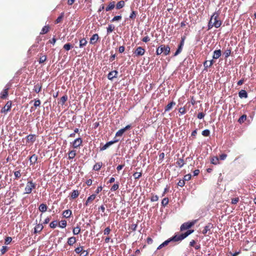} &
      \includegraphics[width=\linewidth]{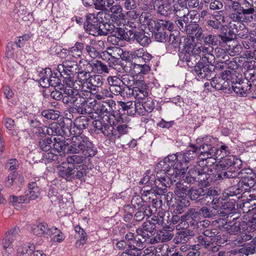} \\
    \rotatebox{90}{ZoomLDM} &
      \includegraphics[width=\linewidth]{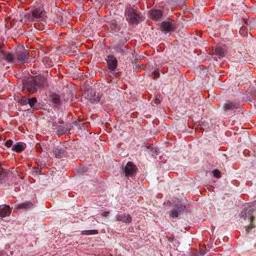} &
      \includegraphics[width=\linewidth]{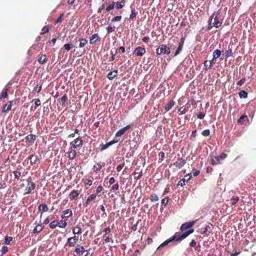} &
      \includegraphics[width=\linewidth]{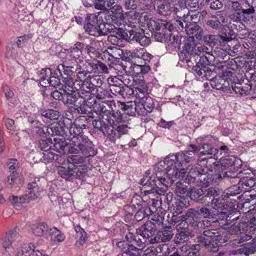} \\
    \rotatebox{90}{PixCell} &
      \includegraphics[width=\linewidth]{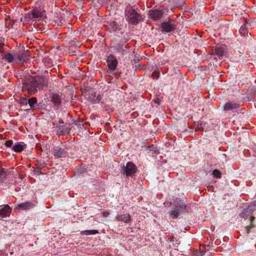} &
      \includegraphics[width=\linewidth]{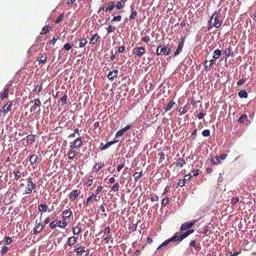} &
      \includegraphics[width=\linewidth]{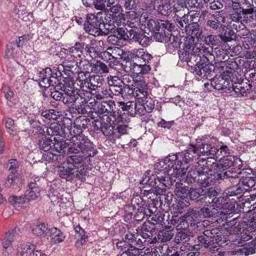} \\
    \rotatebox{90}{STREAM (Ours)} &
      \includegraphics[width=\linewidth]{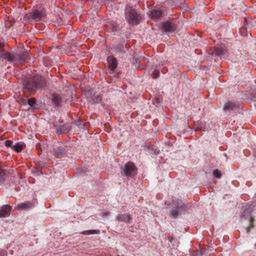} &
      \includegraphics[width=\linewidth]{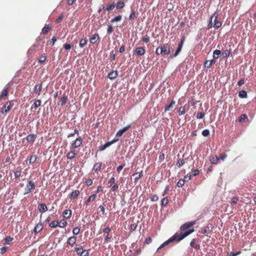} &
      \includegraphics[width=\linewidth]{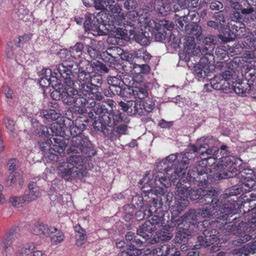} \\
    \rotatebox{90}{RAE} &
      \includegraphics[width=\linewidth]{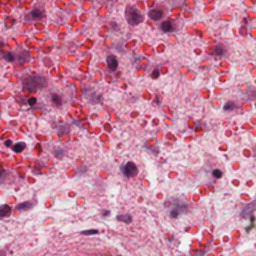} &
      \includegraphics[width=\linewidth]{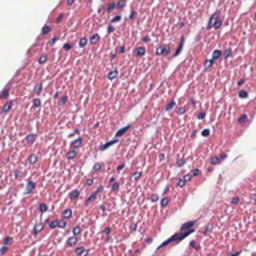} &
      \includegraphics[width=\linewidth]{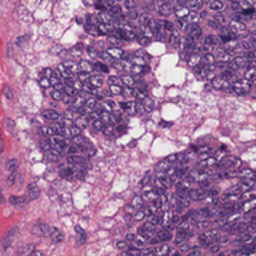} \\
    \rotatebox{90}{SVG} &
      \includegraphics[width=\linewidth]{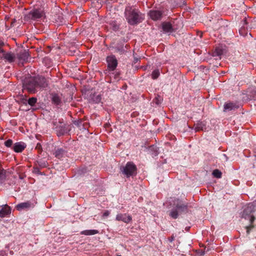} &
      \includegraphics[width=\linewidth]{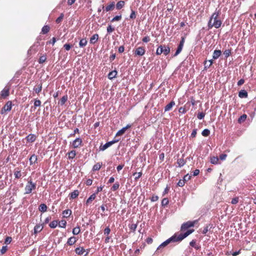} &
      \includegraphics[width=\linewidth]{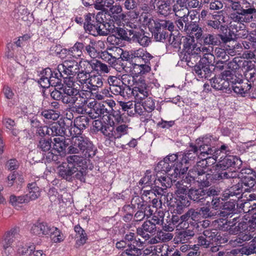} \\
    \rotatebox{90}{REPA-E} &
      \includegraphics[width=\linewidth]{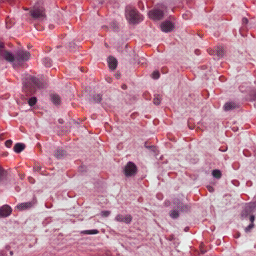} &
      \includegraphics[width=\linewidth]{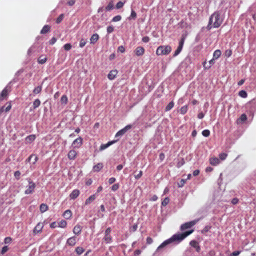} &
      \includegraphics[width=\linewidth]{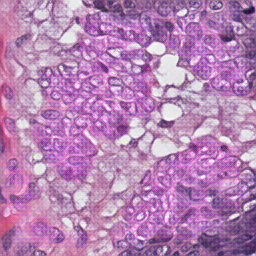} \\
  \end{tabular}
  \caption{Reconstructions on COADREAD from ZoomLDM, PixCell, STREAM (Ours), RAE, SVG and REPA-E, taken from the rFID evaluation, with GT tiles in the top row. The tiles are selected by per-image LPIPS ranking of the GT--reconstruction pairs to highlight perceptual differences across models. The six models shown all use the shared UNI encoder of \cref{tab:results}.}
  \label{fig:gallery_h2}
\end{figure}

\clearpage
\subsection{Generation samples on BRCA}
\label{app:gallery_h3}

\begin{figure}[H]
  \centering
  \setlength{\tabcolsep}{1.5pt}
  \renewcommand{\arraystretch}{1.05}
  \begin{tabular}{@{}>{\centering\arraybackslash}m{0.035\linewidth}@{\hspace{2pt}}*{3}{>{\centering\arraybackslash}m{0.18\linewidth}}@{}}
     & Cluster 1 & Cluster 2 & Cluster 3 \\
    \rotatebox{90}{GT} &
      \includegraphics[width=\linewidth]{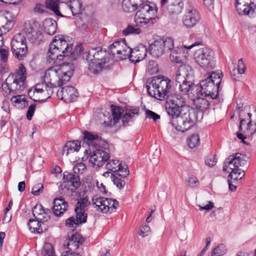} &
      \includegraphics[width=\linewidth]{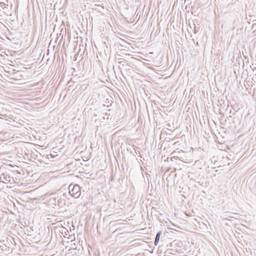} &
      \includegraphics[width=\linewidth]{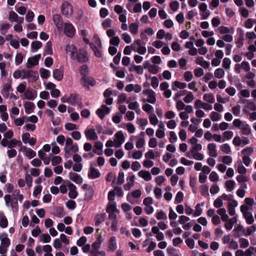} \\
    \rotatebox{90}{ZoomLDM} &
      \includegraphics[width=\linewidth]{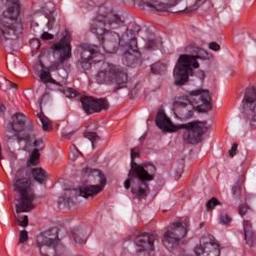} &
      \includegraphics[width=\linewidth]{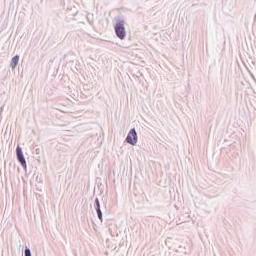} &
      \includegraphics[width=\linewidth]{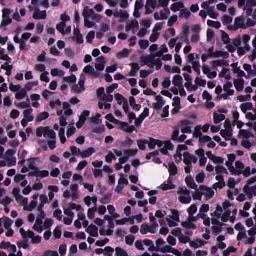} \\
    \rotatebox{90}{PixCell} &
      \includegraphics[width=\linewidth]{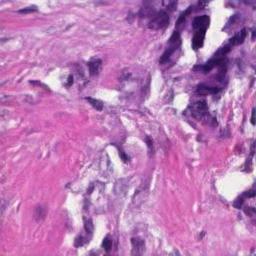} &
      \includegraphics[width=\linewidth]{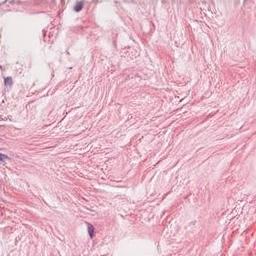} &
      \includegraphics[width=\linewidth]{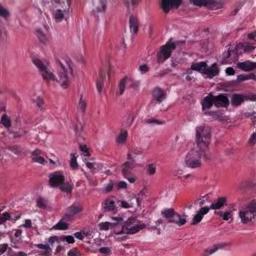} \\
    \rotatebox{90}{STREAM (Ours)} &
      \includegraphics[width=\linewidth]{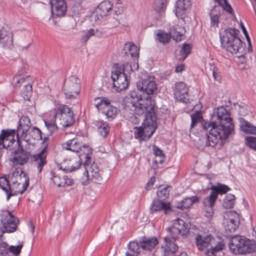} &
      \includegraphics[width=\linewidth]{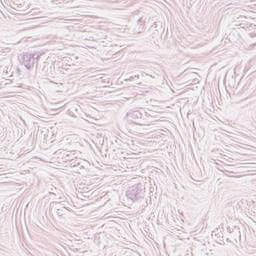} &
      \includegraphics[width=\linewidth]{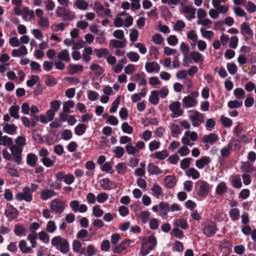} \\
    \rotatebox{90}{RAE} &
      \includegraphics[width=\linewidth]{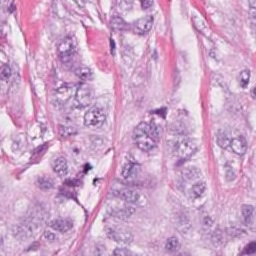} &
      \includegraphics[width=\linewidth]{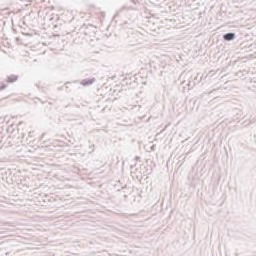} &
      \includegraphics[width=\linewidth]{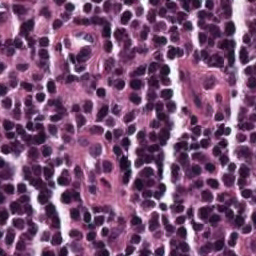} \\
    \rotatebox{90}{SVG} &
      \includegraphics[width=\linewidth]{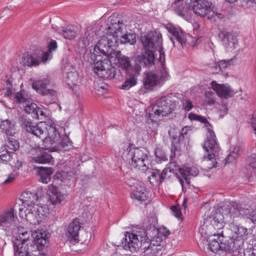} &
      \includegraphics[width=\linewidth]{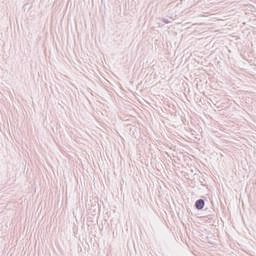} &
      \includegraphics[width=\linewidth]{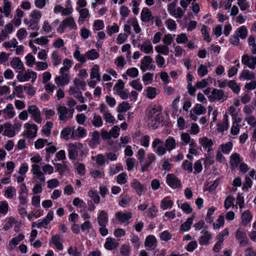} \\
    \rotatebox{90}{REPA-E} &
      \includegraphics[width=\linewidth]{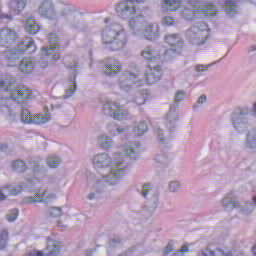} &
      \includegraphics[width=\linewidth]{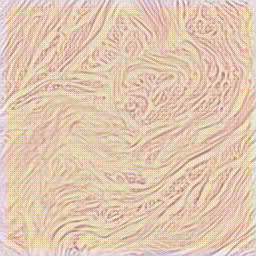} &
      \includegraphics[width=\linewidth]{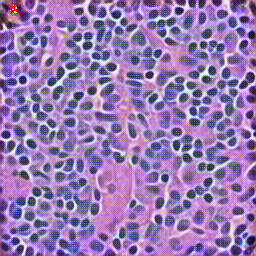} \\
  \end{tabular}
  \caption{Generation samples on BRCA. Unconditional samples used for gFID evaluation from ZoomLDM, PixCell, STREAM (Ours), RAE, SVG and REPA-E, with GT tiles in the top row. Each column corresponds to a representative tissue cluster obtained by K-means on Virchow2 features of the GT set. For each cluster and model we display one of the three samples whose Virchow2 feature lies closest to the cluster centroid. The six models shown all use the shared UNI encoder of \cref{tab:results}.}
  \label{fig:gallery_h3}
\end{figure}

\clearpage
\subsection{Generation samples on COADREAD}
\label{app:gallery_h4}

\begin{figure}[H]
  \centering
  \setlength{\tabcolsep}{1.5pt}
  \renewcommand{\arraystretch}{1.05}
  \begin{tabular}{@{}>{\centering\arraybackslash}m{0.035\linewidth}@{\hspace{2pt}}*{3}{>{\centering\arraybackslash}m{0.18\linewidth}}@{}}
     & Cluster 1 & Cluster 2 & Cluster 3 \\
    \rotatebox{90}{GT} &
      \includegraphics[width=\linewidth]{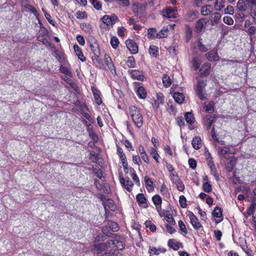} &
      \includegraphics[width=\linewidth]{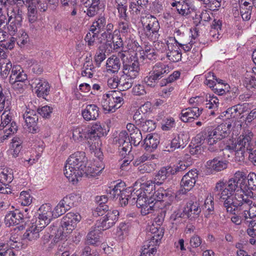} &
      \includegraphics[width=\linewidth]{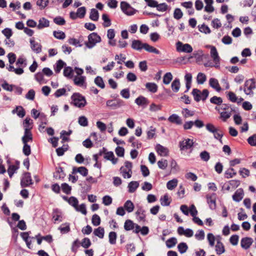} \\
    \rotatebox{90}{ZoomLDM} &
      \includegraphics[width=\linewidth]{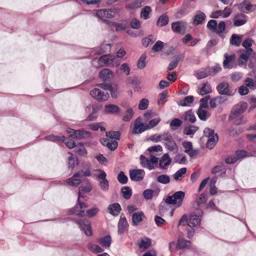} &
      \includegraphics[width=\linewidth]{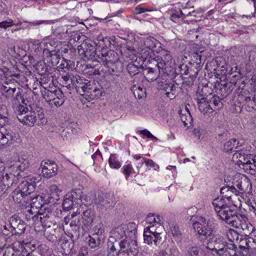} &
      \includegraphics[width=\linewidth]{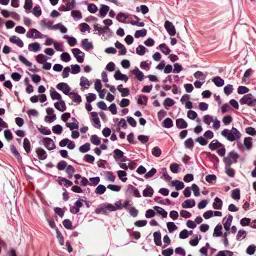} \\
    \rotatebox{90}{PixCell} &
      \includegraphics[width=\linewidth]{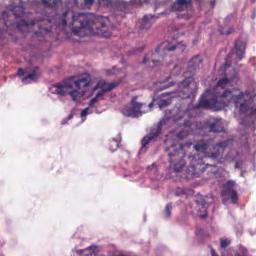} &
      \includegraphics[width=\linewidth]{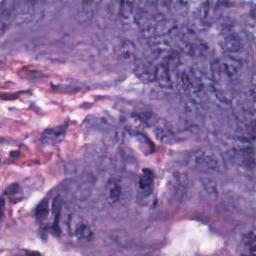} &
      \includegraphics[width=\linewidth]{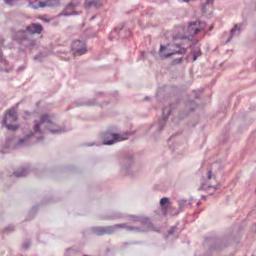} \\
    \rotatebox{90}{STREAM (Ours)} &
      \includegraphics[width=\linewidth]{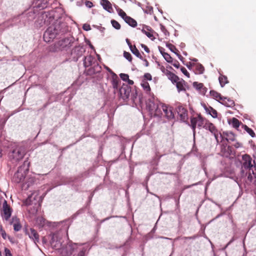} &
      \includegraphics[width=\linewidth]{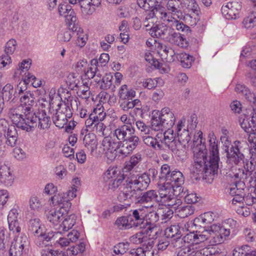} &
      \includegraphics[width=\linewidth]{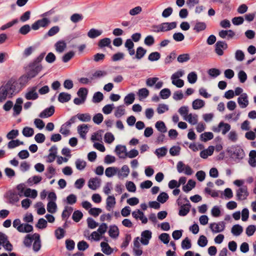} \\
    \rotatebox{90}{RAE} &
      \includegraphics[width=\linewidth]{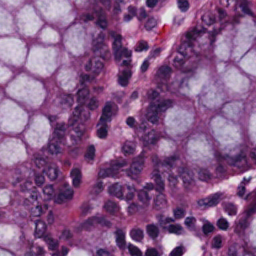} &
      \includegraphics[width=\linewidth]{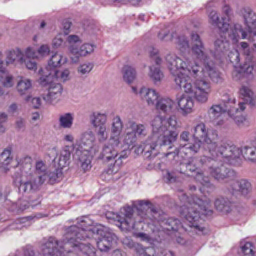} &
      \includegraphics[width=\linewidth]{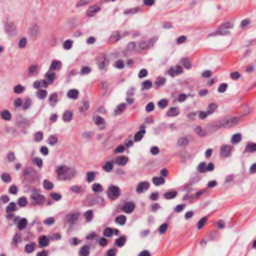} \\
    \rotatebox{90}{SVG} &
      \includegraphics[width=\linewidth]{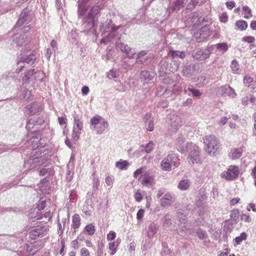} &
      \includegraphics[width=\linewidth]{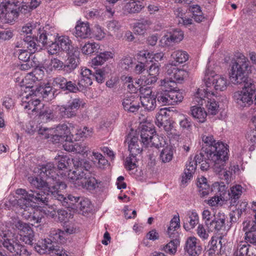} &
      \includegraphics[width=\linewidth]{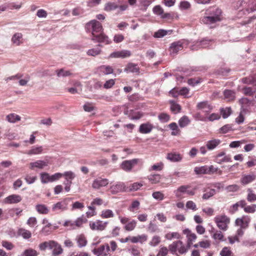} \\
    \rotatebox{90}{REPA-E} &
      \includegraphics[width=\linewidth]{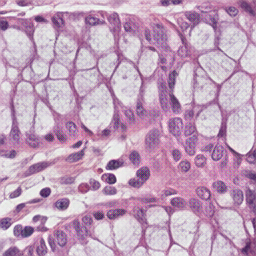} &
      \includegraphics[width=\linewidth]{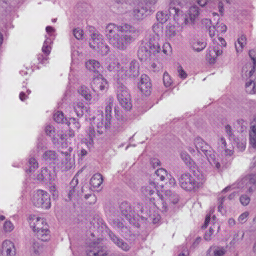} &
      \includegraphics[width=\linewidth]{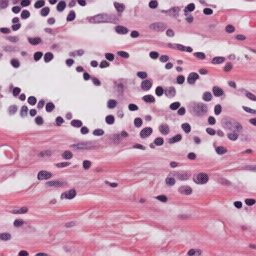} \\
  \end{tabular}
  \caption{Generation samples on COADREAD. Unconditional samples used for gFID evaluation from ZoomLDM, PixCell, STREAM (Ours), RAE, SVG and REPA-E, with GT tiles in the top row. Each column corresponds to a representative tissue cluster obtained by K-means on Virchow2 features of the GT set. For each cluster and model we display one of the three samples whose Virchow2 feature lies closest to the cluster centroid. The six models shown all use the shared UNI encoder of \cref{tab:results}.}
  \label{fig:gallery_h4}
\end{figure}

\end{document}